\documentclass{article}

\usepackage{amsmath,amsfonts,bm}

\def\eqref#1{Eq.\,(\ref{#1})}

\def\1{\bm{1}}

\DeclareMathAlphabet{\mathsfit}{\encodingdefault}{\sfdefault}{m}{sl}
\SetMathAlphabet{\mathsfit}{bold}{\encodingdefault}{\sfdefault}{bx}{n}

\usepackage{microtype}
\usepackage{graphicx}
\usepackage{subcaption}
\usepackage{booktabs} %
\usepackage{enumitem}

\usepackage{hyperref}

\usepackage[preprint]{icml2026}

\usepackage{amsmath}
\usepackage{amssymb}
\usepackage{mathtools}
\usepackage{amsthm}
\usepackage{style}
\usepackage{natbib}
\usepackage{wrapfig} 

\usepackage[capitalize,noabbrev]{cleveref}

\usepackage{tikz} 
\usetikzlibrary{positioning, shadows}

\usepackage[textsize=tiny]{todonotes}

\icmltitlerunning{Personalized Multi-Agent Average Reward TD-Learning via Joint Linear Approximation}

\begin{document}

\twocolumn[
  \icmltitle{Personalized Multi-Agent Average Reward TD-Learning via Joint Linear Approximation}

  \icmlsetsymbol{equal}{*}

  \begin{icmlauthorlist}
    \icmlauthor{Leo Muxing Wang}{yyy}
    \icmlauthor{Pengkun Yang}{xxx}
    \icmlauthor{Lili Su}{yyy}
  \end{icmlauthorlist}

  \icmlaffiliation{yyy}{Northeastern University}
  \icmlaffiliation{xxx}{Tsinghua University}

  \icmlcorrespondingauthor{Leo (Muxing) Wang}{wang.muxin@northeastern.edu}

  \icmlkeywords{Machine Learning, ICML}

  \vskip 0.3in
]

\printAffiliationsAndNotice{}  %

\begin{abstract}
We study personalized multi-agent average reward TD learning, in which a collection of agents interacts with different environments and jointly learns their respective value functions. 
We focus on the setting where there exists a shared linear representation and the agents' optimal weights collectively lie in an unknown linear subspace. 
Inspired by the recent success of personalized federated learning (PFL), we study the convergence of cooperative single-timescale TD learning in which agents iteratively estimate the common subspace and local heads. We showed that this decomposition can filter out conflicting signals, 
effectively mitigating the negative impacts of ``misaligned'' signals, and achieve linear speedup.    
The main technical challenges lie in the heterogeneity, the Markovian sampling, and their intricate interplay in shaping error evolutions. 
Specifically, not only are the error dynamics of multiple variables are closely interconnected, but there is also no direct contraction for the principal angle distance between the optimal subspace and the estimated subspace.  
We hope our analytical techniques can be useful to inspire research on deeper exploration into leveraging common structures. Experiments are provided to show the benefits of learning via a shared structure to the more general control problem.

\end{abstract}
 
\section{Introduction}
\label{sec: intro}
In many real-world applications, ranging from simple assistive robotics to autonomous vehicles,  autonomous agents operate in different local environments. 
For example, consider the simple robot vacuums. Their interacting environments are mostly shaped by the household conditions, which can vary significantly in floor plans, obstacle types (e.g., moving humans and furniture), and spatial configurations. Similarly, autonomous vehicles deployed across different regions of a metropolitan area may encounter substantial variability in road conditions and traffic patterns. Such environmental heterogeneity %
can lead to misaligned learning signals across agents, which may significantly hinder convergence and impair generalization of joint learning.    
Nevertheless, applying standard single-agent reinforcement learning (RL) solutions may lead to overall intensive, yet possibly redundant, computation and sample collection. This is because the cumulative knowledge obtained by one agent may still be partially useful to others when some common structure (though unknown) exists. 
These two competing intuitions raise a natural question: 
\vspace{-0.5em}
\begin{center}
\em Question: Do the benefits of collaboration outweigh the drawbacks -- or is it the other way around?  
\end{center} 
\vskip -0.5\baselineskip 

Existing studies of multi-agent RL (MARL) with heterogeneous environments are mostly restricted to training a common policy or value-/Q-function~\citep{jin_federated_2022,wang2023federated,xie_fedkl_2023,zhang2023convergence}, and the optimality is often described with respect to some imaginary environment constructed as a weighted average of the agents' state-transition kernels \cite{jin_federated_2022}. 
Though this approach has practical value -- particularly when environmental heterogeneity is mild -- in many applications (such as on-device recommender systems \citep{maeng2022towards} and smart healthcare \citep{nguyen2021federated}), heterogeneity is often moderate to severe, which can significantly degrade the performance of a common policy.
In a parallel line of research on multi-agent Markov games or competitive MARL, learning personalized (or localized) policies or value-/Q-functions has recently garnered considerable attention \citep{filar2012competitive, zhang2018fully, qu2022scalable, zhang2023global, daskalakis2023complexity}. %
Specifically, in their setting, agents interact with one another within a shared environment, and the rewards received by individual agents are determined by the joint actions of all agents. It is easy to see that in such a setup, the more agents there are, the more complicated an agent's decision problem becomes. In fact, this line of work often suffers from the so-called curse of multi-agents—the sample complexity increases linearly or even exponentially \citep{song2022when, jin2024v,zhang2018fully} with respect to the number of agents. 

In this paper, we take a step toward addressing the above question by studying cooperative temporal-difference (TD) learning, where a collection of agents collaboratively learn their heterogeneous value functions through local TD updates.  
Inspired by the recent success of personalized federated learning (PFL) and multi-task learning, we focus on the specific underlying common structure that the agents' optimal weights under a shared linear representation collectively lie in a low-dimensional linear subspace \cite{tripuraneni2021provable,collins2021exploiting,thekumparampil2021statistically,duchi2022subspace,du2021fewshot,tian2023learning,niu2024collaborative}. Our contributions can be summarized as 
\begin{itemize}
    \item We propose and analyze the convergence of a cooperative average-reward TD method in which agents iteratively estimate the common subspace and agent-specific heads. We show that the overall reward estimate errors decay at a rate of $\tilde{\calO}(1/T)$, where $\tilde{\calO}$ hides polylog factors of T, with a faster stepsize. Our dynamics of common subspace $\bfB$ and local heads $\omega^k$ are single timescale, i.e., their respective stepsizes are on the same order with respect to $T$. We show that, when $T$ is sufficiently large, under additional conditions, the joint estimation errors of the subspace and the local heads converge to 0 at a rate $\tilde\calO(\frac{1}{\sqrt{TLK}})$, where $L$ is the number of local steps and $K$ is the number of agents. 
    \item Our analysis is highly nontrivial, which may be insightful for a broader audience dealing with coupled heterogeneous dynamics. Technically, the Markovian sample of TD error makes it difficult to obtain a direct contraction for the principal angle distance between the optimal subspace and the estimated subspace, but only an indirect subtraction in terms of the local weights error. 
    To address this challenge, we show that the local weights error can be lower-bounded by the principal angle distance times a constant that depends on the diversity of the optimal local weights. 
    \item We empirically validate that the proposed personalized framework achieves superior performance compared to other frameworks, including single-agent, universal-policy, and two-timescale approaches, with respect to convergence speed, stability, and generalization.
\end{itemize}

During the preparation of this paper, \cite{xionglinear} studied discounted TD learning with two timescales and established a linear convergence speedup. %
Key differences are as follows: We study average-reward TD and we study the single timescale dynamics of the local heads and shared common spaces.  
However, under their notation, the validity of the error contraction in Eq.\,(54) appears to rely on the assumptions $D_1 = o(\beta_t)$ and $D_2 = o(\alpha_{t+1})$. Since both terms involve stepsize ratios such as $\frac{\beta_t}{\alpha_t}$ and $\frac{\beta_t}{\alpha_{t+1}}$, these conditions implicitly require a sufficiently strong separation between the two stepsize sequences. While such assumptions are standard in two-timescale analyses, they are not automatically guaranteed for commonly used polynomial stepsizes and may require additional verification. In contrast, our analysis establishes error contraction under a single-timescale setting and directly controls the coupled error terms through a unified Lyapunov argument, without relying on such asymptotic separation conditions.

\section{Related Work}
\label{sec: related work}
\noindent{\bf MARL under heterogeneous environments.}
There is an emerging interest in mathematically understanding the role of environmental heterogeneity in the performance of the federated versions of classic reinforcement learning algorithms \citep{jin_federated_2022,wang2023federated,xie_fedkl_2023,zhang2023convergence}, yet most existing studies focus on training a common policy or approximating a common function.  
\cite{jin_federated_2022} studied federated Q-learning and policy gradient methods, assuming known transition kernels. 
\cite{wang2025on} studied federated Q-learning with unknown transition kernels, and found that, in the presence of environmental heterogeneity, the eventual convergence rate may depend solely on the number of communication rounds -- multiple updates within each round cannot accelerate convergence.  
\cite{wang2023federated} proposed FedTD(0) with linear function approximation,  %
and linear convergence speedup in a low-heterogeneity regime. 
\cite{xie_fedkl_2023} used the KL-divergence to penalize the deviation of local update from the global policy, and proved that under the setting of heterogeneous environments, the local update is beneficial for global convergence using their method. Yet, no convergence rate is provided. \cite{zhang2024finite} proposed FedSARSA and proved that the algorithm can converge to a near-optimal solution. 
Neither \citep{xie_fedkl_2023} nor \citep{zhang2024finite} characterized sample complexity.

\noindent{\bf Personalization in heterogeneous environments.}
Training personalized policies in cooperative multi-agent systems in the presence of environment heterogeneity remains largely underexplored. 
\cite{yang2023federated} studied the training of personalized policies in the context where each agent has a private reward function but shares a common transition kernel for the environment. The analysis focused solely on the convergence of the global policy, defined as the average of the local policies.  
\cite{jin_federated_2022} proposed a heuristic deep FedRL method, where the policy network can be decomposed into a shared subnetwork and environment embedding layer. 
\cite{xionglinear} studied personalized TD learning with two timescales and established a linear convergence speedup. %

\section{Problem Setup and Preliminaries}
\label{sec: setup}
\noindent {\bf Multi-agent Reinforcement Learning.}
We consider a multi-agent system that consists of one parameter server and $K$ agents. 
The agents are modelled as Markov Decision Processes (MDP) $\{\calM^k \mid \calM^k = < P^k, \calS, \calA, R > \text{ for } k=1,2,...,K \}$, where $\calS$ is the (possibly infinite) state space, 
$\calA$ is a finite action space, $P^k$'s are the transition kernels, and $R$ is the reward function: $\calS\times\calA\rightarrow [-U_r,U_r]$.
The environments $P^k$ with which the agents interact are heterogeneous.  
The parameter server can orchestrate the learning at the agents through iteratively aggregating the updates at the agents.

We study average-reward reinforcement learning \citep{tsitsiklis1999average}, which provides a principled framework for maximizing the mean reward. Even in the single-agent setting, unlike discounted and finite-horizon MDPs, the sample complexity of average-reward settings remains largely unexplored \citep{jiao2026sample}. In this paper, we consider the more challenging multi-agent setting. In particular, we study personalized multi-agent average-reward TD learning.

Fix a policy $\pi$. We assume that, for each agent $k\in [K]$, the limiting state distribution induced by the transition kernel $P^k$ under policy $\pi$ exists; we denote this distribution by $\mu^k$.  
The time-average expected reward at agent $k$ is defined as: 
\begin{equation}
\label{eqn: J theta def} 
\begin{aligned}
J^k(\pi) &:= \lim_{T\rightarrow\infty} \frac{1}{T}\bbE\qth{\sum_{t=0}^{T-1}R(s^k_t,a^k_t)} \\
& ~ ~ =\underset{(s, a)\sim\mu^k\otimes\pi}{\mathbb{E}}[R(s,a)],     
\end{aligned}
\end{equation}
where $(s^k_t, a^k_t)$ is the state-action pair of agent $k$ at iteration $t$, for $t=0, 1,\cdots$.  
The first expectation in Eq.\,(\ref{eqn: J theta def}) is taken over the state-action trajectories generated under the Markov chain, while the second is taken over their limiting distribution. Here, the notation $\mu^k\otimes \pi$ denotes sampling $s\sim \mu^k$ and $a\sim \pi(\cdot \mid s)$. 

The value function is defined as 
\begin{align}
\label{eq: value k}
V^{k,\pi}(s) = \bbE \qth{\sum_{t=0}^\infty R(s^k_t, a^k_t) - J^k(\pi) \mid s_0 = s}. 
\end{align}
For ease of exposition, we drop the index and function argument in $\pi$. That is, for each $k$, $J^k(\pi)$ is written as $J^k$, and $V^{k,\pi}(s)$ as $V^k(s)$ for short. 

\noindent {\bf Single-agent TD Learning.} 
Let $\calV$ be a function space that $V^{k}$ belongs to. 
Let $\phi: \calS \to L^2(\naturals)$ be a feature map of $\calV$.  
Denote the coefficients of the value $V^{k}$ with respect to $\phi$ as $z^{k,*}$, where $z^{k,*} \in L^2(\naturals)$. For simplicity, we assume $\phi$ is known and has finite dimension $d$, i.e., $\phi: \calS\to \reals^d$. Consequently, 
\[
V^k (s) = \iprod{\phi(s)}{z^{k,*}}, ~~~~ \forall ~ s\in \calS, \text{and } k \in [K].  
\]

We consider the following linear function approximation: %
\[
\hat{V} (s; z^k) = \phi(s)^{\top} z^k. 
\]
To drive the approximation $\hat{V} (s; z^k)$ to $V^{k}$, 
single-agent TD learning (particularly TD(0)) 
update is applied at an agent to learn its value function: 
\begin{equation}
\label{eq: TD single-agent} 
\begin{aligned}
z^k_{t+1} &= z^k_t + \beta \bigg(R(a_t^k, s_t^k) - \eta_t^k \\
&\qquad\qquad+ \pth{\phi(s^k_{t+1}) - \phi(s^k_t)}^{\top} z^k_t\bigg) \phi(s_t^k),\\
\eta_{t+1}^k & = \eta_{t}^k + \gamma \pth{R(a_t^k, s_t^k) - \eta_t^k},  
\end{aligned}  
\end{equation}  
where $\beta$ and $\gamma$ are stepsizes, and $\eta_t^k$ is the reward estimator. 
In literature, $\delta_t^k := R(a_t^k, s_t^k) - \eta_t^k + \pth{\phi(s^k_{t+1}) - \phi(s^k_t)}^{\top} z^k_t$ is often referred to as TD(0) error. 
It is well-known that under some standard technical assumptions $\lim_{t\diverge}z^k_{t} \to z^{k,*}$ \citep{tsitsiklis1999average,sutton2018reinforcement}.

\noindent {\bf Multi-agent TD Learning: Commonality and Heterogeneity.}  
In multi-agent TD learning, the $K$ agents aim to collaboratively learn their respective value functions under the given $\pi$ based on their collective state-action trajectories.  
Learning with shared representation is widely recognized as an effective way to separate commonalities from heterogeneity across various heterogeneous data \cite{caruana1997multitask,ando2005framework,bengio2013representation,lecun2015deep,niu2024collaborative}.  

Inspired by the literature of personalized FL and multi-task learning, 
following  \cite{tripuraneni2021provable,collins2021exploiting,thekumparampil2021statistically,duchi2022subspace,du2021fewshot,tian2023learning,niu2024collaborative},  
we assume that $\{z^{k,*}: ~ k=1, \cdots, K\}$ belongs to and fully span an $r$-dimensional subspace of $\reals^d$.  
That is, there exists an orthonormal matrix $\mathbf{B}^*\in \reals^{d\times r}$ such that
\begin{align}
\label{eq: low dimensional rewrite}
z^{k,*} = \mathbf{B}^* \omega^{k,*}, ~~~ \text{where} ~~ \omega^{k,*}\in \reals^r. 
\end{align}
To see the generality of this formulation: $r=1$ when $z^{k,*} = z^{k^{\prime},*}$ for all $k, k^{\prime}$ (i.e., homogeneous transitional kernels $P^k$); $r=K$ when $z^{k,*} \perp z^{k^{\prime},*}$ for all $k^{\prime}\not=k$ -- which is possible when $K\le d$.

\section{Algorithm: PMAAR-TD}
\label{sec: alg}
We propose a personalized multi-agent average-reward TD learning algorithm ({\bf PMAAR-TD}) based on joint linear approximation, in which agents iteratively estimate the underlying truth subspace $\bfB^*$ and their respective local ``heads'' $\omega^{k,*}$. 
Specifically, $\{z^k = \bfB \omega^k\}_{k\in [K]}$ are decomposed into the product of the subspace estimate $\bfB\in \reals^{d\times r}$, which is common over all agents, and an agent-specific head $\omega\in \reals^r$. %

A formal description can be found in Algorithm \ref{alg: fedac-per}. At a high level, in each iteration, each agent locally performs the TD($L$) update, and uses the obtained TD($L$) error $\delta_{t, L}^k$ -- formally defined in Eq.\,(\ref{eq: L step TD error}) -- and rewards to update estimates of its local head $\omega^k$, common subspace $\bfB$, and local reward $\eta^k$.
The updates of $\omega^k$ and $\bfB$ can be interpreted as performing a perturbed stochastic gradient on the following joint linear approximation problem: 
\begin{equation}
\label{eq: joint approximation}   
\begin{aligned}
&\min_{\bfB, \omega^1, \cdots, \omega^K} ~ \frac{1}{2}\sum_{k=1}^K \mathbb{E}_{\mu^k}[(\hat{V}^k(s;\omega^k,\mathbf{B}) - V^{k}(s))^2],\\
& \quad \quad ~~~ \text{where} ~~ \hat{V}^k(s;\omega^k,\mathbf{B}) =  \phi(s)^\top\mathbf{B}{\omega^k}.    
\end{aligned}
\end{equation}
A natural stochastic gradient descent of Eq.\,(\ref{eq: joint approximation}) updates of $\bfB$ and $\omega^k$ for all $k$ are:  
\begin{align*}
& \omega_{t+1}^k =\omega_{t}^k+ \beta \qth{V^{k}(s_{t,0}^k) - \phi(s_{t,0}^k)^\top {\mathbf{B}_t}{\omega_t^k}}{\mathbf{B}_t}^\top \phi(s_{t,0}^k), \\
& \mathbf{B}_{t+1}^k =\mathbf{B}_{t}^k+ \zeta \qth{V^{k}(s_{t,0}^k) -  \phi(s_{t,0}^k)^\top {\mathbf{B}_t}{\omega_t^k} }\phi(s_{t,0}^k) (\omega_t^k)^\top, 
\end{align*}
where $\beta$ and $\zeta$ are stepsizes of $\omega^k$ and $\bfB$, respectively.  
However, $V^k$ is unknown for each $k$. Following the semi-gradient method TD($L$), we use \begin{align*}
    &R(s_{t,0}^k, a_{t,0}^k)-J^k +R(s_{t,1}^k, a_{t,1}^k)-J^k +...\\
    &\qquad+ R(s_{t,L-1}^k, a_{t,L-1}^k)-J^k+\hat{V}^k(s_{t,L}^k;\omega_t^k,\mathbf{B}_t)\\
    =&\sum_{\ell=0}^{L-1} R(s_{t,\ell}^k, a_{t,\ell}^k)-J^k+\phi(s_{t,L}^k)^\top {\mathbf{B}_t}{\omega_t^k}
\end{align*} to estimate $V^{k}(s_{t,0}^k)$. %
The $L$-step TD error at each agent is 
\begin{align}
\label{eq: L step TD error}  
   \delta_{t,L}^k = \sum_{\ell=0}^{L-1} r_{t,\ell}^k-\eta_t^k+(\phi(s_{t,L}^k)  - \phi(s_{t,0}^k))^\top {\mathbf{B}_t}{\omega_t^k},
\end{align}
where $r_{t,\ell}^k = R(s_{t,\ell}^k, a_{t,\ell}^k)$. %
Thus, 
\begin{align}
\label{eqn:update rule for omega L step}
& \tilde{\omega}_{t+1}^k =\omega_{t}^k+ \beta \delta_{t,L}^k{\mathbf{B}_t}^\top \phi(s_{t,0}^k), \\
\label{eqn:update rule for B L step}
& \mathbf{B}_{t+1}^k =\mathbf{B}_{t}+ \zeta  \bfB_{t,\perp}\bfB_{t,\perp}^\top\delta_{t,L}^k\phi(s_{t,0}^k) (\omega_t^k)^\top. 
\end{align}
\begin{algorithm}[tb]
\caption{PMAAR-TD}
\label{alg: fedac-per}
\begin{algorithmic}[1]
\STATE \textbf{Input:} Initial heads $\{\omega_0^k\}_{k=1}^K$;  initial subspace $\bfB_0\in \reals^{d\times r}$, which is orthonormal; three stepsizes: $\beta$ for local heads, $\zeta$ for subspace, and $\gamma$ for reward estimator; total iteration $T$; projection upper bound $U_{\omega}$ such that $\|\omega^{k,*}\| \le U_{\omega}$ for $k\in [K]$.   
\FOR{$k= 1, 2, \dots, K$}
\STATE $\eta_0^k = 0$\; 
\STATE Draw $s_{0,0}^k$ arbitrarily from $\calS$\; 
\ENDFOR 
\FOR{$t = 0, 1, 2, \dots, T-1$}
    \FOR{$k = 1, 2, \dots, K$}
        \STATE $\delta_{t, L}^k = \sum_{\ell=0}^{L-1} (r_{t,\ell}^k - \eta_t^k) + (\phi(s_{t,L}^k) - \phi(s_{t,0}^k))^\top \mathbf{B}_t \omega_t^k$.
        
        \STATE $\omega^k_{t+1} = \Pi_{U_{\omega}} \left( \omega^k_t + \beta\frac{1}{L} \delta_{t, L}^k \mathbf{B}_t^\top \phi(s_{t,0}^k) \right)$.
        \STATE $\mathbf{B}_{t+1}^k = \mathbf{B}_t + \zeta \mathbf{B}_{t,\perp}\mathbf{B}_{t,\perp}^\top \frac{1}{L}\delta_{t, L}^k \phi(s_{t,0}^k)(\omega^k_t)^\top$.
        
        \STATE $\eta^k_{t+1} = \eta^k_t + \gamma \left( \frac{1}{L} \sum_{\ell=0}^{L-1} r_{t,\ell}^k - \eta^k_t \right)$.
        \STATE $s_{t+1,0}^k = s_{t,L}^k$.
    \ENDFOR
    
    \STATE $\bar{\mathbf{B}}_{t+1} = \frac{1}{K} \sum_{k=1}^K \mathbf{B}_{t+1}^k$.
    \STATE $\mathbf{B}_{t+1}, \mathbf{R}_{t+1} = \texttt{QR}(\bar{\mathbf{B}}_{t+1})$.
    \STATE Server broadcasts $\mathbf{B}_{t+1}$ to all agents.
\ENDFOR
\end{algorithmic}
\end{algorithm}
\noindent Nevertheless, Eq.\,(\ref{eqn:update rule for omega L step}) and Eq.\,(\ref{eqn:update rule for B L step}) are not the final updates to obtain $\omega_{t+1}^k$ and $\bfB_{t+1}$. 
Several key algorithmic components  
are incorporated to provide fine-grained control of the perturbation and to enforce contraction of the principal-angle distance between $\bfB_t$ and $\bfB^*$ -- addressing major technical challenges. \\
(A)  Projection operation on local heads $\prod_{U_{\omega}}$: This operation project $\omega$ onto a convex ball with radius $U_{\omega}$. Unlike existing fine-grained analyses, such as in \cite{khodadadian2022federated}, the dynamics of $\bfB$ and $\omega^k$ (one for each agent), together with $\eta_t^k$, are highly coupled due to the personalization. Without relying on any absolute bound on $\omega^k$, characterizing the error dynamics quickly becomes intractable. Despite this technical convenience, it introduces ``not so small'' perturbation that interplays with the errors of other dynamics nontrivially that can be harmful to the convergence and potential speedup.  \\
(B) Subspace projected innovation in Eq.\,(\ref{eqn:update rule for B L step}): Instead of let $\bfB_{t+1}^k$ towards the local innovation $\delta_{t,L}^k\phi(s_{t,0}^k) (\omega_t^k)^\top$, we move $\bfB$ in the direction of the ``residual'' innovation that lies outside of $\bfB_t$. This projection significantly mitigates some otherwise amplified perturbations arising from multiple cross terms. For example, under this projection, the perturbation in the QR decomposition is only second order; see Lemma \ref{lm: perturbation QR}. \\
(C) QR: Under this decomposition, $\bfB_t$ is orthonormal for each $t$, which is a key structural property that contributes to the contraction of the principal-angle distance between $\bfB_t$ and $\bfB^*$ \citep{collins2021exploiting}.

It is worth noting that \cite{xionglinear} does not use the above algorithmic details to control the perturbation. 
This is because that the challenge of single-timescale (of $\bfB$ and $\omega^k$) convergence analysis for this problem is that the estimation errors of the critic and the subspace misalignment are strongly coupled, whereas in double-loop or two-timescale analysis, these error terms are usually naturally decoupled.

\section{Convergence Analysis} 
\label{sec: main convergence}
We present some technical assumptions adopted in our convergence analysis. 
Assumptions \ref{ass: per-agent full exploration}, \ref{assp: uniform ergodicity}, and \ref{ass: bounded features} are widely adopted in the RL literature. 
Assumption \ref{assp: eigval ZZ^T} quantifies how well spread the underlying truth $z^{k,*}$ is in covering the $r$-dimensional subspace of $\bfB^*$.

\begin{assumption}[Exploration]
\label{ass: per-agent full exploration} 
The matrix 
\[
A_L^k = \mathbb{E}[\phi(s^{(0)})(\phi(s^{(L)})-\phi(s^{(0)}))^\top],
\]
where the expectation is taken with respect to $s^{(0)}\sim\mu^k, (a^{(\ell)}, s^{(\ell+1)}) \sim \pi \otimes P^k ~ \forall \ell\in [0, L-1]$, is negative definite and its largest eigenvalue is upper bounded by $-L\lambda$. 
\end{assumption} 

This assumption is widely adopted in analyzing TD learning with linear function approximation \cite{chen2024finite,bhandari2018finite,chen2021closing,olshevsky2023small,qiu2021finite,wu2020finite,zou2019finite}.  
Intuitively, $A_L^k$ captures the exploration of the policy $\pi$ under the transition kernel $P^k$ within $L$ steps. For the tabular setting, Assumption \ref{ass: per-agent full exploration} holds when the policy $\pi$ explores all state-action pairs with $L$ steps. %

\begin{assumption}[Uniform ergodicity]
\label{assp: uniform ergodicity}
Let $\bbP_{0:\tau}^k(\cdot|s_0^k=s)$ denote the state distribution after $\tau$ steps given $s_0^k=s$ for agent $k$.  There exist $m>0$ and $\rho\in(0,1)$, such that at all $k\in [K]$,
    \begin{equation}
        d_{TV}(\mathbb{P}_{0:\tau}^k( \cdot|s_0^k=s),\mu^k(\cdot))\leq m\rho^\tau,~ \forall \tau\geq0,  s\in \calS, 
    \end{equation}
    where $d_{TV}$ is the total variation distance. 
\end{assumption} 
In our analysis, with a little abuse of notation, we also use $\bbP_{0:\tau}^k(\cdot, \cdot|s_0^k=s)$, which denotes the state-action pair at time $\tau$ given $s_0^k = s$.  
Recall that $\mu^k$ is the limiting distribution induced by policy $\pi$ and the transition kernel $P^k, \forall k$. Uniform ergodicity is often assumed to characterize the noise induced by Markovian sampling \citep{chen2024finite,olshevsky2023small,zou2019finite,wu2020finite}.

Furthermore, we assume bounded features. 
\begin{assumption}
\label{ass: bounded features}
$\|\phi(s)\| \leq 1$ for each $s\in \calS$.     
\end{assumption} 
This assumption is rather standard in RL literature \cite{sutton1999policy,sutton2018reinforcement,chen2024finite,olshevsky2023small,zou2019finite,wu2020finite}.  
The existence of a shared common subspace helps to reduce the local learning problem from a $d$-dimensional problem to an $r$-dimensional problem. In PFL, when $r\ll d$, this benefits are well characgterized and are indeed substantial in derived bounds \cite{collins2021exploiting,niu2024collaborative}. Under Assumption \ref{ass: bounded features}, however, the dependence on $d$ is masked, potentially obscuring these gains in the resulting bounds. Adopting more fine-grained assumptions could therefore lead to significantly sharper results. Developing such refinements, however, would require a fundamental rethinking of the problem and is left for future work.

Let 
\begin{align}
\label{eq: collective weights}
\bfZ^* \in \reals^{d\times K}, ~~~ \text{with }z^{k,*} ~~ \text{as the $k$-th column}. 
\end{align}
By Eq.\,(\ref{eq: low dimensional rewrite}), we know that the rank of $\bfZ^*\bfZ^{*\top}$ is $r$. We require that each of the $r$ dimensions of $\bfB^*$ is well-covered by $\{z^{k,*}\}_{k=1}^K$, formally stated in Assumption \ref{assp: eigval ZZ^T}.  
\begin{assumption}
\label{assp: eigval ZZ^T}
    $\frac{1}{K}\lambda^+_{\min}(\bfZ^*\bfZ^{*\top})  \geq \alpha $ for some $\alpha >0$.
\end{assumption}
In our analysis, we also use $\lambda_{\min}^+  = \lambda_{\min}^+(\bfZ^*\bfZ^{*\top})$ to denote the smallest nonzero eigenvalue of $\bfZ^*\bfZ^{*\top}$.

\subsection{Main Convergence Results}
\label{subsec: error iterates}
In Algorithm \ref{alg: fedac-per}, three groups of variables are updated: the common subspace estimate $\bfB_t$, the local heads $\omega_t^k$, and estimates of local time-averaged rewards (defined in Eq.\,(\ref{eqn: J theta def})). In this subsection, we characterize their evolutions over time. For ease of exposition, we introduce a set of notation: 
\begin{align}
x_t^k &= \bfB_t\omega_t^k-z^{k,*}, ~  m_t = \bfB^{*\top}_\perp \bfB_t,   ~  y_t^k = \eta_t^k-J^k \label{eq: main iterates},\\
& \qquad \qquad \tilde\omega_t^k = \omega^k_t + \beta \delta_{t,L}^k \mathbf{B}_t^\top \phi(s_{t,0}^k), \label{eq: intermediate variables 0}\\ 
 &\tilde x_t^k = \bar\bfB_t \tilde\omega_t^k-z^{k,*}, \quad\quad  \bar m_t = \bfB^{*\top}_\perp \bar\bfB_t.
\label{eq: intermediate variables}
\end{align}

Fix $\tau$. 
For any $t\ge \tau$, define 
\[
M_t = \bbE_{t-\tau}\|m_{t}\|_F^2, \quad  \bar X_t = \frac{1}{K}\sum_{k=1}^K \bbE_{t-\tau}\|x_t^k\|^2,
\]
where $\bbE_{t-\tau}$ is the conditional expectation, conditioning on the randomness up to time $t-\tau$. 
\begin{lemma}[(informal) Upper bound of local head errors ]
\label{lm: local head: upper bound} 
Let $\bfP^* = \bfB^* (\bfB^*)^{\top}$, and $\bfP_t = \bfB_t \bfB_t^{\top}$ for each $t$. %
Choose $\tau = \lceil 2\log_\rho(\zeta)\rceil$, $U_{\delta} U_{\omega} \zeta\le \frac{1}{2}$, where $U_{\delta} = 2U_r + 2U_{\omega}$, $\beta = c\zeta$ for arbitrary $c>0$ that can be tuned. There exist arbitrary positive constants $c_3, c_4, c_5, c_6$ such that  
\begin{align*}
    &\bar X_{t+1}  \nonumber
   \leq \Bigg(1-2\lambda c\zeta
   +\Big(c_3^2\frac{c }{L} 
+ c_4^2 c  
+ c_5^2\frac{1}{L}
+c_6^2 +4U_\omega^4\frac{1}{c_5^2L}\Big)\zeta \\  
& \qquad \qquad \qquad \qquad \qquad +\calC_{X,1}(\tau^2)\zeta^2
+ \calC_{X,2}(\tau^2)\zeta^3
\Bigg)\bar X_t\nonumber\\
&\qquad+ \pth{\frac{36U_\omega^2}{c_3^2}\frac{c\zeta}{L}  +6U_\omega U_\delta c \zeta } M_t
\nonumber\\
&\qquad+ \Bigg(
\Big[
\frac{1}{c_4^2}c 
+\frac{U_\omega^4}{c_6^2}
\Big]\zeta
+\calC_{X,3}(\tau^2)\zeta^2
+\calC_{X,4}(\tau^2) \zeta^3
\Bigg)\bar Y_t\nonumber\\
&\qquad+\calC_{X,5}(\tau^2)\frac{\zeta^2}{L}
 +\calC_{X,6}(\tau^2)\frac{\zeta^2}{K} 
+ O(\zeta^3) + O(\gamma^2\zeta^2), 
\end{align*}
where
$\calC_{X,1}(\tau^2), \calC_{X,2}(\tau^2), \calC_{X,3}(\tau^2), \calC_{X,4}(\tau^2), \calC_{X,5}(\tau^2), $ and $\calC_{X,6}(\tau^2)$ are on the order of $\tau^2$, formally defined in Appendix \ref{sec: proof of local head upper bound}. 
\end{lemma}

Intuitively, when $M_t \lesssim \bar X_t$. The above upper bound is similar to the traditional single-agent setting \cite{chen2024finite}.

Deriving Lemma \ref{lm: local head: upper bound} is highly nontrivial and requires a fundamental detour from the existing analysis \cite{chen2024finite}. Furthermore, existing analysis in PFL \cite{collins2021exploiting} is not applicable to our problem due to the Markovian sampling, the TD updates, and the lack of responses.  
Roughly speaking, the main drift of the decay in $\|x_{t}^k\|$ may arise from $\langle x_{t}^k, \bfB_t\bfB_t^\top\phi(s_{t,0}^k)\delta_{t,L}^k\rangle$. 
However, the environmental heterogeneity significantly complicates the characterization of this term. Specifically, %
\begin{align*}
\langle x_{t}^k, \bfB_t\bfB_t^\top\phi(s_{t,0}^k)\delta_t^k\rangle 
= \langle x_{t}^k, \bfP_t  A_L^kx_t^k \rangle + \text{``perturbation''}.
\end{align*}      
Zooming into this expression, the main negative drift arises from the term $\langle x_{t}^k, \bfP_t  A_L^kx_t^k \rangle$. Intuitively, in traditional single-agent or homogeneous environment settings, $\langle x_{t}^k, \bfP_t  A_L^kx_t^k \rangle \approx \langle x_{t}^k, A_Lx_t^k \rangle$, which is mainly controlled by the spectrum of $A_L$, with the desired property directly assumed in Assumption \ref{ass: per-agent full exploration}. However, in the presence of environmental heterogeneity, Assumption \ref{ass: per-agent full exploration} does not directly guarantee a negative drift of the $x_t^k$ due to the existence of $\bfP_t$ and its intricate interplay with the heterogeneous $A_L^k$. 
Specifically, (1) $\bfP_t$ is not of full rank, (2) the local head error will be distorted in a different manner due to the product $\bfP_t A_L^k$, and (3) $\bfP_t$ varies over time.

We further show that it is impossible for $M_t \gg \bar X_t$. In fact, a slightly stronger result holds.   
\begin{lemma}[Lower bound of local head errors]
\label{lm: local head: lower bound}
Suppose that $d\ge 2r$. It holds that 
\begin{align*}
    &\frac{1}{K}\sum_{k=1}^K \|x_t^k\|^2\geq \frac{\|m_t\|^2}{K}\lambda^+_{\min}(\bfZ^*\bfZ^{*\top}). %
\end{align*}
\end{lemma}  
Intuitively, Lemma \ref{lm: local head: lower bound} says that when the principal angle distance between $\bfB_t$ and $\bfB^*$ is large, the well coverage of all the $r$ directions by the underlying truth $\{z^{1,*}, \cdots, z^{K,*}\}$ ensures that the aggregate errors in the local head estimates remain bounded away from zero. 
In other words, the local heads cannot be learned significantly faster than the subspace itself.

\begin{lemma}[Principal angle distance analysis]  
\label{lm: PAD analysis}
Choosing $\beta = c\zeta$ for some $c>0$, and $\zeta$ so that $U_\delta^2 U_\omega^2\zeta^2\leq \frac{1}{6}$, when $d\ge 2r$, there exist arbitrary positive $c_1, c_2, c_3, c_4$ such that 
\begin{align*}
        &M_{t+1} \leq (1-\frac{\lambda\alpha \zeta}{r}+ c_1^2\zeta+c_2^2\zeta) M_t\nonumber \\
        &
        +\pth{\frac{32U_\omega^2\zeta^2}{L^2}+\frac{16U_\omega^2}{c_2^2}\frac{\zeta}{L^2}+4c_3^2\frac{\zeta}{L^2} }\bar X_{t}\\
        &
        +\pth{8U_\omega^2\zeta^2+\frac{4U_\omega^2}{c_1^2}\zeta+4c_4^2\zeta}\bar Y_t \nonumber + \calC_{M,1}\frac{\zeta^2}{K}   
        + \calC_{M,2} \zeta^6 
        \\
        &+\calC_{M,3}\zeta^4+\calC_{M,4}\zeta^3 
    + 4C_{15}\tau C_5c\frac{\zeta^2}{\sqrt L}
        \\
        &+ 4C_{14}\tau C_{10}\frac{\zeta^2}{\sqrt K} 
    + 4C_{19}(\tau)\gamma\zeta^2,
\end{align*}
where the missing specification of the (nearly) constant terms can be found in the analysis. They either are problem-specific constants or depends only on problem-specific consants and $\tau$.  
\end{lemma}

\begin{lemma}[Reward analysis]
\label{lm: reward analysis}
Fixed a $\tau$. 
For $t>\tau$, 
\begin{align*}
&\bbE(y_{t}^k)^2 
    \leq  (y_0^k)^2  e^{-2\gamma t}+  \frac{\gamma}{2-\gamma} 4U_r^2+8U^2_r m\rho^{\tau}(1-\gamma)^\tau \\
    & \qquad +8U_r^2 \tau \gamma(1-\gamma)^\tau
    +8U_r^2 \frac{\tau(\tau-1)}{2}\gamma^2 \\
    & \qquad + 4 |y_0^k| U_r e^{\gamma \tau}e^{-\gamma 2t} + 4U_r(1-\gamma)^{T} |y_0^k| m\rho^{\tau}. 
\end{align*}
Choosing $\tau \le \min\{\frac{2T}{\log(T)},\lceil\log_\rho(\gamma)\rceil, \lceil2\log_\rho(\zeta)\rceil\}$ and $\gamma = \frac{\log(T)}{2T}$,  the average reward can be bounded as  
\begin{align*}
\bbE(y_{T}^k)^2 
\leq \frac{28 U_r^2}{T} + \frac{8U_r^2 \tau \log(T)}{T} + \frac{8U_r m\log(T)}{T} .   
\end{align*}
\end{lemma}

Lemma \ref{lm: reward analysis} says that reward estimation error decays at the rate of $\tilde \calO(\frac{1}{T})$, where $\tilde \calO$ hides the polylog factor of $T$. It is worth noting that in Algorithm \ref{alg: fedac-per} the reward learning at different clients are independent of each other, and the dynamics of $\bfB$ and local heads $\omega^k$ do not affect $y^k$.  
Though there is no averaging over agents, the reward learning itself can be made faster than the learning of subspace and local heads. With minor changes, the same analysis of Lemma \ref{lm: reward analysis} can be extended to the conditional expectation, such as $\bbE_{t-\tau}(y_{t}^k)^2$ for sufficiently large $t$.

\begin{theorem}[Convergence (informal)]
\label{thm: 1}
    Consider Algorithm \ref{alg: fedac-per} with $\beta = c\zeta,\zeta = \frac{4r}{\sqrt{T}\lambda\alpha}, \gamma = \frac{\log(T)}{2T}$, where $c>0$ and satisfies $\frac{\alpha}{4r}\leq c\leq\frac{\alpha}{2r}$, $K>0$ satisfies  $\max\sth{\frac{30}{\lambda K}, \frac{120U_\omega^4}{\lambda K}} \leq \frac{\alpha}{2r}, L=K.$ Let $\calV_{T} =\bar X_{T}+\kappa M_{T}$ be the Lyapunov function, where $\kappa \geq \frac{2(1080U_\omega^2+6U_\omega U_\delta \lambda L^2)}{\lambda^2  L^2} = \Theta(1)$.
    Let $\tau = \lceil 2\log_\rho(\zeta)\rceil$. 
    We have for $T\geq 2\tau$, and for sufficiently large $K$ 
    \begin{align*}
      \calV_{T} = \tilde{O}(\frac{1}{\sqrt{(TK)K}}) + \calO(\frac{\log(T)}{T}),
    \end{align*}
    where $\tilde \calO$ hides the polylog factors.  
\end{theorem}
When $T$ is sufficiently large, the first term dominates, which is a linear speedup. The threshold of sufficiency of ``large'' $L$ or $K$ (recalling that we choose $K=L$) can be found in Eq.\,(\ref{eqn: c feasible 2}), which is independent of $T$.

\section{Experiments}
\label{sec: experiments}
\subsection{Numerical  Experiments on Prediction Problems}
In this section, we evaluate the performance of {\bf PMAAR-TD} for approximating value functions under a fixed policy. We compare our proposed approach, {\bf PMAAR-TD}, against the baselines introduced in Sections \ref{sec: setup}: {\bf Single TD} and {\bf FedTD-Uniform}. {\bf Single TD} is the standard single-agent TD learning; {\bf FedTD-Uniform} periodically averages the parameters of all agents and outputs a universal value function for all agents across different environments. The following is the specifications of how to evaluate the performance of value function approximation:

\begin{enumerate}
    \item A pre-trained single-agent TD model, trained with standard environment parameters, is employed as the fixed policy under which the value function is evaluated. 
    \item A fixed initial state is selected and consistently applied across all methods throughout the entire duration of the experiment.
    \item Value function convergence is monitored by periodically evaluating the preselected initial state via a forward pass, utilizing the resulting value estimate as a proxy for training progress. We adopt this scheme because the state space is continuous, and the design ensures that the preselected state appears in every training episode.
\end{enumerate}

The value of the preselected state is recorded during training serves as a metric to quantify the convergence rate toward the optimal value function for each approach. For the experiments, we utilize {three distinct seeds} and a synchronization rate of ${ E = 50}$ as shared hyperparameters. The specific configuration of the Acrobot environment for each agent is detailed below:

\begin{table}[h]
\centering
\label{tab:fixed_policy_acro_setup}
\resizebox{\columnwidth}{!}{
\begin{tabular}{lcccccc}
\toprule
Parameters & Agent 0 & Agent 1 & Agent 2 & Agent 3 & Agent 4 & Agent 5  \\
\midrule
link length 1 & 2.2 & 1.5 & 2 & 3 & 2.1 & 1.5 \\
link length 2 & 1 & 1 & 0.8 & 1 & 1 & 1 \\
\bottomrule
\end{tabular}
} 
\end{table}

We analyze the results presented in Figure \ref{fig: td comparison} from two primary perspectives:

\begin{enumerate}
    \item {\bf Convergence Speed:} Regarding the convergence speed of the value function (as indicated by the curves), {\bf PMAAR-TD} exhibits a significant acceleration compared to {\bf Single TD} across the majority of agents.
    \item {\bf Approximation Accuracy:} In terms of the reward (y-axis), it is evident that {\bf FedTD-Uniform} converges to a suboptimal value in most agent configurations. This discrepancy suggests a lack of adaptability to heterogeneous environments, whereas {\bf PMAAR-TD} maintains consistency with the optimal baseline (single agent).
\end{enumerate}
Evaluation of value functions is relatively simple when the policy is fixed and the corresponding true values are determined. Consequently, we consider control problems, where policies are learned and updated throughout training, leading to increased uncertainty in learning.

\subsubsection{Comparison with two-timescale setting}
We further evaluate {\bf PMAAR-TD} against the two-timescale reinforcement learning framework (see Figure~\ref{fig:2tsVSper}) within the same fixed-policy Acrobot environment. For this comparison, a global learning rate of $2 \times 10^{-4}$ is employed, while the critic layers are specifically configured with a learning rate of $2 \times 10^{-5}$. The results demonstrate that, consistent with theoretical expectations, the two-timescale approach exhibits significantly slower empirical convergence compared to our single-timescale setting, further validating the efficiency of our proposed synchronization mechanism.

\subsection{Application to Control Problems}
In this section, we extend {\bf PMAAR-TD} to the actor-critic framework by utilizing {\bf PMAAR-TD} as the critic component and incorporate an additional actor component. The action heads are also distinct across agents in our setting, i.e., agents do not share the parameters of those layers.

\subsubsection{Model Design}
 The architecture of the model employed in our experiments is defined as follows:
 
\begin{table}[h]
\centering
\resizebox{\columnwidth}{!}{
\begin{tikzpicture}[
    node distance=0.8cm and 1cm,
    font=\large,
    input/.style={circle, draw=black!80, thick, minimum size=1.4cm, fill=gray!5},
    mlp/.style={rectangle, draw=blue!70, fill=blue!5, thick, minimum width=2.4cm, minimum height=1.6cm, text centered, copy shadow={shadow xshift=2pt, shadow yshift=-2pt, fill=blue!20}},
    head/.style={rectangle, draw=black!80, fill=white, thick, minimum width=2.8cm, minimum height=1.2cm, text centered},
    arrow/.style={-Stealth, thick, draw=black!70}
]

\node (in) [input] {\textbf{$\mathbf{s}_t$}};

\node (affine) [mlp, right=of in] {MLP};

\node (actor) [head, fill=green!5, above right=0.4cm and 1.2cm of affine] {Actor Head};
\node (action) [right=0.8cm of actor] {$\pi(a_t|s_t)$};

\node (critic) [head, fill=yellow!5, below right=0.4cm and 1.2cm of affine] {Critic Head};
\node (value) [right=0.8cm of critic] {$V(s_t)$};

\draw [arrow] (in) -- (affine);

\draw [arrow] (affine.east) -- ++(0.5,0) |- (actor.west);
\draw [arrow] (affine.east) -- ++(0.5,0) |- (critic.west);

\draw [arrow] (actor) -- (action);
\draw [arrow] (critic) -- (value);

\node [below=0.2cm of in, font=\normalsize\itshape] {State};
\node [below=0.2cm of affine, font=\normalsize\itshape] {Feature Extraction};
\node [right=0.2cm of action, font=\normalsize] {Action};
\node [right=0.2cm of value, font=\normalsize] {Value};

\end{tikzpicture}
}
\label{fig:model_arch_large_font}
\end{table}

The fundamental distinction between \textbf{SingleAC} and \textbf{FedAC-Uniform} lies in the synchronization mechanism; specifically, \textbf{SingleAC} operates independently, whereas both \textbf{FedAC-Uniform} and {\bf PMAAR-AC} incorporate periodic synchronization. Furthermore, the key innovation of {\bf PMAAR-AC} is its selective synchronization strategy. During the synchronization phase, {\bf PMAAR-AC} updates only the shared representation layers, while retaining the actor and critic heads as personalized local parameters to preserve agent-specific policies.

\subsubsection{Environment Configurations}
Regarding the environment configurations, we find that the settings used in prior work are not sufficiently heterogeneous \cite{jin_federated_2022,xionglinear}, as a single uniform policy can perform well across all environments. We therefore propose a new design that introduces stronger heterogeneity. Specifically, we create a mirrored environment where the agent's action will be mapped to a reversed/unwanted actual action. When agents operate in normal and mirrored environments respectively, a shared policy cannot be effectively learned, as the optimal behaviors for one group are adversarial to the other.

In our experiments, we employ {three distinct random seeds} and set the synchronization interval $E = {50}$ as shared hyperparameters across all configurations. The environment-specific parameters are detailed as follows:

\begin{table}[h]
\centering
\caption{CartPole environment setup for each agent}
\label{tab:control_cartpole_setup}
\resizebox{\columnwidth}{!}{
\begin{tabular}{lcccccc}
\toprule
Parameters & Agent 0 & Agent 1 & Agent 2 & Agent 3 & Agent 4 & Agent 5  \\
\midrule
pole length & 0.5 & 1 & 0.5 & 1 & 0.5 & 1 \\
gravity & 9.8 & 15 & 9.8 & 15 & 9.8 & 15 \\
force & 10 & 10 & 10 & 10 & 10 & 10 \\
start position & $\frac{\pi}{6}$ & -$\frac{\pi}{6}$ & $\frac{\pi}{6}$ & -$\frac{\pi}{6}$ & $\frac{\pi}{6}$ & -$\frac{\pi}{6}$ \\
env & normal & mirrored & normal & mirrored & normal & mirrored \\

\bottomrule
\end{tabular}
} 
\end{table}

\begin{table}[h]
\centering
\caption{Acrobot environment setup for each agent}
\label{tab:control_acro_setup}
\resizebox{\columnwidth}{!}{
\begin{tabular}{lcccccc}
\toprule
Parameters & Agent 0 & Agent 1 & Agent 2 & Agent 3 & Agent 4 & Agent 5  \\
\midrule
link length 1 & 1 & 1 & 1 & 1 & 1 & 1 \\
link length 2 & 1 & 1 & 1 & 1 & 1 & 1 \\
env & normal & mirrored & normal & mirrored & normal & mirrored \\
\bottomrule
\end{tabular}
}
\end{table}

\subsubsection{Result}
We discuss the results presented in Figure \ref{fig:control_acro} and \ref{fig:control_cartpole} from three primary perspectives:
\begin{enumerate}
    \item {\bf Convergence Speed}: In terms of convergence velocity, it is evident that {\bf PMAAR-AC} outperforms both {\bf FedAC-Uniform} and {\bf Single AC} across the majority of agents in both Acrobot and CartPole environments. Compared with {\bf Single AC}, although {\bf PMAAR-AC} exhibits a relatively marginal performance lag during the initial exploration phase, it frequently surpasses the baselines in the mid-training stage and maintains its performance advantage until convergence.
    \item {\bf Reward}: Regarding the asymptotic reward obtained post-training, {\bf PMAAR-AC} demonstrates superior performance, outperforming other baselines across the majority of agents. Particularly in the Acrobot environment, the integration of additional data samples facilitated by shared feature synchronization enables the model to converge closer to the theoretical optimum compared to alternative methods.
    \item {\bf Training Stability}: Regarding the training stability—represented by the shaded regions reflecting the variance across different seeds—it is evident that for all agents across both environments, {\bf PMAAR-AC} exhibits a consistent narrowing of the confidence intervals as training progresses. Ultimately, our approach achieves the smallest shaded area, indicating significantly lower variance between individual runs and a more robust convergence toward the mean reward curve compared to the baselines.
\end{enumerate}

\begin{figure}[!ht]
    \centering
    \includegraphics[width=0.5\textwidth]{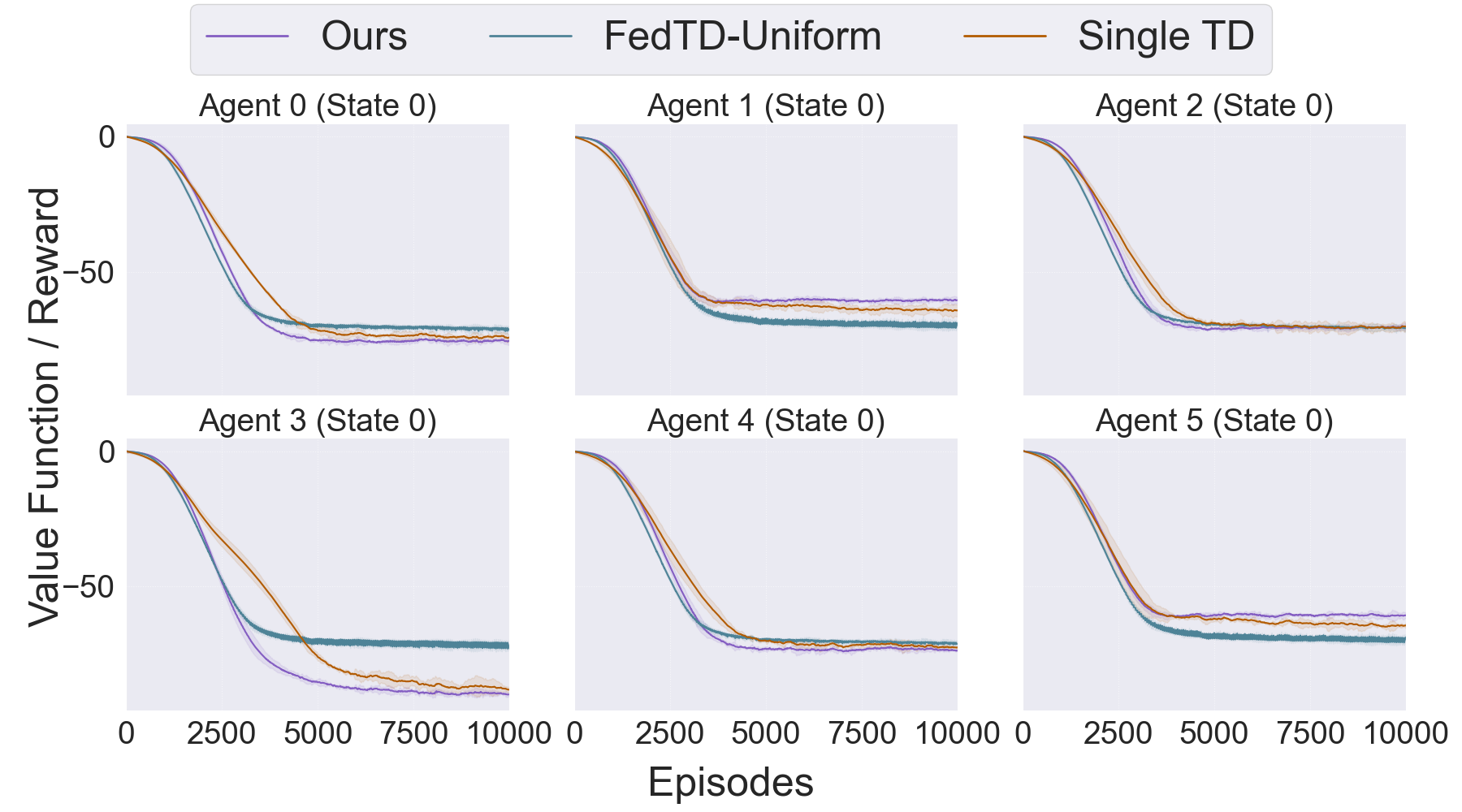}
    \caption{Value function convergence comparison on Acrobot.}
    \label{fig: td comparison}
\end{figure}

\begin{figure}[!ht]
    \centering
    \includegraphics[width=0.5\textwidth]{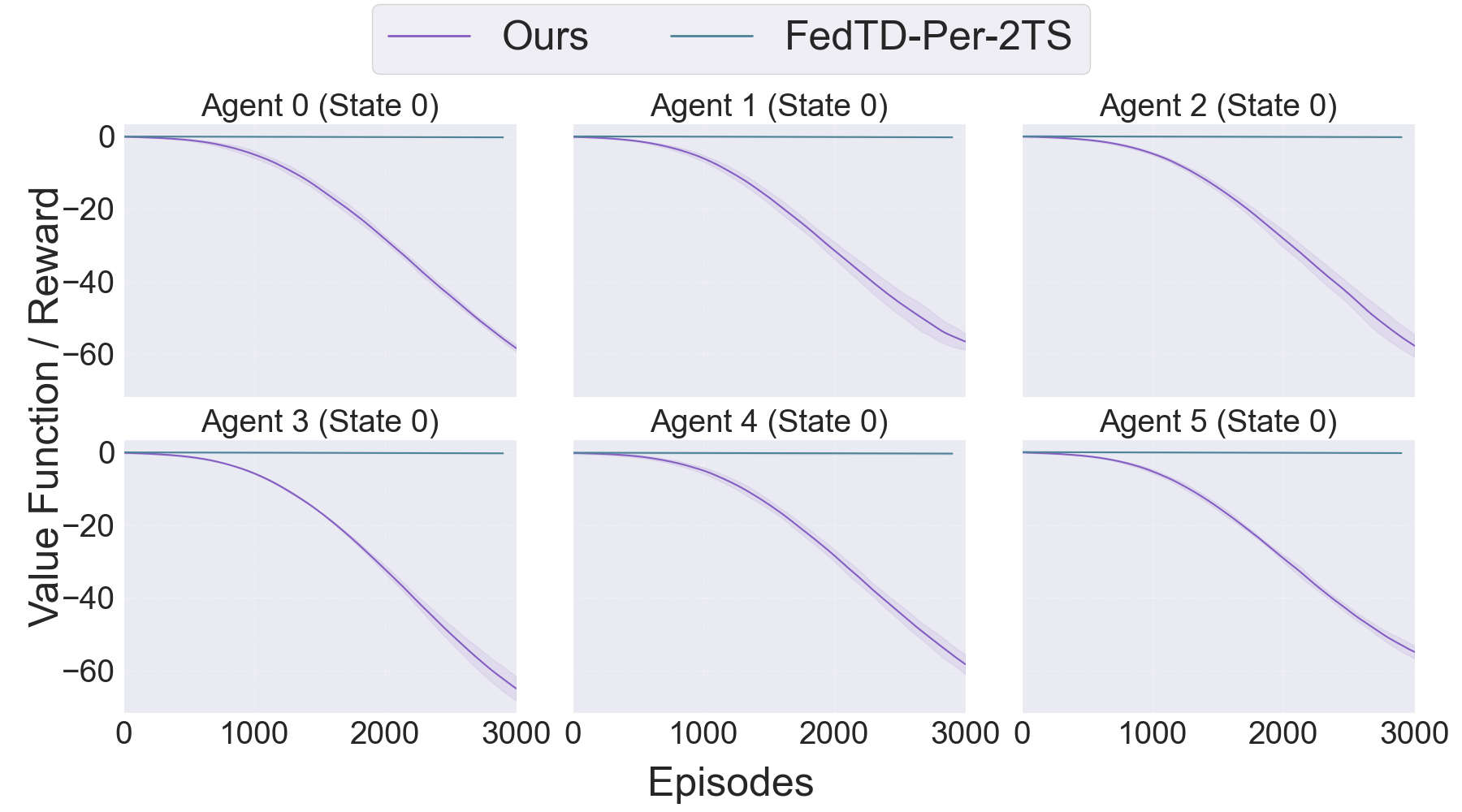}
    \caption{Value function convergence speed comparison between single-timescale and two-timescale settings on Acrobot.}
    \label{fig:2tsVSper}
\end{figure}

\section{Conclusion and Discussions}
\label{sec: dis & conclu}
In this work, we studied cooperative TD learning in heterogeneous environments, where agents share a low-dimensional common structure but face diverse local dynamics. By combining subspace estimation with personalized value function updates, we established finite-time convergence guarantees with linear speedup under Markovian sampling and clarified the limitations imposed by standard assumptions. 
We hope that our analysis framework and insights provide a foundation for future work on exploiting common structures in multi-agent reinforcement learning.

\newpage 
\section*{Impact Statement} 
This paper presents work whose goal is to advance the field
of Machine Learning. There are many potential societal
consequences of our work, none of which we feel must be
specifically highlighted here.

\bibliography{main}

@article{jiao2026sample,
  title={Sample Complexity of Average-Reward Q-Learning: From Single-agent to Federated Reinforcement Learning},
  author={Jiao, Yuchen and Woo, Jiin and Li, Gen and Joshi, Gauri and Chi, Yuejie},
  journal={arXiv preprint arXiv:2601.13642},
  year={2026}
}

@article{chen2021closing,
  title={Closing the gap: Tighter analysis of alternating stochastic gradient methods for bilevel problems},
  author={Chen, Tianyi and Sun, Yuejiao and Yin, Wotao},
  journal={Advances in Neural Information Processing Systems},
  volume={34},
  pages={25294--25307},
  year={2021}
}

@article{qiu2021finite,
  title={On finite-time convergence of actor-critic algorithm},
  author={Qiu, Shuang and Yang, Zhuoran and Ye, Jieping and Wang, Zhaoran},
  journal={IEEE Journal on Selected Areas in Information Theory},
  volume={2},
  number={2},
  pages={652--664},
  year={2021},
  publisher={IEEE}
}

@article{tian2023learning,
  title={Learning from similar linear representations: Adaptivity, minimaxity, and robustness},
  author={Tian, Ye and Gu, Yuqi and Feng, Yang},
  journal={Journal of Machine Learning Research},
  volume={26},
  number={187},
  pages={1--125},
  year={2025}
}

@inproceedings{
du2021fewshot,
title={Few-Shot Learning via Learning the Representation, Provably},
author={Simon Shaolei Du and Wei Hu and Sham M. Kakade and Jason D. Lee and Qi Lei},
booktitle={International Conference on Learning Representations},
year={2021}
}

@article{duchi2022subspace,
  title={Subspace recovery from heterogeneous data with non-isotropic noise},
  author={Duchi, John C and Feldman, Vitaly and Hu, Lunjia and Talwar, Kunal},
  journal={Advances in Neural Information Processing Systems},
  volume={35},
  pages={5854--5866},
  year={2022}
}

@inproceedings{tripuraneni2021provable,
  title={Provable meta-learning of linear representations},
  author={Tripuraneni, Nilesh and Jin, Chi and Jordan, Michael},
  booktitle={International Conference on Machine Learning},
  pages={10434--10443},
  year={2021},
  organization={PMLR}
}

@article{thekumparampil2021statistically,
  title={Statistically and computationally efficient linear meta-representation learning},
  author={Thekumparampil, Kiran K and Jain, Prateek and Netrapalli, Praneeth and Oh, Sewoong},
  journal={Advances in Neural Information Processing Systems},
  volume={34},
  pages={18487--18500},
  year={2021}
}

@article{lecun2015deep,
  title={Deep learning},
  author={LeCun, Yann and Bengio, Yoshua and Hinton, Geoffrey},
  journal={nature},
  volume={521},
  number={7553},
  pages={436--444},
  year={2015},
  publisher={Nature Publishing Group UK London}
}

@article{bengio2013representation,
  title={Representation learning: A review and new perspectives},
  author={Bengio, Yoshua and Courville, Aaron and Vincent, Pascal},
  journal={IEEE transactions on pattern analysis and machine intelligence},
  volume={35},
  number={8},
  pages={1798--1828},
  year={2013},
  publisher={IEEE}
}

@article{ando2005framework,
  title={A framework for learning predictive structures from multiple tasks and unlabeled data.},
  author={Ando, Rie Kubota and Zhang, Tong and Bartlett, Peter},
  journal={Journal of machine learning research},
  volume={6},
  number={11},
  year={2005}
}

@article{caruana1997multitask,
  title={Multitask learning},
  author={Caruana, Rich},
  journal={Machine learning},
  volume={28},
  pages={41--75},
  year={1997},
  publisher={Springer}
}

@article{Chen_2021,
   title={Spectral Methods for Data Science: A Statistical Perspective},
   volume={14},
   ISSN={1935-8245},
   url={http://dx.doi.org/10.1561/2200000079},
   DOI={10.1561/2200000079},
   number={5},
   journal={Foundations and Trends® in Machine Learning},
   publisher={Emerald},
   author={Chen, Yuxin and Chi, Yuejie and Fan, Jianqing and Ma, Cong},
   year={2021},
   pages={566–806} }

@article{niu2024collaborative,
  title={Collaborative learning with shared linear representations: Statistical rates and optimal algorithms},
  author={Niu, Xiaochun and Su, Lili and Xu, Jiaming and Yang, Pengkun},
  journal={arXiv preprint arXiv:2409.04919},
  year={2024}
}

@article{chen2024finite,
  title={Finite-time analysis of single-timescale actor-critic},
  author={Chen, Xuyang and Zhao, Lin},
  journal={Advances in Neural Information Processing Systems},
  volume={36},
  year={2024}
}

@article{olshevsky2023small,
  title={A small gain analysis of single timescale actor critic},
  author={Olshevsky, Alex and Gharesifard, Bahman},
  journal={SIAM Journal on Control and Optimization},
  volume={61},
  number={2},
  pages={980--1007},
  year={2023},
  publisher={SIAM}
}

@article{fallah2020personalized,
  title={Personalized federated learning with theoretical guarantees: A model-agnostic meta-learning approach},
  author={Fallah, Alireza and Mokhtari, Aryan and Ozdaglar, Asuman},
  journal={Advances in Neural Information Processing Systems},
  volume={33},
  pages={3557--3568},
  year={2020}
}

@article{jiang2019improving,
  title={Improving federated learning personalization via model agnostic meta learning},
  author={Jiang, Yihan and Kone{\v{c}}n{\`y}, Jakub and Rush, Keith and Kannan, Sreeram},
  journal={arXiv preprint arXiv:1909.12488},
  year={2019}
}

@inproceedings{collins2021exploiting,
  title={Exploiting shared representations for personalized federated learning},
  author={Collins, Liam and Hassani, Hamed and Mokhtari, Aryan and Shakkottai, Sanjay},
  booktitle={International conference on machine learning},
  pages={2089--2099},
  year={2021},
  organization={PMLR}
}

@article{t2020personalized,
  title={Personalized federated learning with moreau envelopes},
  author={T Dinh, Canh and Tran, Nguyen and Nguyen, Josh},
  journal={Advances in Neural Information Processing Systems},
  volume={33},
  pages={21394--21405},
  year={2020}
}

@inproceedings{li2021ditto,
  title={Ditto: Fair and robust federated learning through personalization},
  author={Li, Tian and Hu, Shengyuan and Beirami, Ahmad and Smith, Virginia},
  booktitle={International Conference on Machine Learning},
  pages={6357--6368},
  year={2021},
  organization={PMLR}
}

@inproceedings{
xu2023personalized,
title={Personalized Federated Learning with Feature Alignment and Classifier Collaboration},
author={Jian Xu and Xinyi Tong and Shao-Lun Huang},
booktitle={The Eleventh International Conference on Learning Representations },
year={2023},
url={https://openreview.net/forum?id=SXZr8aDKia}
}

@article{tsitsiklis1999average,
  title={Average cost temporal-difference learning},
  author={Tsitsiklis, John N and Van Roy, Benjamin},
  journal={Automatica},
  volume={35},
  number={11},
  pages={1799--1808},
  year={1999},
  publisher={Elsevier}
}

@article{zou2019finite,
  title={Finite-sample analysis for sarsa with linear function approximation},
  author={Zou, Shaofeng and Xu, Tengyu and Liang, Yingbin},
  journal={Advances in neural information processing systems},
  volume={32},
  year={2019}
}

@article{wu2020finite,
  title={A finite-time analysis of two time-scale actor-critic methods},
  author={Wu, Yue Frank and Zhang, Weitong and Xu, Pan and Gu, Quanquan},
  journal={Advances in Neural Information Processing Systems},
  volume={33},
  pages={17617--17628},
  year={2020}
}

@article{
wang2025on,
title={On the Convergence Rates of Federated Q-Learning across Heterogeneous Environments},
author={Muxing Wang and Pengkun Yang and Lili Su},
journal={Transactions on Machine Learning Research},
issn={2835-8856},
year={2025},
url={https://openreview.net/forum?id=EkLAG3gt3g},
note={}
}

@book{sutton2018reinforcement,
  title={Reinforcement learning: An introduction},
  author={Sutton, Richard S and Barto, Andrew G},
  year={2018},
  publisher={MIT press}
}

@article{sutton1999policy,
  title={Policy gradient methods for reinforcement learning with function approximation},
  author={Sutton, Richard S and McAllester, David and Singh, Satinder and Mansour, Yishay},
  journal={Advances in neural information processing systems},
  volume={12},
  year={1999}
}

@article{jin2024v,
  title={V-learning—a simple, efficient, decentralized algorithm for multiagent reinforcement learning},
  author={Jin, Chi and Liu, Qinghua and Wang, Yuanhao and Yu, Tiancheng},
  journal={Mathematics of Operations Research},
  volume={49},
  number={4},
  pages={2295--2322},
  year={2024},
  publisher={INFORMS}
}

@article{qu2022scalable,
  title={Scalable reinforcement learning for multiagent networked systems},
  author={Qu, Guannan and Wierman, Adam and Li, Na},
  journal={Operations Research},
  volume={70},
  number={6},
  pages={3601--3628},
  year={2022},
  publisher={INFORMS}
}

@article{zhang2023global,
  title={Global convergence of localized policy iteration in networked multi-agent reinforcement learning},
  author={Zhang, Yizhou and Qu, Guannan and Xu, Pan and Lin, Yiheng and Chen, Zaiwei and Wierman, Adam},
  journal={Proceedings of the ACM on Measurement and Analysis of Computing Systems},
  volume={7},
  number={1},
  pages={1--51},
  year={2023},
  publisher={ACM New York, NY, USA}
}

@inproceedings{
song2022when,
title={When Can We Learn General-Sum Markov Games with a Large Number of Players Sample-Efficiently?},
author={Ziang Song and Song Mei and Yu Bai},
booktitle={International Conference on Learning Representations},
year={2022},
url={https://openreview.net/forum?id=6MmiS0HUJHR}
}

@book{filar2012competitive,
  title={Competitive Markov decision processes},
  author={Filar, Jerzy and Vrieze, Koos},
  year={2012},
  publisher={Springer Science \& Business Media}
}

@inproceedings{daskalakis2023complexity,
  title={The complexity of markov equilibrium in stochastic games},
  author={Daskalakis, Constantinos and Golowich, Noah and Zhang, Kaiqing},
  booktitle={The Thirty Sixth Annual Conference on Learning Theory},
  pages={4180--4234},
  year={2023},
  organization={PMLR}
}

@article{nguyen2021federated,
  title={Federated learning for internet of things: A comprehensive survey},
  author={Nguyen, Dinh C and Ding, Ming and Pathirana, Pubudu N and Seneviratne, Aruna and Li, Jun and Poor, H Vincent},
  journal={IEEE communications surveys \& tutorials},
  volume={23},
  number={3},
  pages={1622--1658},
  year={2021},
  publisher={IEEE}
}

@inproceedings{maeng2022towards,
  title={Towards fair federated recommendation learning: Characterizing the inter-dependence of system and data heterogeneity},
  author={Maeng, Kiwan and Lu, Haiyu and Melis, Luca and Nguyen, John and Rabbat, Mike and Wu, Carole-Jean},
  booktitle={Proceedings of the 16th ACM Conference on Recommender Systems},
  pages={156--167},
  year={2022}
}

@inproceedings{xionglinear,
  title={On the Linear Speedup of Personalized Federated Reinforcement Learning with Shared Representations},
  author={Xiong, Guojun and Wang, Shufan and Jiang, Daniel and Li, Jian},
  booktitle={The Thirteenth International Conference on Learning Representations},
  year={2025}
}

@inproceedings{
yang2023federated,
title={Federated Natural Policy Gradient and Actor Critic Methods for Multi-task Reinforcement Learning},
author={Tong Yang and Shicong Cen and Yuting Wei and Yuxin Chen and Yuejie Chi},
booktitle={The Thirty-eighth Annual Conference on Neural Information Processing Systems},
year={2024},
url={https://openreview.net/forum?id=DUFD6vsyF8}
}

@inproceedings{zhang2018fully,
  title={Fully decentralized multi-agent reinforcement learning with networked agents},
  author={Zhang, Kaiqing and Yang, Zhuoran and Liu, Han and Zhang, Tong and Basar, Tamer},
  booktitle={International conference on machine learning},
  pages={5872--5881},
  year={2018},
  organization={PMLR}
}

@article{salgia2025sample,
  title={The sample-communication complexity trade-off in federated Q-learning},
  author={Salgia, Sudeep and Chi, Yuejie},
  journal={Advances in Neural Information Processing Systems},
  volume={37},
  pages={39694--39747},
  year={2025}
}

@article{zhang2024finite,
  title={Finite-time analysis of on-policy heterogeneous federated reinforcement learning},
  author={Zhang, Chenyu and Wang, Han and Mitra, Aritra and Anderson, James},
  journal={arXiv preprint arXiv:2401.15273},
  year={2024}
}

@article{xie_fedkl_2023,
	title = {{FedKL}: {Tackling} {Data} {Heterogeneity} in {Federated} {Reinforcement} {Learning} by {Penalizing} {KL} {Divergence}},
	volume = {41},
	issn = {1558-0008},
	shorttitle = {{FedKL}},
	url = {https://ieeexplore.ieee.org/abstract/document/10038492?casa_token=yGyMDlnL_FsAAAAA:hqNvzWEb6yVKwTZVdHKLVnorDg07AWx4uujDsLLTTY_7unjr1ew8Yv4_UUAWfCx3X1b9wHNySP8},
	doi = {10.1109/JSAC.2023.3242734},
	number = {4},
	urldate = {2024-04-16},
	journal = {IEEE Journal on Selected Areas in Communications},
	author = {Xie, Zhijie and Song, Shenghui},
	month = apr,
	year = {2023},
	pages = {1227--1242},
}

@article{deng2020adaptive,
  title={Adaptive personalized federated learning},
  author={Deng, Yuyang and Kamani, Mohammad Mahdi and Mahdavi, Mehrdad},
  journal={arXiv preprint arXiv:2003.13461},
  year={2020}
}

@inproceedings{khodadadian2022federated,
  title={Federated reinforcement learning: Linear speedup under markovian sampling},
  author={Khodadadian, Sajad and Sharma, Pranay and Joshi, Gauri and Maguluri, Siva Theja},
  booktitle={International Conference on Machine Learning},
  pages={10997--11057},
  year={2022},
  organization={PMLR}
}

@inproceedings{bhandari2018finite,
  title={A finite time analysis of temporal difference learning with linear function approximation},
  author={Bhandari, Jalaj and Russo, Daniel and Singal, Raghav},
  booktitle={Conference on learning theory},
  pages={1691--1692},
  year={2018},
  organization={PMLR}
}

@inproceedings{jin_federated_2022,
  title={Federated reinforcement learning with environment heterogeneity},
  author={Jin, Hao and Peng, Yang and Yang, Wenhao and Wang, Shusen and Zhang, Zhihua},
  booktitle={International Conference on Artificial Intelligence and Statistics},
  pages={18--37},
  year={2022},
  organization={PMLR}
}

@InProceedings{woo_blessing_2023,
  title = 	 {The Blessing of Heterogeneity in Federated Q-Learning: Linear Speedup and Beyond},
  author =       {Woo, Jiin and Joshi, Gauri and Chi, Yuejie},
  booktitle = 	 {Proceedings of the 40th International Conference on Machine Learning},
  pages = 	 {37157--37216},
  year = 	 {2023},
  editor = 	 {Krause, Andreas and Brunskill, Emma and Cho, Kyunghyun and Engelhardt, Barbara and Sabato, Sivan and Scarlett, Jonathan},
  volume = 	 {202},
  series = 	 {Proceedings of Machine Learning Research},
  month = 	 {23--29 Jul},
  publisher =    {PMLR},
  url = 	 {https://proceedings.mlr.press/v202/woo23a.html}
}

@article{wang2023federated,
  title={Federated temporal difference learning with linear function approximation under environmental heterogeneity},
  author={Wang, Han and Mitra, Aritra and Hassani, Hamed and Pappas, George J and Anderson, James},
  journal={arXiv preprint arXiv:2302.02212},
  year={2023}
}

@inproceedings{zhang2023convergence,
  title={On the convergence of SARSA with linear function approximation},
  author={Zhang, Shangtong and Des Combes, Remi Tachet and Laroche, Romain},
  booktitle={International Conference on Machine Learning},
  pages={41613--41646},
  year={2023},
  organization={PMLR}
}
\bibliographystyle{icml2026}

\newpage 

\appendix
\onecolumn
\begin{center}
\LARGE 
\bf Appendices 
\end{center}
\section{Additional Related Work}

\noindent{\bf MARL under homogeneous environments.}
When the environments of all agents are homogeneous, it has been shown that the federated version of classic reinforcement learning algorithms can significantly alleviate the data collection burden on individual agents \citep{woo_blessing_2023,khodadadian2022federated}. Specifically, the per-agent sample complexity decreases linearly in terms of the number of agents -- an effect commonly referred to in existing literature as linear speedup.  
Recent work has made significant progress in understanding federated RL under homogeneous environments. 
\cite{woo_blessing_2023} studied federated Q-learning.  
In addition to a linear speedup, \cite{woo_blessing_2023} uncovered the blessing of heterogeneity in terms of state-action exploration  -- a completely different notion of heterogeneity from our focus. 
\cite{salgia2025sample} studied the sample-communication complexity trade-off, and developed a new algorithm that achieves order-optimal sample and communication complexities. 

\noindent {\bf Model/Policy personalization.} 
In FL, model personalization (often referred to as PFL) has garnered significant attention in recent years. 
The success of personalization techniques depends on balancing the bias introduced by using global knowledge that may not generalize to individual clients, and the variance inherent in learning from limited local datasets. Popular PFL techniques include regularized local objectives \citep{t2020personalized,li2021ditto}, local-global parameter interpolation \citep{deng2020adaptive}, meta-learning \citep{fallah2020personalized, jiang2019improving}, and representation learning \citep{collins2021exploiting, xu2023personalized}. 
Extending these ideas of PFL to RL introduces an even greater level of difficulty that arises from the inherent sample correlation induced by sequential decision-making, and the non-stationarity of sample quality as the policy evolves.

In the RL literature, training personalized (or localized) policies has recently garnered considerable attention in the context of multi-agent Markov games or competitive MARL \cite{filar2012competitive, zhang2018fully, qu2022scalable, zhang2023global, daskalakis2023complexity}, which depart fundamentally from our focus.  
Specifically, in their setting, agents interact with one another within a shared environment, and the rewards received by individual agents are determined by the joint actions of all agents. Each agent aims to learn a localized policy that contributes to a joint policy across all agents, with the goal of maximizing its own welfare.
A commonly adopted objective among the participating agents is to learn some form of equilibrium, such as a Nash Equilibrium.
Departing from Markov games, we focus on cooperative MARL, where we try to avoid the curse of heterogeneity in the joint learning. 
\section{Notations and Preliminaries}
\label{sec: Notations}

Let $v_{0:t}^k$ denote the Markovian trajectory up to time $t$ generated under initial state distribution $d_0$, given policy $\pi$, and the local transition kernel $P^k$: 
$$v_{0:t}^k=(s_0^k,a_0^k,s_1^k,...,s_{t+1}^k),$$ and $v_{0:t}$ denote the collection of Markovian trajectories generated by all agents:
\begin{align*}
v_{0:t} = \{v_{0:t}^k\}_{k=1}^K.
\end{align*}
Similarly, we define $v_{0:(t,\ell)}^k$ and $v_{0:(t,\ell)}$ as the Markovian trajectory generated up to the $\ell$-th sample in time $t$. 

Let $\bbE_{j}[X_t]$ denote the conditional expectation of any random variable $X_t$ given the collection of Markovian trajectories up to time $j$:
\begin{align*}
    \bbE_{j}[X_t] = \bbE\qth{X_t\mid v_{0:j}} = \bbE\qth{X_t\mid \{v_{0:j}^k\}_{k=1}^K}.
\end{align*}

Denote 
\begin{align}
\label{eq: local reward Markovian noise}
    \tilde{b}_{t,L}^k = \sum_{\ell=0}^{L-1}(r_{t,\ell}^k - J^k)\phi(s_{t,0}^k);\quad \bar{b}_L^k=  \bbE_{s^{(0)}\sim\mu^k, (a^{(\ell)}, s^{(\ell)}) \sim \pi \otimes P^k ~ \forall \ell\in [0, L-1]}\qth{\sum_{\ell=0}^{L-1}(r(s^{(\ell)},a^{(\ell)}) - J^k)\phi(s^{(0)})}.
\end{align}

In addition, let 
\begin{equation}
\label{eq: Markovian noise}   
\begin{aligned}
\xi_{t,L}^k  =& \sum_{\ell=0}^{L-1} (r_{t,\ell}^k - J^k)\phi(s_{t,0}^k)\\&
+ (\phi(s_{t,L}^k) - \phi(s_{t,0}^k))^\top \mathbf{B}_t \omega_t^k\phi(s_{t,0}^k)\\
    &-\sum_{\ell=0}^{L-1}\bbE \qth{(r(s^{(\ell)},a^{(\ell)}) - J^k)\phi(s^{(0)})}\\
    & - \bbE\qth{(\phi(s^{(L)}) - \phi(s^{(0)}))^\top \mathbf{B}_t \omega_t^k\phi(s^{(0)}) },
\end{aligned} 
\end{equation}
where the expectation is taken with respect to 
${s^{(0)}\sim\mu^k, (a^{(\ell)}, s^{(\ell+1)}) \sim \pi \otimes P^k ~ \forall \ell\in [0, L-1]}$.
This can be viewed as the Markovian noise of the TD error of agent $k$ at time $t$. 
Moreover, let 
\begin{align}
\label{def: Markov drift and noise}
&\tilde A_{t,L}^k= \phi(s_{t,0}^k)(\phi(s_{t,L}^k)-\phi(s_{t,0}^k))^\top, \nonumber\\
~~ \text{and } ~      
&\bfb_{t,L}^k = \sum_{\ell=0}^{L-1}\big( (r_{t,\ell}^k - J^k) + (\phi(s_{t,\ell+1}^k)\nonumber\\
&~~~~~~~~~ - \phi(s_{t,\ell}^k))^\top \mathbf{B}^*\omega^{k,*}\big)\phi(s_{t,0}^k). 
\end{align}

The following rewrite appears repeatedly in our analysis.  
\begin{proposition}
\label{prop: key intermediate}
For any $t$ and agent $k$, $\delta_{t,L}^k\phi(s_{t,0}^k)$ can be rewritten as 
\begin{align}
\label{eq: delta-phi: rewritting 1}  
\delta_{t,L}^k\phi(s_t^k) = \xi_{t,L}^k+A_L^kx_t^k+Ly_t^k\phi(s_t^k),
\end{align}     
and 
\begin{align}
\label{eq: delta-phi: rewritting 2} 
\delta_{t,L}^k\phi(s_t^k) = \tilde A_{t,L}^kx_t^k+Ly_t^k\phi(s_{t,0}^k)+\bfb_{t,L}^k.
\end{align} 
\end{proposition}

The proof of Proposition \ref{prop: key intermediate} can be found in Appendix \ref{app: other supporting}.

\section{Bound on QR Decomposition and Perturbation}
\label{sec: QR}

Recall from Algorithm \ref{alg: fedac-per} the update of $\bar \bfB$: 
\begin{align*}
    &\bar{\mathbf{B}}_{t+1} = \mathbf{B}_t
    +\frac{\zeta}{KL}\sum_{k=1}^K \bfB_{t,\perp}\bfB_{t,\perp}^\top\delta_{t,L}^k\phi(s_{t,0}^k)(\omega_{t}^k)^\top
\end{align*}
For ease of exposition, let 
\begin{align}
\label{eq: def Q}
\bfQ_t : = \frac{\zeta}{KL}\sum_{k=1}^K\bfB_{t,\perp}\bfB_{t,\perp}^\top\delta_{t,L}^k\phi(s_{t,0}^k)(\omega_{t}^k)^\top. 
\end{align}
The following couple of lemmas enable us to bound the ``distortion'' caused by the QR decomposition. 
Specifically, Lemma \ref{lm: perturbation QR} bounds the perturbation in terms of $\|\bf Q_t\|$, which is further bounded in Lemma \ref{lmm: perturbation QR: Q}. 
It is worth noting that, in the traditional PFL \cite{collins2021exploiting}, such distortion can be well controlled by the number of $\iid$ local samples that are freshly drawn in each iteration. However, their analysis does not apply to our setting due to the Markov sampling. 
We fundamentally depart from the analysis in \cite{collins2021exploiting} in deriving upper bounds on those distortions via constructing a fixed-point iteration.

\begin{lemma}[Perturbation of QR]
\label{lm: perturbation QR}
For each $t\ge 1$, when $U_{\delta} U_{\omega}\zeta \le 1/2$, it holds that 
\begin{align}
&\Norm{\bfR_{t+1}-\bfI}\leq 2\|\bfQ_t\|_F^2,\label{eqn: R related bounds: R-I}\\
&\|\bfR_{t+1}^{-1}\|\leq \frac{1}{1-\|\bfR_{t+1}-\bfI\|_2}\leq \frac{1}{1-2\|\bfQ_t\|_F^2}, \label{eqn: R related bounds: R inv}\\
&\Norm{\bfR_{t+1}^{-1}-\bfI} \leq 4 \|\bfQ_t\|_F^2. \label{eqn: R related bounds: R inv-I}
\end{align}     
\end{lemma}
\begin{proof}
Recall that $\bar \bfB_{t+1} = \bfB_{t+1}\bfR_{t+1}$. We rewrite $\bfR_{t+1}^\top \bfR_{t+1}$ as   

\begin{align*}
\bfR_{t+1}^\top \bfR_{t+1}
& = \bar{\mathbf{B}}_{t+1}^\top \bar{\mathbf{B}}_{t+1}
= \mathbf{B}_t^\top\mathbf{B}_t + \mathbf{B}_t^\top \mathbf{Q}_t + \mathbf{Q}_t^\top \mathbf{B}_t 
+ \mathbf{Q}_t^\top \mathbf{Q}_t\\
& = \bfI + \mathbf{B}_t^\top \mathbf{Q}_t + \mathbf{Q}_t^\top \mathbf{B}_t 
+ \mathbf{Q}_t^\top \mathbf{Q}_t\\
 & = \bfI + \mathbf{Q}_t^\top \mathbf{Q}_t. 
\end{align*}
The last equality holds because $\bfQ_t$ lives in $\text{span}(\bfB_{t,\perp})$.

It can also be rewritten as 
\begin{align*}
\bfR_{t+1}^\top \bfR_{t+1}
& = (\bfI +  \bfR_{t+1}-\bfI  )^\top (\bfI +  \bfR_{t+1}-\bfI  )\\
& =  \bfI +   (\bfR_{t+1}-\bfI)^\top  + (\bfR_{t+1}-\bfI) + (\bfR_{t+1}-\bfI)^\top (\bfR_{t+1}-\bfI). 
\end{align*}
Thus, we get
\begin{equation}
\label{eq:QR-equation}
(\bfR_{t+1}-\bfI)^\top  + (\bfR_{t+1}-\bfI) 
=\mathbf{Q}_t^\top \mathbf{Q}_t 
- (\bfR_{t+1}-\bfI)^\top (\bfR_{t+1}-\bfI). 
\end{equation}

We construct a fixed-point iteration to ``solve'' $\bfR_{t+1}-\bfI$ in terms of $\mathbf{Q}_t^\top \mathbf{Q}_t$ under Eq.\,(\ref{eq:QR-equation}). 

Toward this, we decompose the symmetric matrix $\mathbf{Q}_t^\top \mathbf{Q}_t $ into 
\[
\mathbf{Q}_t^\top \mathbf{Q}_t 
= \mathbf{R}_0 + \mathbf{R}_0^\top,
\]
where $\mathbf{R}_0$ is upper triangular. It is easy to see that 
\begin{align*}
    \lnorm{\mathbf{R}_0}{F} \leq \lnorm{\mathbf{R}_0 + \mathbf{R}_0^\top}{F}\leq \|\mathbf{Q}_t\|_F^2\leq \zeta^2 U_\delta^2 U_\omega^2.
\end{align*}

Let $\Delta_{t+1}\triangleq \bfR_{t+1}-\bfI - \mathbf{R}_0$. 
Then Eq.\,(\ref{eq:QR-equation}) is equivalent to
\begin{equation}
\label{eq:fixed-point-Delta}
\Delta_{t+1}^\top  +  \Delta_{t+1}
= - (\Delta_{t+1}^\top \Delta_{t+1}  + \Delta_{t+1}^\top \mathbf{R}_0 
+ \mathbf{R}_0^\top \Delta_{t+1} + \mathbf{R}_0^\top\mathbf{R}_0).
\end{equation}
Next, we define a sequence of $\Delta_{t+1,i}$ for $i\ge 0$. 
Let $\Delta_{t+1,0}=0$ for $i\ge 1$ and let $\Delta_{t+1,i}$ be an upper triangular matrix satisfying
\[
\Delta_{t+1,i}^\top  +  \Delta_{t+1,i}
= - (\Delta_{t+1,i-1}^\top \Delta_{t+1,i-1}  + \Delta_{t+1,i-1}^\top \mathbf{R}_0 
+ \mathbf{R}_0^\top \Delta_{t+1,i-1} + \mathbf{R}_0^\top\mathbf{R}_0).
\]

We first show the boundedness of the sequence. 
Note that
\begin{equation}
\label{eq:fixed-point-ub}
\lnorm{\Delta_{t+1,i}}{F}
\le \lnorm{\Delta_{t+1,i}^\top  +  \Delta_{t+1,i}}{F}
\le \lnorm{\Delta_{t+1,i-1}}{F}^2 + 2 \lnorm{\mathbf{R}_0}{F}\lnorm{\Delta_{t+1,i-1}}{F} + \lnorm{\mathbf{R}_0}{F}^2,
\end{equation}
where the second inequality applies the triangle inequality, 
and the first inequality holds because the absolute value of all entries in $\Delta_{t+1,i}^\top +  \Delta_{t+1,i}$ is larger than that in $\Delta_{t+1,i}$.

Provided $\zeta^2 U_\delta^2 U_\omega^2< 1/4$, we have $\lnorm{\mathbf{R}_0}{F}< 1/4$, and for any $\lnorm{\Delta_{t+1,i-1}}{F}\le 1/4$, we have $\lnorm{\Delta_{t+1,i}}{F}\le 1/4$ by Eq.\,(\ref{eq:fixed-point-ub}).  
By induction, the entire sequence satisfies $\lnorm{\Delta_{t+1,i}}{F}\le 1/4$ for all $i\ge 0$.

Next, we show the convergence of the sequence $\Delta_{t+1,i}$. 
By definition,
\begin{align*}
&(\Delta_{t+1,i+1}^\top  +  \Delta_{t+1,i+1}) - (\Delta_{t+1,i}^\top  +  \Delta_{t+1,i})\\
=~& - (  \Delta_{t+1,i}^\top (\Delta_{t+1,i} - \Delta_{t+1,i-1}) 
+ (\Delta_{t+1,i} - \Delta_{t+1,i-1})^\top \Delta_{t+1,i-1} \\
&\qquad + (\Delta_{t+1,i} - \Delta_{t+1,i-1})^\top \mathbf{R}_0 
+ \mathbf{R}_0^\top (\Delta_{t+1,i} - \Delta_{t+1,i-1})
). 
\end{align*}
By taking the norm on both sides, we obtain
\begin{align*}
\lnorm{\Delta_{t+1,i+1} - \Delta_{t+1,i}}{F}
& \le \lnorm{(\Delta_{t+1,i+1}- \Delta_{t+1,i})^\top  +  (\Delta_{t+1,i+1} - \Delta_{t+1,i})}{F} \\
& \le \frac{1}{4} \lnorm{\Delta_{t+1,i} - \Delta_{t+1,i-1}}{F} 
+ \frac{1}{4} \lnorm{\Delta_{t+1,i} - \Delta_{t+1,i-1}}{F}  \\
&\quad + 2 \lnorm{\mathbf{R}_0}{F}  \lnorm{\Delta_{t+1,i} - \Delta_{t+1,i-1}}{F}.
\end{align*}
Therefore, since $\lnorm{\mathbf{R}_0}{F}< \frac{1}{4}$, the difference $\lnorm{\Delta_{t+1,i+1} - \Delta_{t+1,i}}{F}$ geometrically converges to zero, which implies that the sequence $\Delta_{t+1,i}$ converges.

The limit of the sequence satisfies the fixed-point Eq.\,~(\ref{eq:fixed-point-Delta}). 
Taking norm on both sides of~\eqref{eq:fixed-point-Delta} yields
\[
\lnorm{\Delta_{t+1}}{F}
\le \lnorm{\Delta_{t+1}^\top  +  \Delta_{t+1}}{F}
\le \lnorm{\Delta_{t+1}}{F}^2 + 2 \lnorm{\mathbf{R}_0}{F}\lnorm{\Delta_{t+1}}{F} + \lnorm{\mathbf{R}_0}{F}^2.
\]
Since $\lnorm{\Delta_{t+1}}{F}\le 1/4$, we obtain that 
\[
\lnorm{\Delta_{t+1}}{F}
\le \frac{1-2\lnorm{\mathbf{R}_0}{F} - \sqrt{1-4\lnorm{\mathbf{R}_0}{F} }}{2}
\leq 4\lnorm{\mathbf{R}_0}{F}^2.
\]
We conclude that 
\[
\lnorm{\bfR_{t+1}-\bfI}{F}
=\lnorm{\mathbf{R}_0 + \Delta_{t+1}}{F}
\leq \lnorm{\mathbf{R}_0}{F} + 4\lnorm{\mathbf{R}_0}{F}^2 
\leq 2 \lnorm{\mathbf{R}_0}{F} \leq 2\lnorm{\bfQ_t}{F}^2. 
\]

Moreover,
\[
\sigma_{\min}(\bfR_{t+1})\ge 1- \Norm{\bfR_{t+1}-\bfI}, ~~~ \text{by Weyl's inequality},
\]
\[
\sigma_{\max}(\bfR_{t+1})\le 1 + \Norm{\bfR_{t+1}-\bfI}, ~~~ \text{by triangle inequality}, 
\]
and 
\[
\Norm{\bfR_{t+1}^{-1}-\bfI}
=\Norm{\bfR_{t+1}^{-1}(\bfI-\bfR_{t+1})}
\le \Norm{\bfR_{t+1}^{-1}}\Norm{\bfI-\bfR_{t+1}}.
\]
Therefore,
\begin{align*}
&\lnorm{\bfR_{t+1}-\bfI}{2}\leq\lnorm{\bfR_{t+1}-\bfI}{F}\leq 2\|\bfQ_t\|_F^2,\\
&\|\bfR_{t+1}^{-1}\|\leq \frac{1}{1-\|\bfR_{t+1}-\bfI\|_2}\leq \frac{1}{1-\|\bfR_{t+1}-\bfI\|_F}  \le \frac{1}{1-2\lnorm{\bf Q_t}{F}^2},\\
&\Norm{\bfR_{t+1}^{-1}-\bfI} \leq 4\|\bfQ_t\|_F^2.     
\end{align*} 
For the last inequality, we used $2U_\delta^2 U_\omega^2 \zeta^2\leq 1/2$ so that $\|\bfR_{t+1}^{-1}\|\leq 2$.
        
\end{proof}

\begin{lemma}
\label{lmm: perturbation QR: Q}
For any $t\ge \tau$, 
\begin{align}
\label{eqn: Q bound}
\bbE_{t-\tau}\|\bfQ_t\|_F^2\leq&\frac{16U_\omega^2\zeta^2}{L^2}\frac{1}{K}\sum_{k=1}^K \bbE_{t-\tau}\|x_t^k\|^2 
+ 4\frac{U_{\delta,L}^2}{L^2}U_\omega^2\frac{\zeta^2}{K}   + 8 U_\omega^2 \frac{U_{\delta,L}^2}{L^2} \zeta^2 m^2\rho^{2\tau}\nonumber\\
&+4\tau^2\frac{U_{\delta,L}^4}{L^4}\zeta^2\beta^2+\frac{4U_\omega^2\zeta^2}{K}\sum_{k=1}^K \bbE_{t-\tau}| y_t^k|^2. 
\end{align}
\end{lemma}

\begin{proof}

For any $t\ge \tau$, we have      
\begin{align*}
    &\|\bfQ_t\|_F^2=\|\frac{\zeta}{KL}\sum_{k=1}^K \bfP_{t,\perp}\delta_{t,L}^k\phi(s_{t,0}^k)(\omega_{t}^k)^\top\|_F^2\leq \|\frac{\zeta}{KL}\sum_{k=1}^K \delta_{t,L}^k\phi(s_{t,0}^k)(\omega_{t}^k)^\top\|_F^2\\
    =&\frac{\zeta^2}{K^2L^2}\|\sum_{k=1}^K \pth{\tilde A_{t,L}^kx_t^k+\bfb_{t,L}^k+L y_t^k\phi(s_{t,0}^k)}(\omega_{t}^k)^\top\|_F^2~~~(\text{by Proposition \ref{prop: key intermediate}})\\
    \leq&\underbrace{\frac{4\zeta^2L^2}{K^2}\|\sum_{k=1}^K \tilde A_{t,L}^kx_t^k(\omega_{t}^k)^\top\|_F^2}_{I_1}
    + \underbrace{\frac{4\zeta^2}{K^2L^2}\|\sum_{k=1}^K \bfb_{t,L}^k(\omega_{t-\tau}^k)^\top\|_F^2}_{I_2}
    \\
    &+
    \underbrace{\frac{4\zeta^2}{K^2L^2}\|\sum_{k=1}^K \bfb_{t,L}^k(\omega_{t}^k-\omega_{t-\tau}^k)^\top\|_F^2}_{I_3}
    + \underbrace{\frac{4\zeta^2L^2}{K^2L^2}\|\sum_{k=1}^K  y_t^k\phi(s_{t,0}^k)(\omega_{t}^k)^\top\|_F^2}_{I_4}.
\end{align*}
Recall that $\tilde A_{t,L}^k= \sum_{\ell=0}^{L-1}\phi(s_{t,0}^k)(\phi(s_{t,\ell+1}^k) - \phi(s_{t,\ell}^k))^\top,
\bfb_{t,L}^k = \sum_{\ell=0}^{L-1}\pth{ (r_{t,\ell}^k - J^k) + (\phi(s_{t,\ell+1}^k) - \phi(s_{t,\ell}^k))^\top \mathbf{B}^*\omega^{k,*}}\phi(s_{t,0}^k)$.
Then,
\begin{align}
\label{eqn: A_L^k crude bound}
    \|\tilde A_{t,L}^k\| = \|\phi(s_{t,0}^k)(\phi(s_{t,L}^k) - \phi(s_{t,0}^k))^\top\|\leq 2.
\end{align}
For $I_1$, we have,
\begin{align*}
    &I_1 \leq \frac{4\zeta^2}{KL^2}\sum_{k=1}^K \|\tilde A_{t,L}^kx_t^k(\omega_{t}^k)^\top\|^2
    \leq\frac{16U_\omega^2\zeta^2}{L^2}\frac{1}{K}\sum_{k=1}^K \|x_t^k\|^2.
\end{align*}

For $I_2$, the Markovian noise, we have under expectation that,
\begin{align*}
    &\frac{4\zeta^2}{K^2L^2}\bbE_{t-\tau} \|\sum_{k=1}^K \bfb_{t,L}^k(\omega_{t-\tau}^k)^\top\|_F^2\\
    \leq&\frac{4\zeta^2}{K^2L^2}\bbE_{t-\tau}\sum_{k=1}^K\|\bfb_{t,L}^k(\omega_{t-\tau}^k)^\top\|^2 + \frac{8\zeta^2}{K^2L^2} \sum_{k=1}^K\sum_{k'=k+1}^K\bbE_{t-\tau}\iprod{\bfb_{t,L}^k(\omega_{t-\tau}^k)^\top}{\bfb_{t,L}^{k'}(\omega_{t-\tau}^{k'})^\top}\\
    \leq& 4\frac{U_{\delta,L}^2}{L^2}U_\omega^2\frac{\zeta^2}{K}   + 8\zeta^2 U_\omega^2 \frac{U_{\delta,L}^2}{L^2} m^2\rho^{2\tau}. ~~~~~ (\text{by Lemma \ref{lmm: U delta}, independence and the {ergodicity assumption}})
\end{align*}

For $I_3$, 
\begin{align*}
    &\frac{4\zeta^2}{K^2L^2}\|\sum_{k=1}^K \bfb_{t,L}^k(\omega_{t}^k-\omega_{t-\tau}^k)^\top\|_F^2\\
    \leq&\frac{4\zeta^2}{KL^2}  \sum_{k=1}^K\| \bfb_{t,L}^k(\omega_{t}^k-\omega_{t-\tau}^k)^\top\|^2\\
    \leq&\frac{4\zeta^2}{KL^2}  \sum_{k=1}^K\| \bfb_{t,L}^k\|^2\|\omega_{t}^k-\omega_{t-\tau}^k\|^2\\
    \leq& 4\frac{U_{\delta,L}^2}{L^2}\zeta^2 \pth{\tau \frac{U_{\delta,L}}{L}\beta}^2 ~~~~~ (\text{by Lemma \ref{lmm: crude parameter bound}})\\
    =& 4\tau^2\frac{U_{\delta,L}^4}{L^4}\zeta^2\beta^2.
\end{align*}

Lastly, for $I_4$, we have
\begin{align*}
    I_4\leq& \frac{4\zeta^2}{K}\sum_{k=1}^K \| y_t^k\phi(s_{t,0}^k)(\omega_{t}^k)^\top\|^2
    \leq\frac{4U_\omega^2\zeta^2}{K}\sum_{k=1}^K | y_t^k|^2.
\end{align*}
In summary,
\begin{align*}
\bbE_{t-\tau}\|\bfQ_t\|_F^2\leq&\frac{16U_\omega^2\zeta^2}{L^2}\frac{1}{K}\sum_{k=1}^K \bbE_{t-\tau}\|x_t^k\|^2 
+ 4\frac{U_{\delta,L}^2}{L^2}U_\omega^2\frac{\zeta^2}{K}   + 8 U_\omega^2 \frac{U_{\delta,L}^2}{L^2} \zeta^2 m^2\rho^{2\tau}\nonumber\\
&+4\tau^2\frac{U_{\delta,L}^4}{L^4}\zeta^2\beta^2+\frac{4U_\omega^2\zeta^2}{K}\sum_{k=1}^K \bbE_{t-\tau}| y_t^k|^2.
\end{align*}

\end{proof}
\newpage

\section{Bounds on Local Head Errors}
\subsection{Upper bound: Proof of Lemma \ref{lm: local head: upper bound}}\label{sec: proof of local head upper bound}

We first provide a formal statement of Lemma \ref{lm: local head: upper bound}, which involves a set of constants whose formal definitions are deferred to the lemma proof. 
\begin{lemma}[Upper bound of local head errors]
\label{lm: local head: upper bound: formal}    
Let $\bfP^* = \bfB^* (\bfB^*)^{\top}$, and $\bfP_t = \bfB_t \bfB_t^{\top}$ for each $t$. %
Choose $\tau = \lceil 2\log_\rho(\zeta)\rceil$, $U_{\delta} U_{\omega} \zeta\le \frac{1}{2}$, where $U_{\delta} = 2U_r + 2U_{\omega}$, $\beta = c\zeta$ for arbitrary $c>0$ that can be tuned. There exist arbitrary positive constants $c_3, c_4, c_5, c_6$ such that  
\begin{align*}
            &\bar X_{t+1}  \nonumber
   \leq \Bigg(1-2\lambda c\zeta
   +\Big(c_3^2\frac{c }{L} 
+ c_4^2 c  
+ c_5^2\frac{1}{L}
+c_6^2 +4U_\omega^4\frac{1}{c_5^2L}\Big)\zeta   
+\calC_{X,1}(\tau^2)\zeta^2
+ \calC_{X,2}(\tau^2)\zeta^3
\Bigg)\bar X_t\nonumber\\
&\qquad+ \pth{\frac{36U_\omega^2}{c_3^2}\frac{c\zeta}{L}  +6U_\omega U_\delta c \zeta } M_t
\nonumber\\
&\qquad+ \Bigg(
\Big[
\frac{1}{c_4^2}c 
+\frac{U_\omega^4}{c_6^2}
\Big]\zeta
+\calC_{X,3}(\tau^2)\zeta^2
+\calC_{X,4}(\tau^2) \zeta^3
\Bigg)\bar Y_t\nonumber\\
&\qquad+\calC_{X,5}(\tau^2)\frac{\zeta^2}{L}
 +\calC_{X,6}(\tau^2)\frac{\zeta^2}{K} 
+\calC_{X,7}\zeta^3
 +\calC_{X,8}\zeta^4
 + \calC_{X,9} \gamma^2\zeta^2 
+
\calC_{X,10}\zeta^5
+
\calC_{X,11}\gamma^2\zeta^3
+\calC_{X,12}\zeta^7
+ \calC_{X,13}\zeta^6,
\end{align*}
where
\begin{align*}
    &\calC_{X,1}(\tau^2) =  \bigg(C_{31}c 
+(36+40\tau^2)\frac{c^2}{L^2}
+C_{33}\frac{c}{L} 
+\pth{36 U_\omega^4+\pth{2.5+3U_\omega}C_{18}} \frac{1}{L^2}  +\frac{96c^2}{L^2}
+\frac{256U_\omega^4}{L^2}
+96 U_\delta c \frac{U_\omega^3}{c_2^2}\frac{1}{L^2}
\bigg),\\
&\calC_{X,2}(\tau^2) = 6U_\omega U_\delta c  \calC_{M,4}(\tau^2),\\
&\calC_{X,3}(\tau^2)=\Bigg[
C_{32}c
+(9+10\tau^2)c^2
+C_{34}c
+9 U_\omega^4+\pth{2.5+3U_\omega}C_{23} +24c^2
+64U_\omega^4
+24  U_\delta c  \frac{U_\omega^3}{c_1^2}
\Bigg], 
\end{align*}
\begin{align*}
&\calC_{X,4}(\tau^2)=6U_\omega U_\delta c  \calC_{M,5}(\tau^2), 
\calC_{X,5}(\tau^2)=\qth{C_{39}(\tau^2)+24c^2U_\delta^2(1+2\frac{m\rho}{1-\rho})},
\calC_{X,6}(\tau^2)=\qth{ C_{41}(\tau^2)
  + 64\frac{U_{\delta,L}^2}{L^2}U_\omega^4},\\
&\calC_{X,7} = \qth{C_{42}+C_{43}(\tau) +24U_\omega U_\delta c  C_{25}(\tau^2)\frac{1}{L} +
6U_\omega U_\delta c  \calC_{M,2}(\tau^2)\frac{1}{K}},\\
&\calC_{X,8}=\qth{C_{40}(\tau^4)+100U_\omega^6 U_\delta^4+144U_\omega^2\tau^2\frac{U_{\delta,L}^4}{L^4}c^2+
6U_\omega U_\delta c  \calC_{M,6}}, 
\calC_{X,9} = C_{45}(\tau^4),
\calC_{X,10} = 6U_\omega U_\delta c  \calC_{M,1},\\
&\calC_{X,11} = 24U_\omega U_\delta c  C_{28}(\tau^4),\calC_{X,12} = 6U_\omega U_\delta c  \calC_{M,3},\calC_{X,13}=\qth{128 U_\omega^4 \frac{U_{\delta,L}^2}{L^2}  m^2+C_{44}(\tau^2)}, 
\end{align*} 
and the missing definition of constants such as $C_31$ can be found in the analysis. 
\end{lemma}

\begin{proof}
Recall that $\omega_{t+1}^k = \Pi_{U_\omega}(\omega_{t}^k +\frac{\beta}{L}\delta_{t,L}^k\bfB_t^\top\phi(s_{t,0}^k))$, $\bfB_{t+1} = \bar \bfB_{t+1}\bfR_{t+1}^{-1}$, and $\tilde\omega_{t+1}^k = \omega_{t}^k +\frac{\beta}{L}\delta_{t,L}^k\bfB_t^\top\phi(s_{t,0}^k)$. 
We have
\begin{align*}
   \|x_{t+1}^k \| = &\|\bfB_{t+1}\omega_{t+1}^k- z^{k,*}\| \\
    = & \|\bfB_{t+1}\Pi_{U_\omega}(\tilde\omega_{t+1}^k) - z^{k,*}\|\\
    =&\|\bfB_{t+1}\bfR_{t+1}\tilde\omega_{t+1}^k-z^{k,*} + \bfB_{t+1}\Pi_{U_\omega}(\tilde\omega_{t+1}^k)-\bfB_{t+1}\bfR_{t+1}\tilde\omega_{t+1}^k \|\\
    =&\|\bar\bfB_{t+1}\tilde\omega_{t+1}^k-z^{k,*} + \bfB_{t+1}\Pi_{U_\omega}(\tilde\omega_{t+1}^k)-\bfB_{t+1}\bfR_{t+1}\tilde\omega_{t+1}^k\| \\
    =& \| \tilde{x}_{t+1}^k + \bfB_{t+1}\Pi_{U_\omega}(\tilde\omega_{t+1}^k)-\bfB_{t+1}\bfR_{t+1}\tilde\omega_{t+1}^k\|. 
\end{align*}
Taking the square on both sides, we get   
\begin{equation}
\label{eq: upper bound: key intermediate}
\begin{aligned}
    \|x_{t+1}^k\|^2 = & \|\tilde x_{t+1}^k\|^2  + \|\bfB_{t+1}\Pi_{U_\omega}(\tilde\omega_{t+1}^k)-\bfB_{t+1}\bfR_{t+1}\tilde\omega_{t+1}^k\|^2\\
    & + 2 \langle\bfB_{t+1}\Pi_{U_\omega}(\tilde\omega_{t+1}^k)-\bfB_{t+1}\bfR_{t+1}\tilde\omega_{t+1}^k, \tilde x_{t+1}^k \rangle. 
\end{aligned}    
\end{equation}

\underline{We bound $\|\bfB_{t+1}\Pi_{U_\omega}(\tilde\omega_{t+1}^k)-\bfB_{t+1}\bfR_{t+1}\tilde\omega_{t+1}^k\|^2$ as follows:} \\ 
We have \begin{align*}
    &\|\bfB_{t+1}\Pi_{U_\omega}(\tilde\omega_{t+1}^k)-\bfB_{t+1}\bfR_{t+1}\tilde\omega_{t+1}^k\| \\
    =& \|\Pi_{U_\omega}(\tilde\omega_{t+1}^k)-\bfR_{t+1}\tilde\omega_{t+1}^k\|\\
    =& \|\Pi_{U_\omega}(\tilde\omega_{t+1}^k)-\tilde\omega_{t+1}^k+\tilde\omega_{t+1}^k-\bfR_{t+1}\tilde\omega_{t+1}^k\|\\
    \leq&\|\Pi_{U_\omega}(\tilde\omega_{t+1}^k)-\tilde\omega_{t+1}^k\|+\|(\bfI-\bfR_{t+1})\tilde\omega_{t+1}^k\|\\
    \leq&\|\frac{\beta}{L}\delta_{t,L}^k\bfB_t^\top\phi(s_{t,0}^k)\|+\|(\bfI-\bfR_{t+1})\tilde\omega_{t+1}^k\|
\end{align*}
where the last inequality follows from  Eq.\,(\ref{eqn: R related bounds: R-I}) and we use $\frac{U_{\delta,L}}{L}\beta\leq U_\omega$ for large $T$ such that $\|\tilde\omega_{t+1}^k\|\leq 2 U_\omega$.
Then, 
\begin{align*}
    &\|\bfB_{t+1}\Pi_{U_\omega}(\tilde\omega_{t+1}^k)-\bfB_{t+1}\bfR_{t+1}\tilde\omega_{t+1}^k\| ^2\\
    \leq& 2\frac{\beta^2}{L^2}\|\delta_{t,L}^k\bfB_t^\top\phi(s_{t,0}^k)\|^2+32U_\omega^2 \|\bfQ\|_F^4.
\end{align*}

For the first term, we further decompose it as:
\begin{align*}
    &\|\delta_{t,L}^k\bfB_t^\top\phi(s_{t,0}^k)\|^2\leq \|\delta_{t,L}^k\phi(s_{t,0}^k)\|^2\\
    =& \|\tilde A_{t,L}^kx_t^k+\bfb_{t,L}^k+L y_t^k\phi(s_{t,0}^k)\|^2  ~~~~(\text{by Proposition \ref{prop: key intermediate}})\\
    \leq& 3\|\tilde A_{t,L}^kx_t^k\|^2+3\|\bfb_{t,L}^k\|^2+3\|L y_t^k\phi(s_{t,0}^k)\|^2\\
    \leq& 12\|x_t^k\|^2+3\|\bfb_{t,L}^k\|^2+3L^2|y_t^k|^2.
\end{align*}
Taking conditional expectation over the Markovian trajectory up to $t-\tau$, we get 
\begin{align}
\label{eqn:part of projection error}
   &\bbE_{t-\tau} \|\delta_{t,L}^k\bfB_t^\top\phi(s_{t,0}^k)\|^2\nonumber\\
   \leq&12\bbE_{t-\tau}\|x_t^k\|^2+3\bbE_{t-\tau}\|\bfb_{t,L}^k\|^2+3L^2\bbE_{t-\tau}|y_t^k|^2\nonumber\\
   \leq&12\bbE_{t-\tau}\|x_t^k\|^2+3U_\delta^2(1+2\frac{m\rho}{1-\rho})L+3L^2\bbE_{t-\tau}|y_t^k|^2. ~~ (\text{by Lemma \ref{lmm: local noise reduction} and tower property} )
\end{align}

Finally, the second term in Eq.\,(\ref{eq: upper bound: key intermediate}) can be bounded as 
\begin{align}
\label{eq: upper bound of qr-proj squared} 
    &\bbE_{t-\tau}\|\bfB_{t+1}\Pi_{U_\omega}(\tilde\omega_{t+1}^k)-\bfB_{t+1}\bfR_{t+1}\tilde\omega_{t+1}^k\|^2 \nonumber\\
    \leq& 24\frac{\beta^2}{L^2}\bbE_{t-\tau}\|x_t^k\|^2+6\frac{\beta^2}{L}U_\delta^2(1+2\frac{m\rho}{1-\rho})+6\beta^2\bbE_{t-\tau}|y_t^k|^2+32U_\omega^2 \bbE_{t-\tau}\|\bfQ\|_F^4.
\end{align}

{ 
\underline{As for $2 \langle\bfB_{t+1}\Pi_{U_\omega}(\tilde\omega_{t+1}^k)-\bfB_{t+1}\bfR_{t+1}\tilde\omega_{t+1}^k, \tilde x_{t+1}^k \rangle$, }

we first write $\tilde x_{t+1}^k$ as 
\begin{align*}
\tilde x_{t+1}^k 
= \bar \bfB_{t+1}\tilde \omega_{t+1}^k - \bfB^* \omega^{*,k}  
=  \bfB_{t+1}\omega_{t+1}^k - \bfB^* \omega^{k,*}  + \bar \bfB_{t+1}\tilde \omega_{t+1}^k - \bfB_{t+1}\omega_{t+1}^k.  
\end{align*}
Thus, 
\begin{align*}
&\langle\bfB_{t+1}\Pi_{U_\omega}(\tilde\omega_{t+1}^k)-\bfB_{t+1}\bfR_{t+1}\tilde\omega_{t+1}^k, \tilde x_{t+1}^k \rangle  \\
& = \underbrace{\langle\bfB_{t+1}\Pi_{U_\omega}(\tilde\omega_{t+1}^k)-\bfB_{t+1}\bfR_{t+1}\tilde\omega_{t+1}^k, \bfB_{t+1}\omega_{t+1}^k - \bfB^* \omega^{k,*} \rangle}_{I_1} \\
&\quad +  \underbrace{\langle\bfB_{t+1}\Pi_{U_\omega}(\tilde\omega_{t+1}^k)-\bfB_{t+1}\bfR_{t+1}\tilde\omega_{t+1}^k, \bar \bfB_{t+1}\tilde \omega_{t+1}^k - \bfB_{t+1}\omega_{t+1}^k \rangle}_{I_2}.  
\end{align*}
Both $I_1$ and $I_2$ can be well controlled by the projection error of $\tilde{w}$ and the perturbation error that arises from the QR decomposition of $\bar{\bfB}$. 
We bound $I_1$ and $I_2$ separately.

On $I_2$, we have 
\begin{align*}
\bar \bfB_{t+1}\tilde \omega_{t+1}^k - \bfB_{t+1}\omega_{t+1}^k  
&= \bfB_{t+1}\bfR_{t+1} \tilde\omega_{t+1}^k - \bfB_{t+1}\omega_{t+1}^k \\
& = \bfB_{t+1}\bfR_{t+1} \tilde\omega_{t+1}^k - \bfB_{t+1}\bfR_{t+1}\omega_{t+1}^k + \bfB_{t+1}\bfR_{t+1}\omega_{t+1}^k - \bfB_{t+1}\omega_{t+1}^k  \\
& = \bfB_{t+1}\bfR_{t+1} \pth{\tilde\omega_{t+1}^k - \omega_{t+1}^k} 
+ \bfB_{t+1}\pth{\bfR_{t+1}-\bfI}\omega_{t+1}^k. 
\end{align*}
So, $I_2$ can be rewritten as  
\begin{align*}
&\langle\bfB_{t+1}\Pi_{U_\omega}(\tilde\omega_{t+1}^k)-\bfB_{t+1}\bfR_{t+1}\tilde\omega_{t+1}^k, \bar \bfB_{t+1}\tilde \omega_{t+1}^k - \bfB_{t+1}\omega_{t+1}^k \rangle    \\
& = \underbrace{\langle\bfB_{t+1}\Pi_{U_\omega}(\tilde\omega_{t+1}^k)-\bfB_{t+1}\bfR_{t+1}\tilde\omega_{t+1}^k, \bfB_{t+1}\bfR_{t+1} \pth{\tilde\omega_{t+1}^k - \omega_{t+1}^k}  \rangle}_{I_{21}} \\
& \quad + \underbrace{\langle\bfB_{t+1}\Pi_{U_\omega}(\tilde\omega_{t+1}^k)-\bfB_{t+1}\bfR_{t+1}\tilde\omega_{t+1}^k, \bfB_{t+1}\pth{\bfR_{t+1}-\bfI}\omega_{t+1}^k \rangle}_{I_{22}}.
\end{align*} 
For $I_{21}$, by Cauchy Schwartz, we have 
\begin{align*}
    &I_{21} \leq \|\bfB_{t+1}\Pi_{U_\omega}(\tilde\omega_{t+1}^k)-\bfB_{t+1}\bfR_{t+1}\tilde\omega_{t+1}^k\| \| \bfB_{t+1}\bfR_{t+1} \pth{\tilde\omega_{t+1}^k - \omega_{t+1}^k}\|\\
    \leq&\frac{1}{2}\|\bfB_{t+1}\Pi_{U_\omega}(\tilde\omega_{t+1}^k)-\bfB_{t+1}\bfR_{t+1}\tilde\omega_{t+1}^k\| ^2+ \frac{1}{2}\| \bfB_{t+1}\bfR_{t+1} \pth{\tilde\omega_{t+1}^k - \omega_{t+1}^k}\|^2\\
    \leq&\frac{1}{2}\|\bfB_{t+1}\Pi_{U_\omega}(\tilde\omega_{t+1}^k)-\bfB_{t+1}\bfR_{t+1}\tilde\omega_{t+1}^k\| ^2+ \|\tilde\omega_{t+1}^k - \omega_{t+1}^k\|^2.
\end{align*}
For the last inequality, we use $\|\bfR_{t+1}\|\leq 2$ for when $2U_\delta^2 U_\omega^2 \zeta^2\leq 1/2$. 
As for $\|\tilde\omega_{t+1}^k - \omega_{t+1}^k\|^2$, by \eqref{eqn:part of projection error}, we have
\begin{align*}
    &\bbE_{t-\tau}\|\omega_{t+1}^{k}-\tilde\omega_{t+1}^k\|^2
    \leq
    \bbE_{t-\tau}\|\frac{\beta}{L}\delta_{t,L}^k\bfB_t^\top \phi(s_{t,0}^k)\|^2\\
    \leq&12\frac{\beta^2}{L^2}\bbE_{t-\tau}\|x_t^k\|^2+3\frac{\beta^2}{L}U_\delta^2(1+2\frac{m\rho}{1-\rho})+3\beta^2\bbE_{t-\tau}|y_t^k|^2.
\end{align*}

For $I_{22}$, we have 
\begin{align*}
    &I_{22} \leq \|\bfB_{t+1}\Pi_{U_\omega}(\tilde\omega_{t+1}^k)-\bfB_{t+1}\bfR_{t+1}\tilde\omega_{t+1}^k\| \| \bfB_{t+1}\pth{\bfR_{t+1}-\bfI}\omega_{t+1}^k\|\\
    \leq&\frac{1}{2}\|\bfB_{t+1}\Pi_{U_\omega}(\tilde\omega_{t+1}^k)-\bfB_{t+1}\bfR_{t+1}\tilde\omega_{t+1}^k\|^2 + \frac{1}{2}\| \bfB_{t+1}\pth{\bfR_{t+1}-\bfI}\omega_{t+1}^k\|^2\\
    \leq&\frac{1}{2}\|\bfB_{t+1}\Pi_{U_\omega}(\tilde\omega_{t+1}^k)-\bfB_{t+1}\bfR_{t+1}\tilde\omega_{t+1}^k\|^2 + \frac{1}{2}U_\omega^2\|\bfR_{t+1}-\bfI\|^2\\
    \leq&\frac{1}{2}\|\bfB_{t+1}\Pi_{U_\omega}(\tilde\omega_{t+1}^k)-\bfB_{t+1}\bfR_{t+1}\tilde\omega_{t+1}^k\|^2 + 2U_\omega^2\| \bfQ_t\|_F^4.  ~~~~\text{(by \eqref{eqn: R related bounds: R-I})}
\end{align*}

Combining $I_{21}$ and $I_{22}$, we get
\begin{align*}
    &\bbE_{t-\tau}[I_2]\leq \bbE_{t-\tau}\|\bfB_{t+1}\Pi_{U_\omega}(\tilde\omega_{t+1}^k)-\bfB_{t+1}\bfR_{t+1}\tilde\omega_{t+1}^k\|^2 \\
    &+12\frac{\beta^2}{L^2}\bbE_{t-\tau}\|x_t^k\|^2+3\frac{\beta^2}{L}U_\delta^2(1+2\frac{m\rho}{1-\rho})+3\beta^2\bbE_{t-\tau}|y_t^k|^2+ 2U_\omega^2 \bbE_{t-\tau}\| \bfQ_t\|_F^4\\
    \leq&  36\frac{\beta^2}{L^2}\bbE_{t-\tau}\|x_t^k\|^2+9\frac{\beta^2}{L}U_\delta^2(1+2\frac{m\rho}{1-\rho})+9\beta^2\bbE_{t-\tau}|y_t^k|^2+34U_\omega^2 \bbE_{t-\tau}\|\bfQ_t\|_F^4.
     ~~~~~~\text{(by \eqref{eq: upper bound of qr-proj squared})}\\
\end{align*}

Bounding $I_1$ is slightly more involved. 
We have 
\begin{align*}
&I_1 = \langle\bfB_{t+1}\Pi_{U_\omega}(\tilde\omega_{t+1}^k)-\bfB_{t+1}\bfR_{t+1}\tilde\omega_{t+1}^k, \bfB_{t+1}\omega_{t+1}^k - \bfB^* \omega^{k,*} \rangle  \\
& = \underbrace{\langle\bfB_{t+1}\Pi_{U_\omega}(\tilde\omega_{t+1}^k)-\bfB_{t+1}\tilde\omega_{t+1}^k, \bfB_{t+1}\omega_{t+1}^k - \bfB^* \omega^{k,*} \rangle}_{I_{11}} \\
& \quad +\underbrace{\langle\bfB_{t+1}\tilde\omega_{t+1}^k-\bfB_{t+1}\bfR_{t+1}\tilde\omega_{t+1}^k, \bfB_{t+1}\omega_{t+1}^k - \bfB^* \omega^{k,*} \rangle}_{I_{12}}
\end{align*}
For $I_{11}$, let $\calR = \bfU\bfV^\top,$ where $\bfU\Sigma\bfV^\top=\bfB_{t+1}^\top\bfB^*$ is the SVD. Note that $\calR$ is the optimal rotation matrix that aligns $\bfB_{t+1}$ and $\bfB^*$ in the sense that $\calR=\text{argmin}_{\calR^\top\calR=\bfI}\|\bfB_{t+1}\calR-\bfB^*\|_F$. Then, we have
\begin{align*}
&I_{11}=\langle\bfB_{t+1}(\omega_{t+1}^k - \tilde\omega_{t+1}^k), \bfB_{t+1}\calR\calR^{-1}\omega_{t+1}^k - \bfB^* \omega^{k,*} \rangle ~~~ (\text{by rotation})\\
    =& \langle\bfB_{t+1}\calR\calR^{-1}(\omega_{t+1}^k - \tilde\omega_{t+1}^k),(\bfB_{t+1}\calR -\bfB^*)\calR^{-1}\omega_{t+1}^k\rangle\\
    &+\langle\bfB_{t+1}(\omega_{t+1}^k - \tilde\omega_{t+1}^k),\bfB^*(\calR^{-1}\omega_{t+1}^k- \omega^{k,*}) \rangle\\
    =& (\calR^{-1}(\omega_{t+1}^k - \tilde\omega_{t+1}^k))^{\top} (\bfI -(\bfB_{t+1}\calR)^\top\bfB^*)\calR^{-1}\omega_{t+1}^k\\
    &+\langle\bfB_{t+1}\calR\calR^{-1}(\omega_{t+1}^k - \tilde\omega_{t+1}^k),(\bfB^*-\bfB_{t+1}\calR)(\calR^{-1}\omega_{t+1}^k- \omega^{k,*}) \rangle\\
    &+ \langle\bfB_{t+1}\calR\calR^{-1}(\omega_{t+1}^k - \tilde\omega_{t+1}^k),\bfB_{t+1}\calR(\calR^{-1}\omega_{t+1}^k- \omega^{k,*}) \rangle\\
    =&(\calR^{-1}(\omega_{t+1}^k - \tilde\omega_{t+1}^k))^{\top} (\bfI -(\bfB_{t+1}\calR)^\top\bfB^*)\calR^{-1}\omega_{t+1}^k\\
    &+\langle\calR^{-1}(\omega_{t+1}^k - \tilde\omega_{t+1}^k),((\bfB_{t+1}\calR)^\top\bfB^*-\bfI)(\calR^{-1}\omega_{t+1}^k- \omega^{k,*}) \rangle\\
    &+\langle \omega_{t+1}^k - \tilde\omega_{t+1}^k, \omega_{t+1}^k- \calR\omega^{k,*} \rangle.\\
\end{align*}

We have the following facts: 
\begin{itemize}
    \item Fact 1: {Since the ball with radius $U_{\omega}$ is convex, and $\omega^{k,*}$ belongs to the ball, it holds that }
    
    $$\langle \omega_{t+1}^k - \tilde\omega_{t+1}^k, \omega_{t+1}^k- \calR\omega^{k,*} \rangle \le 0$$. 
    \item Fact 2: 
   \begin{align}
    \label{eq: fact 2}
    \|\bfI -(\bfB_{t+1}\calR)^\top\bfB^*\|\leq \|\bfB_\perp^{*\top}\bfB_{t+1}\|^2.
   \end{align} 
   The proof of Fact 2 is postponed to Appendix~\ref{app:fact2-proof}.
\end{itemize}
Therefore,
\begin{align*}
    I_{11} \leq& \norm{\omega_{t+1}^k - \tilde\omega_{t+1}^k}\norm{\bfI -(\bfB_{t+1}\calR)^\top\bfB^*} \norm{\omega_{t+1}^k} \\
&\quad+ \norm{\omega_{t+1}^k - \tilde\omega_{t+1}^k}\norm{\bfI -(\bfB_{t+1}\calR)^\top\bfB^*}\norm{\calR^{-1}\omega_{t+1}^k - \omega^{k,*}} \\
\leq&3U_\omega\norm{\omega_{t+1}^k - \tilde\omega_{t+1}^k}\|m_{t+1}\|^2.
\end{align*}
Similarly, for $I_{12}$, we have
\begin{align*}
    &I_{12} = \langle\bfB_{t+1}\tilde\omega_{t+1}^k-\bfB_{t+1}\bfR_{t+1}\tilde\omega_{t+1}^k, \bfB_{t+1}\omega_{t+1}^k - \bfB^* \omega^{k,*} \rangle\\
    &= \langle\bfB_{t+1}(\bfI-\bfR_{t+1})\tilde\omega_{t+1}^k, \bfB_{t+1}\omega_{t+1}^k - \bfB^* \omega^{k,*} \rangle\\
    & \leq 4 U_\omega^2 \|\bfI-\bfR_{t+1}\|\\
    &\leq 8 U_\omega^2 \|\bfQ_t\|_F^2.
\end{align*}
Summing up $I_1$ and $I_2$ and multiplying by 2, we have
\begin{align}
\label{eq: x_t+1 intermediate bound cross term}
   &  2\bbE_{t-\tau}\langle\bfB_{t+1}\Pi_{U_\omega}(\tilde\omega_{t+1}^k)-\bfB_{t+1}\bfR_{t+1}\tilde\omega_{t+1}^k, \tilde x_{t+1}^k \rangle \nonumber \\
   \leq& 72\frac{\beta^2}{L^2}\bbE_{t-\tau}\|x_t^k\|^2+18\frac{\beta^2}{L}U_\delta^2(1+2\frac{m\rho}{1-\rho})+18\beta^2\bbE_{t-\tau}|y_t^k|^2+68U_\omega^2 \bbE_{t-\tau}\|\bfQ_t\|_F^4\nonumber\\
     &+6U_\omega\bbE_{t-\tau}\norm{\omega_{t+1}^k - \tilde\omega_{t+1}^k}\|m_{t+1}\|^2+16U_\omega^2 \bbE_{t-\tau}\|\bfQ_t\|_F^2.
\end{align}
}

\underline{We bound $\|\tilde x_{t+1}^k\|$ as follows:}
\begin{align*}
    \tilde x_{t+1}^k
    & = \bar\bfB_{t+1}\tilde\omega_{t+1}^k- z^{k,*} \\ 
    & = \pth{\bfB_{t} + \frac{\zeta}{KL}\sum_{i=1}^K \bfP_{t,\perp}\delta_{t,L}^i\phi(s_{t,0}^i)(\omega_t^i)^\top} \pth{\omega_{t}^k + \frac{\beta}{L} \delta_{t,L}^k\bfB_t^{\top}\phi(s_{t,0}^k)} - z^{k,*}  \\
    &= \bfB_t\omega_{t}^k - z^{k,*} + \frac{\beta}{L}\bfB_t(\delta_{t,L}^k\bfB_t^\top\phi(s_{t,0}^k)) + \frac{\zeta}{KL}\sum_{i=1}^K  \bfP_{t,\perp}\delta_{t,L}^i\phi(s_{t,0}^i)(\omega_t^i)^\top\omega_t^k \\
    &\qquad + \zeta \beta(\frac{1}{KL^2}\sum_{i=1}^K  \bfP_{t,\perp}\delta_{t,L}^i\phi(s_{t,0}^i)(\omega_t^i)^\top) \pth{\delta_{t,L}^k\bfB_t^{\top}\phi(s_{t,0}^k)}. 
\end{align*}
Squaring both sides, we have 
\begin{align}
\label{eqn:xtilde squared decomp}
    \|\tilde x_{t+1}^k\|^2=&\|\bar\bfB_{t+1}\tilde\omega_{t+1}^k- z^{k,*}\|^2 \nonumber\\
    \leq& \|\bfB_t\omega_{t}^k- z^{k,*}\|^2+\underbrace{2\frac{\beta}{L}\langle\bfB_t\omega_{t}^k- z^{k,*}, \bfB_t\bfB_t^\top\delta_{t,L}^k\phi(s_{t,0}^k)\rangle}_{I_1} \nonumber\\ 
    &\quad +\underbrace{2\zeta\langle\bfB_t\omega_{t}^k- z^{k,*}, \frac{1}{KL}\sum_{i=1}^K  \bfP_{t,\perp}\phi(s_{t,0}^i)\delta_{t,L}^i(\omega_t^i)^\top\omega_t^k\rangle}_{I_2} \nonumber\\
    & \quad + \underbrace{2\frac{\zeta \beta}{L^2} \iprod{\bfB_t\omega_{t}^k- z^{k,*}}{(\frac{1}{K}\sum_{i=1}^K  \bfP_{t,\perp}\delta_{t,L}^i\phi(s_{t,0}^i)(\omega_t^i)^\top) \pth{\delta_{t,L}^k\bfB_t^{\top}\phi(s_{t,0}^k)}}}_{I_3}
    \nonumber\\
    &\quad + \underbrace{3\frac{\beta^2}{L^2} \|\bfB_t(\delta_{t,L}^k\bfB_t^\top\phi(s_{t,0}^k))\|^2}_{I_4}
    + \underbrace{3\frac{\zeta^2}{L^2} \|\frac{1}{K}\sum_{i=1}^K  \bfP_{t,\perp}\delta_{t,L}^i\phi(s_{t,0}^i)(\omega_t^i)^\top\omega_t^k\|^2 }_{I_5}\nonumber\\
    & \quad + \underbrace{3\frac{\beta^2\zeta^2}{L^4} \|(\frac{1}{K}\sum_{i=1}^K  \bfP_{t,\perp}\delta_{t,L}^i\phi(s_{t,0}^i)(\omega_t^i)^\top) \pth{\delta_{t,L}^k\bfB_t^{\top}\phi(s_{t,0}^k)}\|^2}_{I_6}. 
\end{align}

We bound these terms separately.

\underline{Bounding $I_1$.} Roughly speaking, the main drift of the decay in $\|x_{t}^k\|$ may arise from $2\beta\langle x_{t}^k, \bfB_t\bfB_t^\top\phi(s_t^k)\delta_t^k\rangle$ (i.e., $I_1$ in Eq.\,(\ref{eqn:xtilde squared decomp})).
However, the environmental heterogeneity significantly complicates the characterization of this term. Specifically, by Proposition \ref{prop: key intermediate}, $I_1$ can be written as:
\begin{align}
\nonumber
&2\frac{\beta}{L}\langle x_{t}^k, \bfB_t\bfB_t^\top\phi(s_{t,0}^k)\delta_{t,L}^k\rangle\\
\nonumber 
=&2\frac{\beta}{L}\langle x_{t}^k, \bfB_t\bfB_t^\top\xi_{t,L}^k \rangle+2\frac{\beta}{L}\langle x_{t}^k, \bfB_t\bfB_t^\top A_L^kx_t^k \rangle+2\beta\langle x_{t}^k, \bfB_t\bfB_t^\top y_t^k \phi(s_{t,0}^k)\rangle  \\
=& \underbrace{2\frac{\beta}{L}\langle x_{t}^k, \bfP_t  A_L^kx_t^k \rangle}_{I_{11}} + \underbrace{2\frac{\beta}{L}\langle x_{t}^k, \bfP_t \xi_{t,L}^k \rangle}_{I_{12}} + \underbrace{2{\beta}\langle x_{t}^k, \bfP_t  y_t^k \phi(s_{t,0}^k)\rangle}_{I_{13}},
\label{eq: roughly drift of upper bound}
\end{align}     
recalling that $\bfP_t = \bfB_t\bfB_t^\top$. Inside $I_1$, the main negative drift arises from the term $\langle x_{t}^k, \bfP_t  A_L^kx_t^k \rangle$. Intuitively, in traditional single-agent or homogeneous environment settings, $\langle x_{t}^k, \bfP_t  A_L^kx_t^k \rangle \approx \langle x_{t}^k, Ax_t^k \rangle$, which is mainly controlled by the spectrum of $A$, with the desired property directly assumed in Assumption \ref{ass: per-agent full exploration}. However, in the presence of environmental heterogeneity, Assumption \ref{ass: per-agent full exploration} does not directly guarantee a negative drift of the $x_t^k$ due to the existence of $\bfP_t$ and its intricate interplay with the heterogeneous $A_L^k$. 
Specifically, (1) $\bfP_t$ is not of full rank, (2) the local head error will be distorted in a different manner due to the product $\bfP_t A_L^k$, and (3) $\bfP_t$ varies over time.  
To address this, we further decompose $I_{11}$ as 
\begin{align}
\label{eq: local head: upper bound: drift: intermediate} 
     &2\frac{\beta}{L}\langle x_t^k, \bfP_t A_L^k x_t^k \rangle \nonumber \\
    =&  2\frac{\beta}{L}\langle x_t^k, (\bfP_t-\bfP^*) A_L^k x_t^k \rangle+2\frac{\beta}{L}\langle x_t^k, \bfP^* A_L^k x_t^k \rangle \nonumber\\
    =&  \underbrace{2\frac{\beta}{L}\langle x_t^k,  A_L^k x_t^k \rangle}_{I_{111}}  + \underbrace{2\frac{\beta}{L}\langle x_t^k, (\bfP_t-\bfP^*) A_L^k x_t^k \rangle}_{I_{112}} - \underbrace{2\frac{\beta}{L}\langle x_t^k, \bfP_\perp^* A_L^k x_t^k \rangle}_{I_{113}}.
\end{align}
The first term of Eq.\,(\ref{eq: local head: upper bound: drift: intermediate}), $I_{111}$, is a contraction by Assumption \ref{ass: per-agent full exploration}. Specifically,
\begin{align*}
    2\frac{\beta}{L}\langle x_t^k,  A_L^k x_t^k \rangle\leq -2\lambda\beta\|x_t^k\|^2.
\end{align*}

$I_{112}$ can be controlled via the principal angle distance between the subspace estimate $\bfB_t$ and the underlying truth $\bfB^*$. Since
\begin{align}
\label{eqn: expected A crude bound}
    \|A_L^k\| \le  \mathbb{E}_{s^{0}\sim\mu^k,a^0\sim \pi,...,s^{(L)}\sim P^k}\lnorm{\phi(s^{(0)})(\phi(s^{(L)})-\phi(s^{(0)}))^\top}{} \leq  2,
\end{align} 
it holds that 
\begin{align*}
    2\frac{\beta}{L}\langle x_t^k, (\bfP_t-\bfP^*) A_L^k x_t^k \rangle\leq 8U_\omega\frac{\beta}{L}\|\bfP_t-\bfP^*\|\|x_t^k\|.
\end{align*}
The third term, $I_{113}$, can be bounded as  
\begin{align}
\label{eq: aaa}
-2\frac{\beta}{L}\langle x_t^k, \bfP_\perp^* A_L^k x_t^k \rangle 
&= -2\frac{\beta}{L}(\bfP_\perp^*x_t^k)^\top A_L^k x_t^k = - 2\frac{\beta}{L}(\bfP_\perp^* \pth{\bfB_t\omega_t^k - \bfB^*\omega^{k,*}})^\top A_L^k x_t^k\nonumber\\
& = -2\frac{\beta}{L}\pth{\bfB_\perp^*\bfB_\perp^{*\top} \pth{\bfB_t\omega_t^k - \bfB^*\omega^{k,*}}}^{\top}A_L^k x_t^k \nonumber\\
& = - 2\frac{\beta}{L}(\bfB_\perp^*\bfB_\perp^{*\top} \bfB_t\omega_t^k)^\top A_L^k x_t^k\nonumber\\
&\leq 4\frac{\beta}{L} U_\omega\|\bfB_\perp^{*\top} \bfB_t\|\| x_t^k\|.
\end{align}
wherein the local head error $x_t^k$ and the principal angle distance $\|\bfB_\perp^{*\top} \bfB_t\|$ are also coupled.   
Then, 
\begin{align}
\label{eqn: first term in first ip final}
    &I_{11}\leq -2\lambda\beta\|x_t^k\|^2+12U_\omega\frac{\beta}{L}  \|m_t\|\|x_t^k\|.
\end{align}

It is worth noting that when $\|\bfB_\perp^{*\top} \bfB_t\| =0$, i.e., when $\bfB_t$ and $\bfB^*$ span the same $r$-dimensional subspace, the second and the third terms in Eq.\,(\ref{eq: local head: upper bound: drift: intermediate}) become zero, and Eq.\,(\ref{eq: local head: upper bound: drift: intermediate}) reduces to the standard negative drift term. We postpone the characterization of the convergence of $\|\bfB_\perp^{*\top} \bfB_t\|$ to {Lemma \ref{lm: PAD analysis}}.

For $I_{13}$,we have
\begin{align}
\label{eqn: third term in first ip final}
    I_{13}\leq 2\beta \|x_t^k\||y_t^k| .
\end{align}

Recall $\tilde{b}_{t,L}^k = \sum_{\ell=0}^{L-1}(r_{t,\ell}^k - J^k)\phi(s_{t,0}^k);\quad \bar{b}_L^k=  \bbE_{s^{(0)}\sim\mu^k, (a^{(\ell)}, s^{(\ell)}) \sim \pi \otimes P^k ~ \forall \ell\in [0, L-1]}\qth{\sum_{\ell=0}^{L-1}(r(s^{(\ell)},a^{(\ell)}) - J^k)\phi(s^{(0)})}.$

For $I_{12}$, recall from \eqref{eq: Markovian noise} that, 
\begin{align}
\label{eqn: xi decomp}
\xi_{t,L}^k  
= \tilde b_{t,L}^k - \bar b_L^k+ (\tilde A_{t,L}^k-A_L^k)\bfB_t\omega_t^k.
\end{align}
Then,
\begin{align*}
    &2\frac{\beta}{L}\iprod{x_t^k}{\bfP_t\xi_{t,L}^k}\\
    =&2\frac{\beta}{L}\iprod{x_t^k}{\bfP_t\big[\tilde b_{t,L}^k - \bar b_L^k+ (\tilde A_{t,L}^k-A_L^k)\bfB_t\omega_t^k\big]}\\
     =&2\frac{\beta}{L}\iprod{x_t^k}{\bfP_t(\tilde b_{t,L}^k - \bar b_L^k)} + 2\frac{\beta}{L}\iprod{x_t^k}{\bfP_t (\tilde A_{t,L}^k-A_L^k)\bfB_t\omega_t^k}\\
     =&\underbrace{2\frac{\beta}{L}\iprod{x_t^k-x_{t-\tau}^k}{\bfP_t(\tilde b_{t,L}^k - \bar b_L^k)}}_{I_{121}} 
     +\underbrace{2\frac{\beta}{L}\iprod{x_{t-\tau}^k}{(\bfP_t-\bfP_{t-\tau})(\tilde b_{t,L}^k - \bar b_L^k)}}_{I_{122}} 
     +\underbrace{2\frac{\beta}{L}\iprod{x_{t-\tau}^k}{\bfP_{t-\tau}(\tilde b_{t,L}^k - \bar b_L^k)}}_{I_{123}}\\
     & + \underbrace{2\frac{\beta}{L}\iprod{x_t^k-x_{t-\tau}^k}{\bfP_t (\tilde A_{t,L}^k-A_L^k)\bfB_t\omega_t^k}}_{I_{124}}
     +\underbrace{2\frac{\beta}{L}\iprod{x_{t-\tau}^k}{(\bfP_t-\bfP_{t-\tau}) (\tilde A_{t,L}^k-A_L^k)\bfB_t\omega_t^k}}_{I_{125}}\\
     &+\underbrace{2\frac{\beta}{L}\iprod{x_{t-\tau}^k}{\bfP_{t-\tau} (\tilde A_{t,L}^k-A_L^k)(\bfB_t\omega_t^k-\bfB_{t-\tau}\omega_{t-\tau}^k)}}_{I_{126}}
     +\underbrace{2\frac{\beta}{L}\iprod{x_{t-\tau}^k}{\bfP_{t-\tau} (\tilde A_{t,L}^k-A_L^k)\bfB_{t-\tau}\omega_{t-\tau}^k}}_{I_{127}}.
\end{align*}
For $I_{121}$, we have under conditional expectation,
\begin{align}
\label{eqn:121}
    &2\frac{\beta}{L}\bbE_{t-\tau}\iprod{x_t^k-x_{t-\tau}^k}{\bfP_t\big(\tilde b_{t,L}^k - \bar b_L^k\big)}\nonumber\\
    \leq&2\frac{\beta}{L}\bbE_{t-\tau}\|x_t^k-x_{t-\tau}^k\| \|\bfP_t\|\|\tilde b_{t,L}^k - \bar b_L^k\|\nonumber\\
    \leq&2\frac{\beta}{L}\bbE_{t-\tau}\|\bfB_t\omega_t^k-\bfB_{t-\tau}\omega_{t-\tau}^k\|\|\tilde b_{t,L}^k - \bar b_L^k\|\nonumber\\
    \leq& \frac{1}{2}\bbE_{t-\tau}\|\bfB_t\omega_t^k-\bfB_{t-\tau}\omega_{t-\tau}^k\|^2 +  2\frac{\beta^2}{L^2}\bbE_{t-\tau}\|\tilde b_{t,L}^k - \bar b_L^k\|^2\nonumber\\
    \leq& \bbE_{t-\tau}\|\bfB_t\omega_t^k-\bfB_t\omega_{t-\tau}^k\|^2+\bbE_{t-\tau}\|\bfB_t\omega_{t-\tau}^k- \bfB_{t-\tau}\omega_{t-\tau}^k\|^2 +  2\frac{\beta^2}{L^2}\bbE_{t-\tau}\|\tilde b_{t,L}^k - \bar b_L^k\|^2\nonumber\\
    \leq& \bbE_{t-\tau}\|\omega_t^k-\omega_{t-\tau}^k\|^2+U_\omega\bbE_{t-\tau}\|\bfB_t- \bfB_{t-\tau}\|^2 +  2\frac{\beta^2}{L^2}\bbE_{t-\tau}\|\tilde b_{t,L}^k - \bar b_L^k\|^2.
\end{align}

Similarly, for $I_{122}$, we have
\begin{align}
\label{eqn:122}
    &2\frac{\beta}{L} \bbE_{t-\tau}\iprod{x_{t-\tau}^k}{(\bfP_t-\bfP_{t-\tau})\big[\tilde b_{t,L}^k - \bar b_L^k \big]}\nonumber\\
    \leq&2\frac{\beta}{L} \bbE_{t-\tau}\|x_{t-\tau}^k\| \|\bfP_t-\bfP_{t-\tau}\|\|\tilde b_{t,L}^k - \bar b_L^k \| ~~~~ (\text{by Cauchy Schwartz})\nonumber\\
    \leq&8U_\omega \frac{\beta}{L}  \bbE_{t-\tau}\|\bfB_t-\bfB_{t-\tau}\|\|\tilde b_{t,L}^k - \bar b_L^k \|\nonumber\\
    \leq&  \frac{1}{2}\bbE_{t-\tau}\|\bfB_t-\bfB_{t-\tau}\|^2 +32U_\omega^2 \frac{\beta^2}{L^2}\bbE_{t-\tau}\|\tilde b_{t,L}^k - \bar b_L^k \|^2.
\end{align}
For $I_{124}$,
\begin{align}
\label{eqn:124}
    &2\frac{\beta}{L}\bbE_{t-\tau}\iprod{x_t^k-x_{t-\tau}^k}{\bfP_t (\tilde A_{t,L}^k-A_L^k)\bfB_t\omega_t^k}\nonumber\\
    \leq& 2U_\omega\frac{\beta}{L} \bbE_{t-\tau}\|\bfB_t\omega_t^k -\bfB_{t-\tau}\omega_{t-\tau}^k\| \|\tilde A_{t,L}^k-A_L^k\|\nonumber\\
    \leq& \frac{1}{2}\bbE_{t-\tau}\|\bfB_t\omega_t^k -\bfB_{t-\tau}\omega_{t-\tau}^k\|^2 + 2U_\omega^2\frac{\beta^2}{L^2} \bbE_{t-\tau}\|\tilde A_{t,L}^k-A_L^k\|^2\nonumber\\
    \leq&  \bbE_{t-\tau}\|\omega_t^k-\omega_{t-\tau}^k\|^2+U_\omega\bbE_{t-\tau}\|\bfB_t- \bfB_{t-\tau}\|^2+ 2U_\omega^2\frac{\beta^2}{L^2} \bbE_{t-\tau}\|\tilde A_{t,L}^k-A_L^k\|^2 .
\end{align}
For $I_{125}$,
\begin{align}
\label{eqn:125}
    &2\frac{\beta}{L}\bbE_{t-\tau}\iprod{x_{t-\tau}^k}{(\bfP_t-\bfP_{t-\tau}) (\tilde A_{t,L}^k-A_L^k)\bfB_t\omega_t^k}\nonumber\\
    \leq& 8\frac{\beta}{L} U_\omega^2\bbE_{t-\tau}\|\bfB_t-\bfB_{t-\tau}\| \|\tilde A_{t,L}^k-A_L^k\|\nonumber \\
    \leq& \frac{1}{2}\bbE_{t-\tau}\|\bfB_t-\bfB_{t-\tau}\|^2+ 32U_\omega^4\frac{\beta^2}{L^2} \bbE_{t-\tau}\|\tilde A_{t,L}^k-A_L^k\|^2.
\end{align}
For $I_{126}$, we have,
\begin{align}
\label{eqn:126}
    &2\frac{\beta}{L}\bbE_{t-\tau}\iprod{x_{t-\tau}^k}{\bfP_{t-\tau} (\tilde A_{t,L}^k-A_L^k)(\bfB_t\omega_t^k-\bfB_{t-\tau}\omega_{t-\tau}^k)}\nonumber\\
    \leq & 4\frac{\beta}{L} U_\omega \bbE_{t-\tau}\|\tilde A_{t,L}^k-A_L^k\| \|\bfB_t\omega_t^k-\bfB_{t-\tau}\omega_{t-\tau}^k\|\nonumber\\
    \leq& 8\frac{\beta^2}{L^2} U_\omega^2 \bbE_{t-\tau}\|\tilde A_{t,L}^k-A_L^k\|^2 +\bbE_{t-\tau}\|\omega_t^k-\omega_{t-\tau}^k\|^2+U_\omega\bbE_{t-\tau}\|\bfB_t- \bfB_{t-\tau}\|^2
\end{align}

Summing up the upper bounds of $I_{121}$(\eqref{eqn:121})$ ,I_{122} $(\eqref{eqn:122})$, I_{124}$ (\eqref{eqn:124}), $I_{125}$ (\eqref{eqn:125}) and $I_{126}$ (\eqref{eqn:126}), conditioning on $v_{t-\tau}$, we get,
\begin{align}
\label{eqn: sum of I121 to I126}
    &\bbE_{t-\tau}\qth{I_{121}+I_{122}+I_{124}+I_{125}+I_{126}}\nonumber\\
    \leq& 3\bbE_{t-\tau}\|\omega_t^k-\omega_{t-\tau}^k\|^2+ \pth{1+3U_\omega}\bbE_{t-\tau}\|\bfB_t- \bfB_{t-\tau}\|^2\nonumber\\
    &+\pth{10U_\omega^2+32U_\omega^4}\frac{\beta^2}{L^2} \bbE_{t-\tau} \|\tilde A_{t,L}^k-A_L^k\|^2+\pth{2+32U_\omega^2}\frac{\beta^2}{L^2}\bbE_{t-\tau}\|\tilde b_{t,L}^k - \bar b_L^k \|^2\nonumber\\
    \leq&  3\tau \sum_{j=t-\tau}^{t-1}\bbE_{t-\tau}\|\omega_{j+1}^k-\omega_{j}^k\|^2+ \pth{1+3U_\omega}\tau \sum_{j=t-\tau}^{t-1}\bbE_{t-\tau}\|\bfB_{j+1}- \bfB_{j}\|^2\nonumber\\
    &+\pth{10U_\omega^2+32U_\omega^4}\frac{\beta^2}{L^2} \bbE_{t-\tau} \|\tilde A_{t,L}^k-A_L^k\|^2+\pth{2+32U_\omega^2}\frac{\beta^2}{L^2}\bbE_{t-\tau}\|\tilde b_{t,L}^k - \bar b_L^k \|^2.
\end{align}

For $I_{123}$, by Assumption \ref{assp: uniform ergodicity} and Lemma \ref{lmm: MC observable mixing}, we have 
\begin{align*}
2\frac{\beta}{L} \bbE_{t-\tau}\iprod{x_{t-\tau}^k}{\bfP_{t-\tau}\big[\tilde b_{t,L}^k - \bar b_L^k \big]}
& \le 2\frac{\beta}{L} \iprod{x_{t-\tau}^k}{\bfP_{t-\tau}\bbE_{t-\tau}\big[\tilde b_{t,L}^k - \bar b_L^k \big]} \\
& \le 2\frac{\beta}{L} \|x_{t-\tau}^k\| \|\bfP_{t-\tau}\| \|\bbE_{t-\tau}\big[\tilde b_{t,L}^k - \bar b_L^k \big]\| \\
& \le 2\frac{\beta}{L} 2U_{\omega} 4L U_r m \rho^{\tau} 
= 16U_\omega U_r \beta m\rho^{\tau}.
\end{align*}

For $I_{127}$, similarly, by Assumption \ref{assp: uniform ergodicity} and Lemma \ref{lmm: MC observable mixing}, we have 
\begin{align*}
    &2\frac{\beta}{L}\bbE_{t-\tau}\iprod{x_{t-\tau}^k}{\bfP_{t-\tau} (\tilde A_{t,L}^k-A_L^k)\bfB_{t-\tau}\omega_{t-\tau}^k}\leq 16 U_\omega^2\frac{\beta}{L} m\rho^{\tau}.
\end{align*} 

Then, by the tower property, we have 
\begin{align}
\label{eqn: I123+I127}
    \bbE_{t-\tau}[I_{123}+I_{127}]
    \leq 16U_\omega( U_r+ \frac{U_\omega}{L})\beta m\rho^{\tau} 
    \le 16U_\omega( U_r+ U_\omega)\beta m\rho^{\tau}. 
\end{align}

Putting \eqref{eqn: I123+I127} and \eqref{eqn: sum of I121 to I126} together, we get
\begin{align}
\label{eqn: second term in first ip final}
    &\bbE_{t-\tau}[I_{12}] \leq 3\tau \sum_{j=t-\tau}^{t-1}\bbE_{t-\tau}\|\omega_{j+1}^k-\omega_{j}^k\|^2+ \pth{1+3U_\omega}\tau \sum_{j=t-\tau}^{t-1}\bbE_{t-\tau}\|\bfB_{j+1}- \bfB_{j}\|^2\nonumber\\
    &+\pth{10U_\omega^2+32U_\omega^4}\frac{\beta^2}{L^2} \bbE_{t-\tau} \|\tilde A_{t,L}^k-A_L^k\|^2+\pth{2+32U_\omega^2}\frac{\beta^2}{L^2}\bbE_{t-\tau}\|\tilde b_{t,L}^k - \bar b_L^k \|^2 + 16U_\omega( U_r+U_\omega)\beta m\rho^{\tau}.
\end{align}
Invoking Assumption \ref{ass: per-agent full exploration}, putting \eqref{eqn: first term in first ip final}, \eqref{eqn: second term in first ip final}, and \eqref{eqn: third term in first ip final},  back to \eqref{eq: roughly drift of upper bound}, and taking conditional expectation on $v_{0: t-\tau}$, we get 
\begin{align}
&\bbE_{t-\tau}[I_1] = 2\frac{\beta}{L}\bbE_{t-\tau}\langle x_{t}^k, \bfB_t\bfB_t^\top\phi(s_{t,0}^k)\delta_{t,L}^k\rangle \nonumber\\
\leq&  -2\lambda\beta\bbE_{t-\tau}\|x_t^k\|^2+12U_\omega\frac{\beta}{L}  \bbE_{t-\tau}\|m_t\|\|x_t^k\|+
2\beta \bbE_{t-\tau}\|x_t^k\||y_t^k|\nonumber \\
&  +3\tau \sum_{j=t-\tau}^{t-1}\bbE_{t-\tau}\|\omega_{j+1}^k-\omega_{j}^k\|^2+ \pth{1+3U_\omega}\tau \sum_{j=t-\tau}^{t-1}\bbE_{t-\tau}\|\bfB_{j+1}- \bfB_{j}\|^2\nonumber\\
    &+\pth{10U_\omega^2+32U_\omega^4}\frac{\beta^2}{L^2} \bbE_{t-\tau} \|\tilde A_{t,L}^k-A_L^k\|^2+\pth{2+32U_\omega^2}\frac{\beta^2}{L^2}\bbE_{t-\tau}\|\tilde b_{t,L}^k - \bar b_L^k \|^2 +  16U_\omega( U_r+U_\omega)\beta m\rho^{\tau}.
\end{align}
Apply Young's inequality, for arbitrary $c_3>0$, $c_4>0$, we have 
\begin{align*}
    &12U_\omega\frac{\beta}{L}  \bbE_{t-\tau}\|m_t\|\|x_t^k\|\leq c_3^2\frac{\beta}{L} \bbE_{t-\tau}\|x_t^k\|^2 + \frac{36U_\omega^2}{c_3^2}\frac{\beta}{L} \bbE_{t-\tau}\|m_t\|^2,\\
    &2\beta \bbE_{t-\tau}\|x_t^k\||y_t^k|\leq c_4^2\beta  \bbE_{t-\tau}\|x_t^k\|^2 + \frac{1}{c_4^2}\beta\bbE_{t-\tau}|y_t^k|^2,
\end{align*} 
we get
\begin{align}
\label{eq: bound of 1st inner product: upper bound: local head}
    &\bbE_{t-2\tau}[I_1]\leq-2\lambda\beta\bbE_{t-2\tau}\|x_t^k\|^2+c_3^2\frac{\beta}{L} \bbE_{t-2\tau}\|x_t^k\|^2 + \frac{36U_\omega^2}{c_3^2}\frac{\beta}{L} \bbE_{t-2\tau}\|m_t\|^2+
 c_4^2\beta  \bbE_{t-2\tau}\|x_t^k\|^2 + \frac{1}{c_4^2}\beta\bbE_{t-2\tau}|y_t^k|^2\nonumber \\
&  +3\tau \sum_{j=t-\tau}^{t-1}\bbE_{t-2\tau}\|\omega_{j+1}^k-\omega_{j}^k\|^2+ \pth{1+3U_\omega}\tau \sum_{j=t-\tau}^{t-1}\bbE_{t-2\tau}\|\bfB_{j+1}- \bfB_{j}\|^2\nonumber\\
    &+\pth{10U_\omega^2+32U_\omega^4}\frac{\beta^2}{L^2} \bbE_{t-2\tau} \|\tilde A_{t,L}^k-A_L^k\|^2+\pth{2+32U_\omega^2}\frac{\beta^2}{L^2}\bbE_{t-2\tau}\|\tilde b_{t,L}^k - \bar b_L^k \|^2 + 16U_\omega( U_r+U_\omega)\beta m\rho^{\tau}.
\end{align}

\underline{Bounding $I_2$}. By Proposition \ref{prop: key intermediate},  $I_2$ can be written as,
\begin{align}
\label{eq: local head: upper bound: x tilde: 2nd inner produc}
    &\frac{2\zeta}{KL}\langle x_t^k, \sum_{i=1}^K \bfP_{t,\perp}\phi(s_{t,0}^i)\delta_{t,L}^i(\omega_t^i)^\top\omega_t^k\rangle \nonumber\\
    = &\frac{2\zeta}{KL}\langle x_t^k, \sum_{i=1}^K  \bfP_{t,\perp}\tilde A_{t,L}^ix_t^i(\omega_t^i)^\top\omega_t^k\rangle+\frac{2\zeta}{KL}\langle x_t^k, \sum_{i=1}^K  \bfP_{t,\perp}\bfb_{t,L}^i(\omega_t^i)^\top\omega_t^k\rangle+\frac{2\zeta}{K}\langle x_t^k, \sum_{i=1}^K  \bfP_{t,\perp}y_t^i\phi(s_{t,0}^i)(\omega_t^i)^\top\omega_t^k\rangle \nonumber\\
    \leq &\frac{4 U_\omega^2\zeta}{KL}\sum_{i=1}^K \|x_t^k\| \|x_t^i\|+\underbrace{\frac{2\zeta}{KL}\langle x_t^k, \sum_{i=1}^K  \bfP_{t,\perp}\bfb_{t,L}^i(\omega_t^i)^\top\omega_t^k\rangle}_{I_{21}}+\frac{2U_\omega^2\zeta}{K}\sum_{i=1}^K \|x_t^k\| |y_t^i|.
\end{align}
For $I_{21}$, similar to the derivations of the upper bounds of terms from $I_{121}$ to $I_{126}$, we have 
\begin{align}
\label{eqn: I_{21}}
    &\frac{2\zeta}{KL}\sum_{i=1}^K \langle x_t^k,  \bfP_{t,\perp}\bfb_{t,L}^i(\omega_t^i)^\top\omega_t^k\rangle\nonumber\\
    =&\frac{2\zeta}{KL}\sum_{i=1}^K\langle x_t^k-x_{t-\tau}^k,   \bfP_{t,\perp}\bfb_{t,L}^i(\omega_t^i)^\top\omega_t^k\rangle
    +\frac{2\zeta}{KL} \sum_{i=1}^K\langle x_{t-\tau}^k,  (\bfP_{t,\perp}-\bfP_{t-\tau,\perp})\bfb_{t,L}^i(\omega_t^i)^\top\omega_t^k\rangle\nonumber\\
    &+\frac{2\zeta}{KL} \sum_{i=1}^K\langle x_{t-\tau}^k,  \bfP_{t-\tau,\perp}\bfb_{t,L}^i(\omega_t^i-\omega_{t-\tau}^i)^\top\omega_t^k\rangle 
    +\frac{2\zeta}{KL} \sum_{i=1}^K\langle x_{t-\tau}^k,  \bfP_{t-\tau,\perp}\bfb_{t,L}^i(\omega_{t-\tau}^i)^\top(\omega_t^k-\omega_{t-\tau}^k)\rangle\nonumber\\
    &+\frac{2\zeta}{KL} \sum_{i=1}^K\langle x_{t-\tau}^k,  \bfP_{t-\tau,\perp}\bfb_{t,L}^i(\omega_{t-\tau}^i)^\top\omega_{t-\tau}^k\rangle\nonumber\\
    \leq  
    &\frac{3}{2}\|\omega_t^k-\omega_{t-\tau}^k\|^2  +\frac{3}{2}\|\bfB_{t}-\bfB_{t-\tau}\|^2 + \pth{18U_\omega^4+32U_\omega^6}\frac{\zeta^2}{KL^2}\sum_{i=1}^K\|\bfb_{t,L}^i\|^2
    + \frac{1}{2}\frac{1}{K} \sum_{i=1}^K \|\omega_t^i-\omega_{t-\tau}^i\|^2  \nonumber\\
    &+\frac{2\zeta}{KL} \sum_{i=1}^K\langle x_{t-\tau}^k,  \bfP_{t-\tau,\perp}\bfb_{t,L}^i(\omega_{t-\tau}^i)^\top\omega_{t-\tau}^k\rangle.
\end{align}
For the last term in Eq.\,(\ref{eqn: I_{21}}), by Cauchy–Schwarz inequality and Lemma \ref{lmm: MC observable mixing}, we have 
\begin{align}
\label{eqn: I21 Markovian}
\frac{2\zeta}{KL} \sum_{i=1}^K\bbE_{t-\tau}\langle x_{t-\tau}^k,  \bfP_{t-\tau,\perp}\bfb_{t,L}^i(\omega_{t-\tau}^i)^\top\omega_{t-\tau}^k\rangle  
& \le \frac{2\zeta}{KL} \sum_{i=1}^K 2U_{\omega}^3 \|\bbE_{t-\tau}[\bfb_{t,L}^i]\|\nonumber\\ 
&\leq 8 \frac{U_{\delta,L}}{L} U_\omega^3 \zeta m\rho^{\tau}.  
\end{align}
Then, putting \eqref{eqn: I21 Markovian} back to \eqref{eqn: I_{21}} and then putting the result back to \eqref{eq: local head: upper bound: x tilde: 2nd inner produc} and using tower property, we have
\begin{align}
    &\bbE_{t-\tau}[I_2] \leq 4 U_\omega^2\frac{\zeta}{KL}\sum_{i=1}^K \bbE_{t-\tau}\|x_t^k\| \|x_t^i\|+2U_\omega^2\frac{\zeta}{K}\sum_{i=1}^K \bbE_{t-\tau}\|x_t^k\| |y_t^i|+ \frac{3}{2}\bbE_{t-\tau}\|\omega_t^k-\omega_{t-\tau}^k\|^2\nonumber\\
    &  +\frac{3}{2}\bbE_{t-\tau}\|\bfB_{t}-\bfB_{t-\tau}\|^2 + \pth{18U_\omega^4+32U_\omega^6}\frac{\zeta^2}{KL^2}\sum_{i=1}^K\bbE_{t-\tau}\|\bfb_{t,L}^i\|^2
    + \frac{1}{2}\frac{1}{K} \sum_{i=1}^K \bbE_{t-\tau}\|\omega_t^i-\omega_{t-\tau}^i\|^2  \nonumber\\
    &+ 8 \frac{U_{\delta,L}}{L} U_\omega^3 \zeta m\rho^{\tau}.
\end{align}
Applying Young's inequality, for arbitrary $c_5>0$ and $c_6>0$, we have 
\begin{align*}
    &4 U_\omega^2\frac{\zeta}{KL}\sum_{i=1}^K \bbE_{t-\tau}\|x_t^k\| \|x_t^i\|
    \leq c_5^2\frac{\zeta}{L}\bbE_{t-\tau}\|x_t^k\|^2 + \frac{4U_\omega^4}{c_5^2}\frac{\zeta}{KL}\sum_{i=1}^K  \bbE_{t-\tau}\|x_t^i\|^2,\\
    &2U_\omega^2\frac{\zeta}{K}\sum_{i=1}^K \bbE_{t-\tau}\|x_t^k\| |y_t^i| \leq  c_6^2\zeta\bbE_{t-\tau}\|x_t^k\|^2 + \frac{U_\omega^4}{c_6^2}\frac{\zeta}{K}\sum_{i=1}^K  \bbE_{t-\tau}|y_t^i|^2.
\end{align*}
We get 
\begin{align}
    \label{eqn: second term in x tilde upper bound}
    &\bbE_{t-\tau}[I_2] \leq c_5^2\frac{\zeta}{L}\bbE_{t-\tau}\|x_t^k\|^2 + \frac{4U_\omega^4}{c_5^2}\frac{\zeta}{KL}\sum_{i=1}^K  \bbE_{t-\tau}\|x_t^i\|^2
    +c_6^2\zeta\bbE_{t-\tau}\|x_t^k\|^2 + \frac{U_\omega^4}{c_6^2}\frac{\zeta}{K}\sum_{i=1}^K  \bbE_{t-\tau}|y_t^i|^2
    \nonumber\\
    &+ \frac{3}{2}\bbE_{t-\tau}\|\omega_t^k-\omega_{t-\tau}^k\|^2  +\frac{3}{2}\bbE_{t-\tau}\|\bfB_{t}-\bfB_{t-\tau}\|^2 + \pth{18U_\omega^4+32U_\omega^6}\frac{\zeta^2}{KL^2}\sum_{i=1}^K\bbE_{t-\tau}\|\bfb_{t,L}^i\|^2
    \nonumber\\
    &+ \frac{1}{2}\frac{1}{K} \sum_{i=1}^K \bbE_{t-\tau}\|\omega_t^i-\omega_{t-\tau}^i\|^2  
    + 8 \frac{U_{\delta,L}}{L} U_\omega^3 \zeta m\rho^{\tau}.
\end{align}
\underline{For $I_3$}, by Proposition \ref{prop: key intermediate}, we have 
\begin{align*}
    I_3  %
    =&\underbrace{\iprod{x_t^k}{\pth{\frac{1}{K}\sum_{i=1}^K \bfP_{t,\perp}\delta_{t,L}^i\phi(s_{t,0}^i)(\omega_t^i)^\top} \pth{\bfB_t^{\top}\tilde A_{t,L}^k x_t^k}}}_{I_{31}}
    +\underbrace{\iprod{x_t^k}{\pth{\frac{1}{K}\sum_{i=1}^K \bfP_{t,\perp}\delta_{t,L}^i\phi(s_{t,0}^i)(\omega_t^i)^\top} \pth{\bfB_t^{\top}Ly_t^k\phi(s_{t,0}^k)}}}_{I_{32}}\\
    &+\underbrace{\iprod{x_t^k}{\pth{\frac{1}{K}\sum_{i=1}^K \bfP_{t,\perp}\delta_{t,L}^i\phi(s_{t,0}^i)(\omega_t^i)^\top} \pth{\bfB_t^{\top}\bfb_{t,L}^k}}}_{I_{33}}.
\end{align*}
By Lemma \ref{lmm: U delta}, $I_{31}$ and $I_{32}$ can be bounded as 
\begin{align}
I_{31} %
    &\leq 2U_{\delta,L} U_\omega\|x_t^k\|^2, \label{eqn: I31} \\
I_{32} & \leq L U_{\delta,L} U_\omega\|x_t^k\||y_t^k| \label{eqn: I32}.
\end{align}
For $I_{33}$, we have,
\begin{align}
\label{eqn: I_33 before expectation}
    I_{33}=&\iprod{x_t^k}{\pth{\frac{1}{K}\sum_{i=1}^K \bfP_{t,\perp}\tilde A_{t,L}^i x_t^i(\omega_t^i)^\top} \pth{\bfB_t^{\top}\bfb_{t,L}^k}} +\iprod{x_t^k}{\pth{\frac{1}{K}\sum_{i=1}^K \bfP_{t,\perp}L y_t^i \phi(s_{t,0}^i)(\omega_t^i)^\top} \pth{\bfB_t^{\top}\bfb_{t,L}^k }}\nonumber\\
    &+\iprod{x_t^k}{\pth{\frac{1}{K}\sum_{i=1}^K \bfP_{t,\perp} \bfb_{t,L}^i (\omega_t^i)^\top} \pth{\bfB_t^{\top}\bfb_{t,L}^k }}\nonumber\\
    \leq& 2U_{\delta,L} U_\omega \frac{1}{K}\sum_{i=1}^K\|x_t^k\|\|x_t^i\|+ LU_{\delta,L}U_\omega \frac{1}{K}\sum_{i=1}^K\|x_t^k\||y_t^i| +\iprod{x_t^k}{\pth{\frac{1}{K}\sum_{i=1}^K \bfP_{t,\perp} \bfb_{t,L}^i(\omega_t^i)^\top} \pth{\bfB_t^{\top}\bfb_{t,L}^k}}.
\end{align}
For the last term, we have,
\begin{align*}
    &\bbE_{t-\tau}\iprod{x_t^k}{\pth{\frac{1}{K}\sum_{i=1}^K \bfP_{t,\perp} \bfb_{t,L}^i(\omega_t^i)^\top} \pth{\bfB_t^{\top}\bfb_{t,L}^k}}\\
    =&\bbE_{t-\tau}(x_t^k-x_{t-\tau}^k)^\top\bfP_{t,\perp}\frac{1}{K}\sum_{i=1}^K  \bfb_{t,L}^i(\omega_t^i)^\top \pth{\bfB_t^{\top}\bfb_{t,L}^k}\\
    &+\bbE_{t-\tau}(x_{t-\tau}^k)^\top(\bfP_{t,\perp}-\bfP_{t-\tau,\perp})\frac{1}{K}\sum_{i=1}^K  \bfb_{t,L}^i(\omega_t^i)^\top \pth{\bfB_t^{\top}\bfb_{t,L}^k}\\
    &+\bbE_{t-\tau}(x_{t-\tau}^k)^\top\bfP_{t-\tau,\perp}\frac{1}{K}\sum_{i=1}^K  \bfb_{t,L}^i (\omega_t^i-\omega_{t-\tau}^i)^\top \pth{\bfB_t^{\top}\bfb_{t,L}^k }\\
    &+\bbE_{t-\tau}(x_{t-\tau}^k)^\top\bfP_{t-\tau,\perp}\frac{1}{K}\sum_{i=1}^K  \bfb_{t,L}^i (\omega_{t-\tau}^i)^\top \pth{(\bfB_t-\bfB_{t-\tau})^{\top}\bfb_{t,L}^k }\\
    &+\bbE_{t-\tau}(x_{t-\tau}^k)^\top\bfP_{t-\tau,\perp}\frac{1}{K}\sum_{i=1}^K  \bfb_{t,L}^i (\omega_{t-\tau}^i)^\top \pth{\bfB_{t-\tau}^{\top}\bfb_{t,L}^k }  ~~~~~ (\text{bridging})\\
    \leq& \bbE_{t-\tau}\|\omega_t^k-\omega_{t-\tau}^k\|U_{\delta,L}^2 U_\omega+7U_{\delta,L}^2 U_\omega^2\bbE_{t-\tau}\|\bfB_t-\bfB_{t-\tau}\|
    + 2U_\omega U_{\delta,L}^2 \frac{1}{K}\sum_{i=1}^K\bbE_{t-\tau}\|\omega_t^i-\omega_{t-\tau}^i\|\\
    &+2U_\omega^2U_{\delta,L}^2\frac{1}{K} +\bbE_{t-\tau}(x_{t-\tau}^k)^\top\bfP_{t-\tau,\perp}\frac{1}{K}\sum_{i\ne k}  \bfb_{t,L}^i (\omega_{t-\tau}^i)^\top \pth{\bfB_{t-\tau}^{\top}\bfb_{t,L}^k } \\
   \leq &  3\tau U_\omega \frac{U_{\delta,L}^3}{L} \beta+ 14\frac{U_{\delta,L}^3}{L} U_\omega^3 \tau  \zeta +28\frac{U_{\delta,L}^4}{L^2} U_\omega^4 \tau  \zeta^2+2U_\omega^2U_{\delta,L}^2\frac{1}{K}   +8U_\omega^2 U_{\delta,L}^2 m^2 \rho^{2\tau}, 
\end{align*}
where the last inequality follows from (a) the independence across agents, (b) Assumption \ref{assp: uniform ergodicity}, and (c) Lemma \ref{lmm: crude parameter bound}. 

Putting this bound back to $I_{33}$ (\eqref{eqn: I_33 before expectation}), and taking conditional expectation, we get,
\begin{align}
\label{eqn: I33}
    &\bbE_{t-\tau}[I_{33}]\leq 2U_{\delta,L} U_\omega \frac{1}{K}\sum_{i=1}^K\bbE_{t-\tau}\|x_t^k\|\|x_t^i\|+ LU_{\delta,L}U_\omega \frac{1}{K}\sum_{i=1}^K\bbE_{t-\tau}\|x_t^k\||y_t^i|\nonumber\\
    &+ 3\tau U_\omega \frac{U_{\delta,L}^3}{L} \beta+ 14\frac{U_{\delta,L}^3}{L} U_\omega^3 \tau  \zeta +28\frac{U_{\delta,L}^4}{L^2} U_\omega^4 \tau  \zeta^2+2U_\omega^2U_{\delta,L}^2\frac{1}{K}   +8U_\omega^2 U_{\delta,L}^2 m^2 \rho^{2\tau}.
\end{align}
Putting $I_{31}$ (\eqref{eqn: I31}), $I_{32}$ (\eqref{eqn: I32}), and $ I_{33}$ (\eqref{eqn: I33}) back to $I_3$, and taking expectation, we have 
\begin{align}
\label{eqn: third term in x tilde upper bound}
    &\bbE_{t-\tau}[I_3]\leq 2\frac{\zeta\beta}{L^2} \bigg(2U_{\delta,L} U_\omega\bbE_{t-\tau}\|x_t^k\|^2+L U_{\delta,L} U_\omega\bbE_{t-\tau}\|x_t^k\||y_t^k|\nonumber\\
    &\qquad + 2U_{\delta,L} U_\omega \frac{1}{K}\sum_{i=1}^K\bbE_{t-\tau}\|x_t^k\|\|x_t^i\|+ LU_{\delta,L}U_\omega \frac{1}{K}\sum_{i=1}^K\bbE_{t-\tau}\|x_t^k\||y_t^i|\nonumber\\
    &\qquad + 3\tau U_\omega \frac{U_{\delta,L}^3}{L} \beta+ 14\frac{U_{\delta,L}^3}{L} U_\omega^3 \tau  \zeta +28\frac{U_{\delta,L}^4}{L^2} U_\omega^4 \tau  \zeta^2+2U_\omega^2U_{\delta,L}^2\frac{1}{K}   + 8U_\omega^2 U_{\delta,L}^2 m^2 \rho^{2\tau}\bigg)\nonumber\\
    \leq& 4\frac{U_{\delta,L}}{L} U_\omega\frac{\zeta\beta}{L}\bbE_{t-\tau}\|x_t^k\|^2
    +2\frac{U_{\delta,L}}{L}  U_\omega\zeta\beta\bbE_{t-\tau}\|x_t^k\||y_t^k|\nonumber\\
    &\qquad +4 \frac{U_{\delta,L}}{L} U_\omega \frac{\zeta\beta}{L}\frac{1}{K}\sum_{i=1}^K\bbE_{t-\tau}\|x_t^k\|\|x_t^i\|
    +2\frac{U_{\delta,L}}{L} U_\omega \zeta\beta\frac{1}{K}\sum_{i=1}^K\bbE_{t-\tau}\|x_t^k\||y_t^i|\nonumber\\
    &\qquad +6 \tau U_\omega \frac{U_{\delta,L}^3}{L^3} \zeta\beta^2
    +28\frac{U_{\delta,L}^3}{L^3} U_\omega^3 \tau  \zeta^2\beta
    +56\frac{U_{\delta,L}^4}{L^4} U_\omega^4 \tau  \zeta^3\beta
    +4\frac{U_{\delta,L}^2}{L^2}U_\omega^2\frac{\zeta\beta}{K}   
    +16\frac{U_{\delta,L}^2}{L^2}U_\omega^2  m^2\zeta\beta \rho^{2\tau}\nonumber\\
    \leq& \underbrace{(6\frac{U_{\delta,L}}{L^2}U_\omega +2\frac{U_{\delta,L}}{L}U_\omega)}_{:=C_{31}}\zeta\beta\bbE_{t-\tau}\|x_t^k\|^2
    +\underbrace{\frac{U_{\delta,L}}{L} U_\omega}_{:=C_{32}} \zeta\beta\bbE_{t-\tau}|y_t^k|^2\nonumber\\
    &\qquad
    +\underbrace{2 \frac{U_{\delta,L}}{L} U_\omega}_{:=C_{33}}\frac{\zeta\beta}{L}\frac{1}{K}\sum_{i=1}^K\bbE_{t-\tau}\|x_t^i\|^2
    + \underbrace{\frac{U_{\delta,L}}{L} U_\omega}_{:=C_{34}}\zeta\beta\frac{1}{K}\sum_{i=1}^K\bbE_{t-\tau}|y_t^i|^2\nonumber\\
    &+\underbrace{6 \tau U_\omega \frac{U_{\delta,L}^3}{L^3} }_{:=C_{35}(\tau)}\zeta\beta^2
    +\underbrace{28\frac{U_{\delta,L}^3}{L^3} U_\omega^3 \tau }_{C_{36}(\tau)} \zeta^2\beta
    +\underbrace{56\frac{U_{\delta,L}^4}{L^4} U_\omega^4 \tau}_{C_{37}(\tau)}  \zeta^3\beta
    +\underbrace{4\frac{U_{\delta,L}^2}{L^2}U_\omega^2}_{C_{38}}\frac{\zeta\beta}{K}   
    +4\frac{U_{\delta,L}^2}{L^2}U_\omega^2  m^2\zeta\beta \rho^{2\tau} ~~~~(\text{Young's inequality})\nonumber\\
    =& C_{31}\zeta\beta\bbE_{t-\tau}\|x_t^k\|^2
    +C_{32}\zeta\beta\bbE_{t-\tau}|y_t^k|^2
    +C_{33}\frac{\zeta\beta}{L}\frac{1}{K}\sum_{i=1}^K\bbE_{t-\tau}\|x_t^i\|^2
    + C_{34}\zeta\beta\frac{1}{K}\sum_{i=1}^K\bbE_{t-\tau}|y_t^i|^2\nonumber\\
    &+C_{35}(\tau) \zeta\beta^2+C_{36}(\tau)\zeta^2\beta+C_{37}(\tau)\zeta^3\beta+C_{38}\frac{\zeta\beta}{K}   +4U_\omega^2 \frac{U_{\delta,L}^2}{L^2}  m^2\zeta\beta \rho^{2\tau}.
\end{align}

\underline{For $I_4$}, we have, 
\begin{align}
\label{eqn: I_4 before exp}
\|\bfB_t(\delta_{t,L}^k\bfB_t^\top\phi(s_{t,0}^k))\|^2  
\leq&3\|\bfB_t\bfB_t^\top(\tilde A_{t,L}^kx_t^k)\|^2 + 3L^2\|\bfB_t\bfB_t^\top y_t^k\phi(s_{t,0}^k)\|^2+3\|\bfB_t\bfB_t^\top\bfb_{t,L}^k\|^2 ~~~~\text{(by Proposition \ref{prop: key intermediate})}\nonumber\\
\leq&12\|x_t^k\|^2+ 3L^2|y_t^k|^2+3\|\bfb_{t,L}^k \|^2.
\end{align}
From Lemma \ref{lmm: local noise reduction}, we know 
\begin{align}
\label{eqn: tower property for bfb}
    \bbE_{t-\tau}\|\bfb_{t,L}^k\|^2
    \leq U_\delta^2(1+2\frac{m\rho}{1-\rho})L.
\end{align}
Then, taking conditional expectation on \eqref{eqn: I_4 before exp}, we get
\begin{align}
\label{I4 final}
    \bbE_{t-\tau}\|\bfB_t(\delta_{t,L}^k\bfB_t^\top\phi(s_{t,0}^k))\|^2\leq 12\bbE_{t-\tau}\|x_t^k\|^2+ 3L^2\bbE_{t-\tau}|y_t^k|^2+3U_\delta^2(1+2\frac{m\rho}{1-\rho})L.
\end{align}

Therefore, we have
\begin{align}
\label{eqn: fourth term in x tilde upper bound}
    &\bbE_{t-\tau}[I_{4}] \leq 36\frac{\beta^2}{L^2} \bbE_{t-\tau}\|x_t^k\|^2+9\beta^2 \bbE_{t-\tau}|y_t^k|^2+9U_\delta^2(1+2\frac{m\rho}{1-\rho})\frac{\beta^2}{L}.
\end{align}

\underline{For $I_5$}, similarly, under conditional expectation,  we have 
\begin{align*}
&\bbE_{t-\tau}\|\frac{1}{K}\sum_{i=1}^K  \bfP_{t,\perp}\delta_{t,L}^i\phi(s_{t,0}^i)(\omega_t^i)^\top\omega_t^k\|^2\\
\leq&3\bbE_{t-\tau}\|\frac{1}{K}\sum_{i=1}^K \bfP_{t,\perp}\tilde A_{t,L}^ix_t^i(\omega_t^i)^\top\omega_t^k\|^2 + 3L^2\bbE_{t-\tau}\|\frac{1}{K}\sum_{i=1}^K  \bfP_{t,\perp}y_t^i\phi(s_{t,0}^i)(\omega_t^i)^\top\omega_t^k\|^2\\
&+ 3\bbE_{t-\tau}\|\frac{1}{K}\sum_{i=1}^K  \bfP_{t,\perp}\bfb_{t,L}^i(\omega_t^i)^\top\omega_t^k\|^2  \\ %
\leq& 12 U_\omega^4 \frac{1}{K}\sum_{i=1}^K\bbE_{t-\tau}\|x_t^i\|^2+3L^2 U_\omega^4\frac{1}{K}\sum_{i=1}^K\bbE_{t-\tau}|y_t^i|^2+3U_\omega^4\frac{1}{K}\sum_{i=1}^K \bbE_{t-\tau}\| \bfb_{t,L}^i\|^2 \\
\leq&12 U_\omega^4 \frac{1}{K}\sum_{i=1}^K\bbE_{t-\tau}\|x_t^i\|^2+3L^2 U_\omega^4\frac{1}{K}\sum_{i=1}^K\bbE_{t-\tau}|y_t^i|^2+ 3U_\omega^4 U_\delta^2(1+2\frac{m\rho}{1-\rho})L. ~~~~~(\text{by \eqref{eqn: tower property for bfb}})
\end{align*}
Then, we have
\begin{align}
\label{eqn: fifth term in x tilde upper bound}
    &\bbE_{t-\tau}[I_5] \leq 36 U_\omega^4 \frac{\zeta^2}{L^2}\frac{1}{K}\sum_{i=1}^K\bbE_{t-\tau}\|x_t^i\|^2
    +9\zeta^2 U_\omega^4\frac{1}{K}\sum_{i=1}^K\bbE_{t-\tau}|y_t^i|^2
    + 9U_\omega^4 U_\delta^2(1+2\frac{m\rho}{1-\rho})\frac{\zeta^2}{L}.
\end{align}

\underline{For $I_6$}, we have,
\begin{align}
\label{eqn: sixth term in x tilde upper bound}
& I_6\leq 3\frac{\beta^2\zeta^2}{L^4} \|(\frac{1}{K}\sum_{i=1}^K  \bfP_{t,\perp}\delta_{t,L}^i\phi(s_{t,0}^i)(\omega_t^i)^\top) \pth{\delta_{t,L}^k\bfB_t^{\top}\phi(s_{t,0}^k)}\|^2 \le 3 \frac{U_{\delta,L}^4}{L^4} U_{\omega}^2 \beta^2\zeta^2. 
\end{align}
Plugging the upper bounds of $I_1$, $I_2$, $I_3$, $I_4$, $I_5$, and $I_6$(\eqref{eq: bound of 1st inner product: upper bound: local head}, \eqref{eqn: second term in x tilde upper bound}, \eqref{eqn: third term in x tilde upper bound}, \eqref{eqn: fourth term in x tilde upper bound}, \eqref{eqn: fifth term in x tilde upper bound}, and \eqref{eqn: sixth term in x tilde upper bound}) back to \eqref{eqn:xtilde squared decomp}, and taking expectation, we have  

\begin{align*}
    &\bbE_{t-\tau}\|\tilde x_{t+1}^k\|^2
    \leq
    (1-2\lambda\beta)\bbE_{t-\tau}\|x_t^k\|^2
    +c_3^2\frac{\beta}{L} \bbE_{t-\tau}\|x_t^k\|^2 
    + \frac{36U_\omega^2}{c_3^2}\frac{\beta}{L} \bbE_{t-\tau}\|m_t\|^2
    + c_4^2\beta  \bbE_{t-\tau}\|x_t^k\|^2
 + \frac{1}{c_4^2}\beta\bbE_{t-\tau}|y_t^k|^2\nonumber \\
&  +3\tau \sum_{j=t-\tau}^{t-1}\bbE_{t-\tau}\|\omega_{j+1}^k-\omega_{j}^k\|^2
+ \pth{1+3U_\omega}\tau \sum_{j=t-\tau}^{t-1}\bbE_{t-\tau}\|\bfB_{j+1}- \bfB_{j}\|^2\nonumber\\
    &+\pth{10U_\omega^2+32U_\omega^4}\frac{\beta^2}{L^2} \bbE_{t-\tau} \|\tilde A_{t,L}^k-A_L^k\|^2
    +\pth{2+32U_\omega^2}\frac{\beta^2}{L^2}\bbE_{t-\tau}\|\tilde b_{t,L}^k - \bar b_L^k \|^2
    + 16U_\omega( U_r+U_\omega) m\beta\rho^{\tau}\\
    &+ c_5^2\frac{\zeta}{L}\bbE_{t-\tau}\|x_t^k\|^2
    + \frac{4U_\omega^4}{c_5^2}\frac{\zeta}{KL}\sum_{i=1}^K\bbE_{t-\tau}\|x_t^i\|^2
    +c_6^2\zeta\bbE_{t-\tau}\|x_t^k\|^2
    + \frac{U_\omega^4}{c_6^2}\zeta\frac{1}{K}\sum_{i=1}^K\bbE_{t-\tau}|y_t^i|^2
    \nonumber\\
    &+ \frac{3}{2}\bbE_{t-\tau}\|\omega_t^k-\omega_{t-\tau}^k\|^2  +\frac{3}{2}\bbE_{t-\tau}\|\bfB_{t}-\bfB_{t-\tau}\|^2
    + \pth{18U_\omega^4+32U_\omega^6}U_\delta^2(1+2\frac{m\rho}{1-\rho})\frac{\zeta^2}{L}
    \nonumber\\
    &+ \frac{1}{2K} \sum_{i=1}^K \bbE_{t-\tau}\|\omega_t^i-\omega_{t-\tau}^i\|^2  
    + 4 \frac{U_{\delta,L}}{L} U_\omega^3 \zeta m\rho^{\tau}
    +C_{31}\zeta\beta\bbE_{t-\tau}\|x_t^k\|^2
    +C_{32}\zeta\beta\bbE_{t-\tau}|y_t^k|^2
    +C_{33}\frac{\zeta\beta}{L}\frac{1}{K}\sum_{i=1}^K\bbE_{t-\tau}\|x_t^i\|^2\nonumber\\
    &
    + C_{34}\zeta\beta\frac{1}{K}\sum_{i=1}^K\bbE_{t-\tau}|y_t^i|^2+C_{35}(\tau) \zeta\beta^2+C_{36}(\tau)\zeta^2\beta+C_{37}(\tau)\zeta^3\beta+C_{38}\frac{\zeta\beta}{K}   + 16U_\omega^2 \frac{U_{\delta,L}^2}{L^2}  m^2\zeta\beta \rho^{2\tau}
    \\
    &+36\frac{\beta^2}{L^2} \bbE_{t-\tau}\|x_t^k\|^2+9\beta^2 \bbE_{t-\tau}|y_t^k|^2+9U_\delta^2(1+2\frac{m\rho}{1-\rho})\frac{\beta^2}{L}\\
    &+36 U_\omega^4 \frac{\zeta^2}{L^2}\frac{1}{K}\sum_{i=1}^K\bbE_{t-\tau}\|x_t^i\|^2
    +9\zeta^2 U_\omega^4\frac{1}{K}\sum_{i=1}^K\bbE_{t-\tau}|y_t^i|^2
    + 9U_\omega^4 U_\delta^2(1+2\frac{m\rho}{1-\rho})\frac{\zeta^2}{L}
    +3 \frac{U_{\delta,L}^4}{L^4} U_{\omega}^2 \beta^2\zeta^2.
\end{align*}

Combining terms, we get,
\begin{align*}
&\bbE_{t-\tau}\|\tilde x_{t+1}^k\|^2
    \leq
    \pth{1-2\lambda\beta+\pth{c_3^2\frac{\beta}{L} 
+ c_4^2\beta 
+ c_5^2\frac{\zeta}{L}
+c_6^2\zeta
+C_{31}\zeta\beta
+36\frac{\beta^2}{L^2} }}\bbE_{t-\tau}\|x_t^k\|^2\\
&+ \frac{36U_\omega^2}{c_3^2}\frac{\beta}{L} \bbE_{t-\tau}\|m_t\|^2
+ \pth{\frac{1}{c_4^2}\beta
+C_{32}\zeta\beta
+9\beta^2}\bbE_{t-\tau}|y_t^k|^2\\
&+\pth{\frac{4U_\omega^4}{c_5^2}\frac{\zeta}{L}
+C_{33}\frac{\zeta\beta}{L}
+36 U_\omega^4 \frac{\zeta^2}{L^2}}\frac{1}{K}\sum_{i=1}^K\bbE_{t-\tau}\|x_t^i\|^2\\
&+\pth{\frac{U_\omega^4}{c_6^2}\zeta
+ C_{34}\zeta\beta
+9\zeta^2 U_\omega^4}\frac{1}{K}\sum_{i=1}^K\bbE_{t-\tau}|y_t^i|^2\\
&+ 16U_\omega( U_r+U_\omega) m \beta \rho^{\tau}
+ \pth{18U_\omega^4+32U_\omega^6}U_\delta^2(1+2\frac{m\rho}{1-\rho})\frac{\zeta^2}{L}
+ 4 \frac{U_{\delta,L}}{L} U_\omega^3 \zeta m\rho^{\tau}\\
&+C_{35}(\tau) \zeta\beta^2
+C_{36}(\tau)\zeta^2\beta+C_{37}(\tau)\zeta^3\beta
+C_{38}\frac{\zeta\beta}{K}
+16U_\omega^2 \frac{U_{\delta,L}^2}{L^2}  m^2\zeta\beta \rho^{2\tau}
+9U_\delta^2(1+2\frac{m\rho}{1-\rho})\frac{\beta^2}{L}
\\&+ 9U_\omega^4 U_\delta^2(1+2\frac{m\rho}{1-\rho})\frac{\zeta^2}{L}
+3 \frac{U_{\delta,L}^4}{L^4} U_{\omega}^2 \beta^2\zeta^2\\
    &
    +\frac{9}{2}\tau \sum_{j=t-\tau}^{t-1}\bbE_{t-\tau}\|\omega_{j+1}^k-\omega_{j}^k\|^2
+ \pth{\frac{5}{2}+3U_\omega}\tau \sum_{j=t-\tau}^{t-1}\bbE_{t-\tau}\|\bfB_{j+1}- \bfB_{j}\|^2\\
&
+\pth{10U_\omega^2+32U_\omega^4}\frac{\beta^2}{L^2} \bbE_{t-\tau} \|\tilde A_{t,L}^k-A_L^k\|^2
    +\pth{2+32U_\omega^2}\frac{\beta^2}{L^2}\bbE_{t-\tau}\|\tilde b_{t,L}^k - \bar b_L^k \|^2\\
    & 
+\frac{1}{2K}\tau \sum_{j=t-\tau}^{t-1} \sum_{i=1}^K \bbE_{t-\tau}\|\omega_{j+1}^i-\omega_{j}^i\|^2.
\end{align*}
Applying Lemma \ref{lmm: local noise reduction}, Lemma \ref{lmm: one step temporal difference bound of omega}, Lemma \ref{lmm: one-step perturbation of B}, and using $\beta = c\zeta$ we further get,
\begin{align}
    \label{eqn: xtilde final}
    &\bbE_{t-\tau}\|\tilde x_{t+1}^k\|^2
    \leq
    \pth{1-2\lambda c\zeta+\pth{c_3^2\frac{c\zeta}{L} 
+ c_4^2 c\zeta 
+ c_5^2\frac{\zeta}{L}
+c_6^2\zeta
+C_{31}c\zeta^2
+(36+40\tau^2)\frac{c^2\zeta^2}{L^2} }}\bbE_{t-\tau}\|x_t^k\|^2\nonumber\\
&+ \frac{36U_\omega^2}{c_3^2}\frac{c\zeta}{L} \bbE_{t-\tau}\|m_t\|^2
+ \pth{\frac{1}{c_4^2}c\zeta
+C_{32}c\zeta^2
+(9+10\tau^2)c^2\zeta^2}\bbE_{t-\tau}|y_t^k|^2\nonumber\\
&+\pth{\frac{4U_\omega^4}{c_5^2}\frac{\zeta}{L}
+C_{33}\frac{c\zeta^2}{L}
+\pth{36 U_\omega^4+\pth{\frac{5}{2}+3U_\omega}C_{18}} \frac{\zeta^2}{L^2}}\frac{1}{K}\sum_{i=1}^K\bbE_{t-\tau}\|x_t^i\|^2\nonumber\\
&+\pth{\frac{U_\omega^4}{c_6^2}\zeta
+ C_{34}c\zeta^2
+\pth{9 U_\omega^4+\pth{\frac{5}{2}+3U_\omega}C_{23}}\zeta^2}\frac{1}{K}\sum_{i=1}^K\bbE_{t-\tau}|y_t^i|^2 \nonumber\\
&+C_{39}(\tau^2)\frac{\zeta^2}{L}
 + C_{41}(\tau^2)\frac{\zeta^2}{K} 
+ C_{42}\zeta\rho^{\tau}
+C_{43}(\tau)\zeta^3
+C_{44}(\tau^2)\zeta^2\rho^{2\tau}
     +C_{40}(\tau^4)\zeta^4
     + C_{45}(\tau^4) \gamma^2\zeta^2
    ,
\end{align}
where
\begin{align*}
   & C_{39}(\tau^2) = \bigg[\pth{10U_\omega^2+32U_\omega^4}\pth{16  + 32\frac{m\rho}{1-\rho} }c^2 
    +\pth{2+32U_\omega^2}U_\delta^2(1+2\frac{m\rho}{1-\rho})c^2 \\&\qquad + \pth{18U_\omega^4+32U_\omega^6}U_\delta^2(1+2\frac{m\rho}{1-\rho})
    +9U_\delta^2(1+2\frac{m\rho}{1-\rho})c^2
    + 9U_\omega^4 U_\delta^2(1+2\frac{m\rho}{1-\rho}) + 5\tau^2C_7c^2\bigg],\\
    &C_{40}(\tau^4) = \bigg[C_{37}(\tau)c
+3 \frac{U_{\delta,L}^4}{L^4} U_{\omega}^2c^2  
    +5\tau^2C_5(\tau^2)c^4\frac{1}{L^2}
    + 5\tau^2C_6(\tau^2)c^2\frac{1}{L^2} \\
    &\qquad +\pth{\frac{5}{2}+3U_\omega}\tau^2\bigg(C_{16}(\tau^2)c^2\frac{1}{L^2}
    + C_{17}(\tau^2)\frac{1}{L^2}
    + C_{19}(\tau^2)c^2 
    + C_{24}\bigg) \bigg]\\
    &C_{41}(\tau^2)=\qth{\pth{\frac{5}{2}+3U_\omega}\tau^2 C_{20}+C_{38}c}, C_{42} = \qth{16U_\omega( U_r+U_\omega) m c + 4 \frac{U_{\delta,L}}{L} U_\omega^3  m},C_{43}(\tau) = \qth{C_{35}(\tau) c^2 +C_{36}(\tau)c},\\
    &C_{44}(\tau^2)=\qth{16U_\omega^2 \frac{U_{\delta,L}^2}{L^2}  m^2c    +\pth{\frac{5}{2}+3U_\omega}\tau^2 C_{21}m^2}, C_{45}(\tau^4)=\qth{5\tau^2C_8(\tau^2)c^2
    + \pth{\frac{5}{2}+3U_\omega}\tau^2 C_{22}(\tau^2)}.
\end{align*}

Putting \eqref{eq: x_t+1 intermediate bound cross term} and \eqref{eq: upper bound of qr-proj squared}, back to \eqref{eq: upper bound: key intermediate}, and by tower property, we get,
\begin{align}
\label{eqn: x t+1 before sub m t+1 and q}
    &\bbE_{t-\tau}\|x_{t+1}^k\|^2 = \bbE_{t-\tau}\|\tilde x_{t+1}^k\|^2 + \bbE_{t-\tau}\|\bfB_{t+1}\Pi_{U_\omega}(\tilde\omega_{t+1}^k)-\bfB_{t+1}\bfR_{t+1}\tilde\omega_{t+1}^k\|^2 \nonumber\\& +2\bbE_{t-\tau}\langle\bfB_{t+1}\Pi_{U_\omega}(\tilde\omega_{t+1}^k)-\bfB_{t+1}\bfR_{t+1}\tilde\omega_{t+1}^k, \tilde x_{t+1}^k \rangle \nonumber\\
   \leq& \bbE_{t-\tau}\|\tilde x_{t+1}^k\|^2 \nonumber
   +96\frac{\beta^2}{L^2}\bbE_{t-\tau}\|x_t^k\|^2+24\frac{\beta^2}{L}U_\delta^2(1+2\frac{m\rho}{1-\rho})+24\beta^2\bbE_{t-\tau}|y_t^k|^2+100U_\omega^2 \bbE_{t-\tau}\|\bfQ_t\|_F^4\nonumber\\
     &+6U_\omega\bbE_{t-\tau}\norm{\omega_{t+1}^k - \tilde\omega_{t+1}^k}\|m_{t+1}\|^2+16U_\omega^2 \bbE_{t-\tau}\|\bfQ_t\|_F^2.
\end{align}

For $\norm{\omega_{t+1}^k - \tilde\omega_{t+1}^k}$, we have 
\begin{align*}
\norm{\omega_{t+1}^k - \tilde\omega_{t+1}^k}\leq \|\frac{\beta}{L} \delta_{t,L}^k\bfB_t^\top \phi(s_{t,0}^k)\|\leq \beta U_\delta , ~~~~(\text{ since } \frac{U_{\delta,L}}{L}\leq U_\delta)
\end{align*}
for $\bbE_{t-\tau}\|\bfQ_t\|_F^4$, we use the crude bound that $\|\bfQ_t\|_F\leq U_\delta U_\omega\zeta$, and for $\bbE_{t-\tau}\|\bfQ_t\|_F^2$, we invoke \eqref{eqn: Q bound}. Then, \eqref{eqn: x t+1 before sub m t+1 and q} can be written as 
\begin{align}
    \label{eqn: x t+1 before sub m t+1}
    &\bbE_{t-\tau}\|x_{t+1}^k\|^2  \nonumber\\
   \leq& \bbE_{t-\tau}\|\tilde x_{t+1}^k\|^2 \nonumber
   +96\frac{\beta^2}{L^2}\bbE_{t-\tau}\|x_t^k\|^2+24\frac{\beta^2}{L}U_\delta^2(1+2\frac{m\rho}{1-\rho})+24\beta^2\bbE_{t-\tau}|y_t^k|^2+100U_\omega^6 U_\delta^4 \zeta^4\nonumber\\
     &+6U_\omega U_\delta \beta\bbE_{t-\tau}\|m_{t+1}\|^2
     + \frac{256U_\omega^4\zeta^2}{L^2}\frac{1}{K}\sum_{i=1}^K \bbE_{t-\tau}\|x_t^i\|^2 
+ 64\frac{U_{\delta,L}^2}{L^2}U_\omega^4\frac{\zeta^2}{K} 
+ 128 U_\omega^4 \frac{U_{\delta,L}^2}{L^2} \zeta^2 m^2\rho^{2\tau}\nonumber\\
&+144U_\omega^2\tau^2\frac{U_{\delta,L}^4}{L^4}\zeta^2\beta^2
+\frac{64U_\omega^4\zeta^2}{K}\sum_{i=1}^K \bbE_{t-\tau}| y_t^i|^2 .
\end{align}
 Recall $M_t = \bbE_{t-\tau}\|m_{t}\|_F^2; \bar X_t = \frac{1}{K}\sum_{k=1}^K \bbE_{t-\tau}\|x_t^k\|^2 ;  \bar Y_t = \frac{1}{K}\sum_{k=1}^K \bbE_{t-\tau}| y_t^k|^2.$
Invoking \eqref{eqn: xtilde final}, and taking average over agents, \eqref{eqn: x t+1 before sub m t+1} can be written as,
\begin{align}
\label{eqn: Xt+1 before sub Mt+1}
     &\bar X_{t+1}  \nonumber\\
   \leq& \Bigg(1-2\lambda c\zeta
   +\Big(c_3^2\frac{c }{L} 
+ c_4^2 c  
+ c_5^2\frac{1}{L}
+c_6^2 +4U_\omega^4\frac{1}{c_5^2L}\Big)\zeta     \nonumber\\
&
+\bigg(C_{31}c 
+(36+40\tau^2)\frac{c^2}{L^2}
+C_{33}\frac{c}{L} 
+\pth{36 U_\omega^4+\pth{2.5+3U_\omega}C_{18}} \frac{1}{L^2}
+\frac{96c^2}{L^2}
+\frac{256U_\omega^4}{L^2}\bigg)\zeta^2  \Bigg)\bar X_t\nonumber\\
&+ \frac{36U_\omega^2}{c_3^2}\frac{c\zeta}{L} \bbE_{t-\tau}\|m_t\|^2
\nonumber\\
&+ \Bigg(
\Big[
\frac{1}{c_4^2}c 
+\frac{U_\omega^4}{c_6^2}
\Big]\zeta
+\Big[
C_{32}c
+(9+10\tau^2)c^2
+C_{34}c
+9 U_\omega^4+\pth{2.5+3U_\omega}C_{23}
+24c^2
+64U_\omega^4
\big]\zeta^2
\Bigg)\bar Y_t\nonumber\\
&+\qth{C_{39}(\tau^2)+24c^2U_\delta^2(1+2\frac{m\rho}{1-\rho})}\frac{\zeta^2}{L}
 +\qth{ C_{41}(\tau^2)
  + 64\frac{U_{\delta,L}^2}{L^2}U_\omega^4}\frac{\zeta^2}{K} 
\nonumber\\
&+ C_{42}\zeta\rho^{\tau}
+C_{43}(\tau)\zeta^3
     +\qth{C_{40}(\tau^4)+100U_\omega^6 U_\delta^4+144U_\omega^2\tau^2\frac{U_{\delta,L}^4}{L^4}c^2}\zeta^4
     + C_{45}(\tau^4) \gamma^2\zeta^2 \nonumber\\
     &+6U_\omega U_\delta \beta\bbE_{t-\tau}\|m_{t+1}\|^2
+ \qth{128 U_\omega^4 \frac{U_{\delta,L}^2}{L^2}  m^2+C_{44}(\tau^2)}\zeta^2\rho^{2\tau}.
\end{align}
By \eqref{eqn: Mt+1 final}, \eqref{eqn: Xt+1 before sub Mt+1} can be written as,
\begin{align}
\label{eqn: Xt+1 after sub Mt+1}
         \bar X_{t+1}  
   \leq&  \Bigg(1-2\lambda c\zeta
   +\Big(c_3^2\frac{c }{L} 
+ c_4^2 c  
+ c_5^2\frac{1}{L}
+c_6^2 +4U_\omega^4\frac{1}{c_5^2L}\Big)\zeta     \nonumber\\
&
+\bigg(C_{31}c 
+(36+40\tau^2)\frac{c^2}{L^2}
+C_{33}\frac{c}{L} 
+\pth{36 U_\omega^4+\pth{2.5+3U_\omega}C_{18}} \frac{1}{L^2}\nonumber\\
&\qquad +\frac{96c^2}{L^2}
+\frac{256U_\omega^4}{L^2}
+96 U_\delta c \frac{U_\omega^3}{c_2^2}\frac{1}{L^2}
\bigg)\zeta^2
+ 6U_\omega U_\delta c  \calC_{M,4}(\tau^2)\zeta^3
\Bigg)\bar X_t\nonumber\\
&+ \frac{36U_\omega^2}{c_3^2}\frac{c\zeta}{L} M_t+6U_\omega U_\delta c \zeta  M_t
\nonumber\\
&+ \Bigg(
\Big[
\frac{1}{c_4^2}c 
+\frac{U_\omega^4}{c_6^2}
\Big]\zeta
+\Bigg[
C_{32}c
+(9+10\tau^2)c^2
+C_{34}c
+9 U_\omega^4+\pth{2.5+3U_\omega}C_{23}\nonumber\\
&\qquad+24c^2
+64U_\omega^4
+24  U_\delta c  \frac{U_\omega^3}{c_1^2}
\Bigg]\zeta^2
+ 6U_\omega U_\delta c  \calC_{M,5}(\tau^2)\zeta^3
\Bigg)\bar Y_t\nonumber\\
&+\qth{C_{39}(\tau^2)+24c^2U_\delta^2(1+2\frac{m\rho}{1-\rho})}\frac{\zeta^2}{L}
 +\qth{ C_{41}(\tau^2)
  + 64\frac{U_{\delta,L}^2}{L^2}U_\omega^4}\frac{\zeta^2}{K} 
\nonumber\\
&+ C_{42}\zeta\rho^{\tau}
+\qth{C_{43}(\tau) +24U_\omega U_\delta c  C_{25}(\tau^2)\frac{1}{L} +
6U_\omega U_\delta c  \calC_{M,2}(\tau^2)\frac{1}{K}}\zeta^3\nonumber\\
 &    +\qth{C_{40}(\tau^4)+100U_\omega^6 U_\delta^4+144U_\omega^2\tau^2\frac{U_{\delta,L}^4}{L^4}c^2+
6U_\omega U_\delta c  \calC_{M,6}}\zeta^4
     + C_{45}(\tau^4) \gamma^2\zeta^2 \nonumber\\
& 
+
6U_\omega U_\delta c  \calC_{M,1}\zeta^5
+
24U_\omega U_\delta c  C_{28}(\tau^4)\gamma^2\zeta^3
+6U_\omega U_\delta c  \calC_{M,3} \zeta^7
 + \qth{128 U_\omega^4 \frac{U_{\delta,L}^2}{L^2}  m^2+C_{44}(\tau^2)}\zeta^2\rho^{2\tau}.
\end{align}
Choosing $\tau = \lceil 2\log_\rho(\zeta)\rceil$, we have
\begin{align}
\label{eqn: Xt+1 final}
            &\bar X_{t+1}  \nonumber
   \leq \Bigg(1-2\lambda c\zeta
   +\Big(c_3^2\frac{c }{L} 
+ c_4^2 c  
+ c_5^2\frac{1}{L}
+c_6^2 +4U_\omega^4\frac{1}{c_5^2L}\Big)\zeta   
+\calC_{X,1}(\tau^2)\zeta^2
+ \calC_{X,2}(\tau^2)\zeta^3
\Bigg)\bar X_t\nonumber\\
&\qquad+ \pth{\frac{36U_\omega^2}{c_3^2}\frac{c\zeta}{L}  +6U_\omega U_\delta c \zeta } M_t
\nonumber\\
&\qquad+ \Bigg(
\Big[
\frac{1}{c_4^2}c 
+\frac{U_\omega^4}{c_6^2}
\Big]\zeta
+\calC_{X,3}(\tau^2)\zeta^2
+\calC_{X,4}(\tau^2) \zeta^3
\Bigg)\bar Y_t\nonumber\\
&\qquad+\calC_{X,5}(\tau^2)\frac{\zeta^2}{L}
 +\calC_{X,6}(\tau^2)\frac{\zeta^2}{K} 
+\calC_{X,7}\zeta^3
 +\calC_{X,8}\zeta^4
 + \calC_{X,9} \gamma^2\zeta^2 
+
\calC_{X,10}\zeta^5
+
\calC_{X,11}\gamma^2\zeta^3
+\calC_{X,12}\zeta^7
+ \calC_{X,13}\zeta^6,
\end{align}
where
\begin{align*}
    &\calC_{X,1}(\tau^2) =  \bigg(C_{31}c 
+(36+40\tau^2)\frac{c^2}{L^2}
+C_{33}\frac{c}{L} 
+\pth{36 U_\omega^4+\pth{2.5+3U_\omega}C_{18}} \frac{1}{L^2}  +\frac{96c^2}{L^2}
+\frac{256U_\omega^4}{L^2}
+96 U_\delta c \frac{U_\omega^3}{c_2^2}\frac{1}{L^2}
\bigg),\\
&\calC_{X,2}(\tau^2) = 6U_\omega U_\delta c  \calC_{M,4}(\tau^2),\\
&\calC_{X,3}(\tau^2)=\Bigg[
C_{32}c
+(9+10\tau^2)c^2
+C_{34}c
+9 U_\omega^4+\pth{2.5+3U_\omega}C_{23} +24c^2
+64U_\omega^4
+24  U_\delta c  \frac{U_\omega^3}{c_1^2}
\Bigg], 
\end{align*}
\begin{align*}
&\calC_{X,4}(\tau^2)=6U_\omega U_\delta c  \calC_{M,5}(\tau^2), 
\calC_{X,5}(\tau^2)=\qth{C_{39}(\tau^2)+24c^2U_\delta^2(1+2\frac{m\rho}{1-\rho})},
\calC_{X,6}(\tau^2)=\qth{ C_{41}(\tau^2)
  + 64\frac{U_{\delta,L}^2}{L^2}U_\omega^4},\\
&\calC_{X,7} = \qth{C_{42}+C_{43}(\tau) +24U_\omega U_\delta c  C_{25}(\tau^2)\frac{1}{L} +
6U_\omega U_\delta c  \calC_{M,2}(\tau^2)\frac{1}{K}},\\
&\calC_{X,8}=\qth{C_{40}(\tau^4)+100U_\omega^6 U_\delta^4+144U_\omega^2\tau^2\frac{U_{\delta,L}^4}{L^4}c^2+
6U_\omega U_\delta c  \calC_{M,6}}, 
\calC_{X,9} = C_{45}(\tau^4),
\calC_{X,10} = 6U_\omega U_\delta c  \calC_{M,1},\\
&\calC_{X,11} = 24U_\omega U_\delta c  C_{28}(\tau^4),\calC_{X,12} = 6U_\omega U_\delta c  \calC_{M,3},\calC_{X,13}=\qth{128 U_\omega^4 \frac{U_{\delta,L}^2}{L^2}  m^2+C_{44}(\tau^2)}. 
\end{align*}

\end{proof}

\newpage 
\subsection{Lower bound: Proof of Lemma \ref{lm: local head: lower bound}}
\label{subsec: proof: local head: lower bound}
Recall that $x_t^k = \bfB_t\omega_t^k-z^{k,*} = \bfB_t\omega_t^k-\bfB^*\omega^{k,*}$. 
By definition, for each $k\in [K]$, $\bfB_t\omega_t^k$ lies in the subspace spanned by $\bfB_t$. As illustrated in Fig.\ref{fig: lower bound illustration}, the distance between $\bfB_t\omega_t^k$ and $z^{k,*}$ is no smaller than the distance 
\begin{wrapfigure}[4]{r}{0.44\textwidth}  
\centering 
\vspace{-1.5em}
\includegraphics[width = 0.8\linewidth]{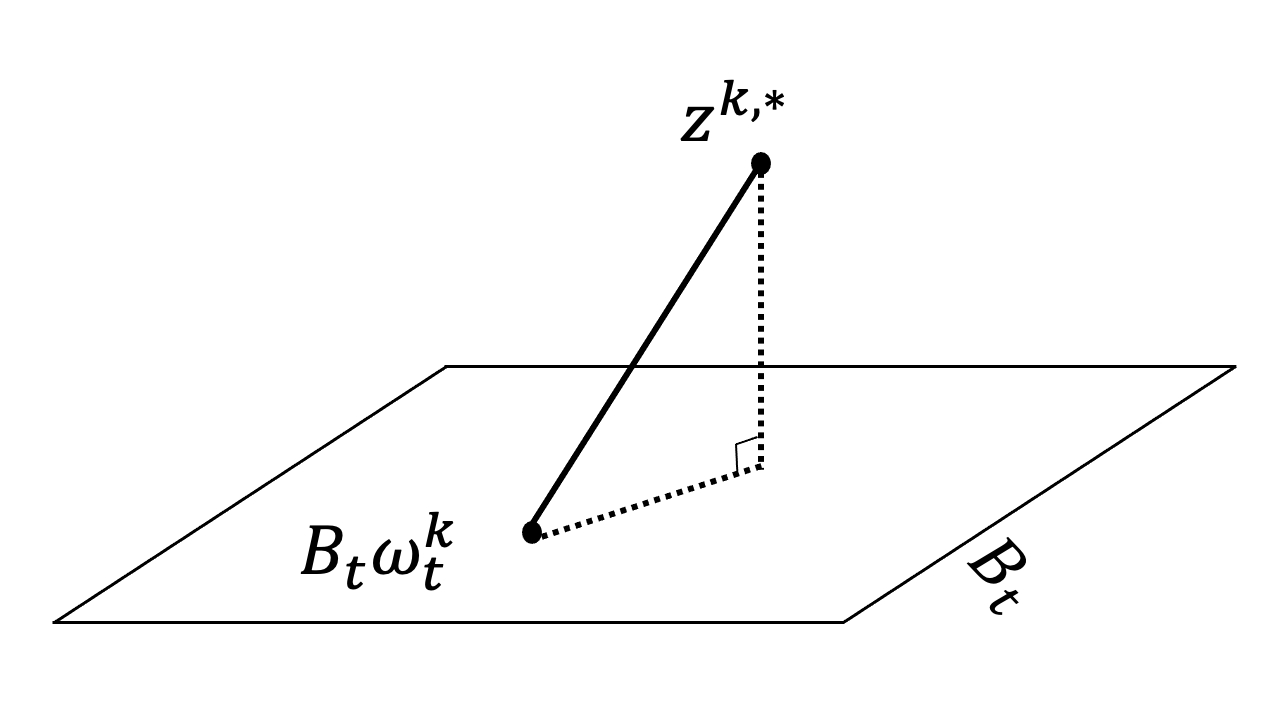}
\vskip -0.5\baselineskip 
\caption{\footnotesize Geometry illustration.} 
\label{fig: lower bound illustration} 
\end{wrapfigure} 
between $z^{k,*}$ and the subspace spanned by $\bfB_t$. 
Hence, %

\vspace{-0.6em}
\begin{minipage}{0.6\textwidth}
\begin{align*}
\|x_t^k\| &= \|\bfB_t\omega_t^k-z^{k,*}\|  \geq \|\bfP_{t,\perp}z^{k,*}\| \\
&= \|(\bfI-\bfP_{t})z^{k,*}\| = \|(\bfP^*-\bfP_{t})z^{k,*}\| \\
&=\|(\bfP_{t}-\bfP^*)z^{k,*}\|.
\end{align*}
\end{minipage}

Recall that $\bfP_t = \bfB_t\bfB_t^{\top}$ and $\bfP^* = \bfB^*(\bfB^*)^{\top}$. We first rewrite $\bfP_{t}-\bfP^*$ to relate it to the cross terms $\bfB_t^{\top}\bfB_{\perp}^*$ and $\bfB_{t,\perp}^{\top}\bfB^*$. 
Observe that 
\begin{align*}
    \begin{bmatrix}
        \bfB_t^{\top} \\\bfB_{t,\perp}^{\top}
    \end{bmatrix} (\bfB_t\bfB_t^\top-\bfB^*\bfB^{*\top}) \begin{bmatrix}
        \bfB_{\perp}^* &\bfB^*
    \end{bmatrix} = \begin{bmatrix}
        \bfB_t^{\top}\bfB_{\perp}^* &\textbf{0} \\
        \textbf{0} & -\bfB_{t,\perp}^{\top}\bfB^*
    \end{bmatrix}.
\end{align*}
It is easy to check that 
\begin{align*}
\bfB_t\bfB_t^\top-\bfB^*\bfB^{*\top} 
= 
\begin{bmatrix}
        \bfB_t &\bfB_{t,\perp}
\end{bmatrix} 
\begin{bmatrix}
\bfB_t^{\top}\bfB_{\perp}^* &\textbf{0} \\
\textbf{0} & -\bfB_{t,\perp}^{\top}\bfB^*
\end{bmatrix} 
\begin{bmatrix}
\bfB_{\perp}^{*\top} \\\bfB^{*\top}.
\end{bmatrix}
\end{align*}

We can do singular value decomposition for the diagonal blocks; for ease of exposition, we drop the time index $t$ in the decomposition. 
That is 
\begin{align*}
    \bfB_t^{\top}\bfB_{\perp}^* = \bfX_1 \bf\Sigma_1 \bfY_1^\top, ~~~\text{and} ~~~ 
    -\bfB_{t,\perp}^{\top}\bfB^* = \bfX_2 \bf\Sigma_2 \bfY_2^\top. 
\end{align*}
Then, we have 
\begin{align*}
    &\bfB_t\bfB_t^\top-\bfB^*\bfB^{*\top} = \begin{bmatrix}
        \bfB_t &\bfB_{t,\perp}
    \end{bmatrix}\begin{bmatrix}
        \bfX_1&\textbf{0}\\
        \textbf{0}& \bfX_2
    \end{bmatrix}\begin{bmatrix}
        \bf\Sigma_1&\textbf{0}\\
        \textbf{0}& \bf\Sigma_2
    \end{bmatrix}\begin{bmatrix}
        \bfY_1^\top&\textbf{0}\\
        \textbf{0}& \bfY_2^\top
    \end{bmatrix}\begin{bmatrix}
        \bfB_{\perp}^{*\top} \\\bfB^{*\top}
    \end{bmatrix}\\
    =&\begin{bmatrix}
        \bfB_t\bfX_1&
         \bfB_{t,\perp}\bfX_2
    \end{bmatrix}\begin{bmatrix}
        \bf\Sigma_1&\textbf{0}\\
        \textbf{0}& \bf\Sigma_2
    \end{bmatrix}\begin{bmatrix}
        (\bfB_{\perp}^*\bfY_1)^\top\\
        (\bfB^*\bfY_2)^\top
    \end{bmatrix}.
\end{align*}
So, 
\begin{align*}
    &\|(\bfP_t-\bfP^*)z^{k,*}\| = \lnorm{\begin{bmatrix}
        \bfB_t\bfX_1&
         \bfB_{t,\perp}\bfX_2
    \end{bmatrix}\begin{bmatrix}
        \bf\Sigma_1&\textbf{0}\\
        \textbf{0}& \bf\Sigma_2
    \end{bmatrix}\begin{bmatrix}
        (\bfB_{\perp}^*\bfY_1)^\top\\
        (\bfB^*\bfY_2)^\top
    \end{bmatrix}z^{k,*}}{}\\
    =&\lnorm{\begin{bmatrix}
        \bf\Sigma_1&\textbf{0}\\
        \textbf{0}& \bf\Sigma_2
    \end{bmatrix}\begin{bmatrix}
        (\bfB_{\perp}^*\bfY_1)^\top\\
        (\bfB^*\bfY_2)^\top
    \end{bmatrix}z^{k,*}}{}     
\end{align*}

From the proof of \citep[Lemma 2.4]{Chen_2021}, we know that when $d\geq 2r$, $\bf\Sigma_1$ and $\bf\Sigma_2$ are identical up to permutation. 
Let $\sigma_1$ denote the largest singular value.  
From \citep[Lemma 2.5]{Chen_2021}, we know that 
\begin{align}
\label{eq: singular value: principal angle}
\sigma_1 = \|\bfB_t^{\top}\bfB_{\perp}^*\| = \|\bfB_{t,\perp}^{\top}\bfB^*\| = \|\bfB_t\bfB_t^\top-\bfB^*\bfB^{*\top}\| = \|m_t\|. 
\end{align} 
Let $\bfv_1$ and $\bfv_1^*$ denote the corresponding right singular vectors of $\bfB_t^{\top}\bfB_{\perp}^*$ and $\bfB_{t,\perp}^{\top}\bfB^*$.
Let $\tilde{\bfv}_1 = B^* \bfv_1 \in \reals^d$, and $\tilde{\bfv}^*_1 = B^*_{\perp} \bfv_1^*\in \reals^d$.  It is easy to see that $\tilde{\bfv}_1\in \text{span}(\bfB^*)$ and $\tilde{\bfv}_1^*\in \text{span}(\bfB^*_{\perp})$. 
Hence, we get 
\begin{align*}
    \|(\bfP_t-\bfP^*)z^{k,*}\| 
    =&\lnorm{\begin{bmatrix}
        \bf\Sigma_1&\textbf{0}\\
        \textbf{0}& \bf\Sigma_2
    \end{bmatrix}\begin{bmatrix}
        (\bfB_{\perp}^*\bfY_1)^\top\\
        (\bfB^*\bfY_2)^\top
    \end{bmatrix}z^{k,*}}{}\\
    \geq &\sqrt{(\sigma_1 \tilde{\bfv}_1^\top z^{k,*})^2+(\sigma_1 (\tilde{\bfv}_1^*)^\top z^{k,*})^2}\\
     =  & \sigma_1 |\tilde{\bfv}_1^\top z^{k,*}|, 
\end{align*}
where the last line follows from $z^{k,*}\in\text{span}(\bfB^*)$ and thus $(\tilde{\bfv}_1^*)^\top z^{k,*} = 0, \forall k$.

Recall from Eq.\,(\ref{eq: collective weights}) that $\bfZ^*\in \reals^{d\times K}$ stacks $z^{k,*}$ as columnes. 
Averaging the reconstruction error across all agents, we get
\begin{align*}
    &\frac{1}{K}\sum_{k=1}^K \|x_t^k\|^2\geq \frac{1}{K}\sum_{k=1}^K \sigma_1^2 (\tilde{\bfv}_1^\top z^{k,*})^2 =\sigma_1^2\frac{1}{K}\sum_{k=1}^K  (\tilde{\bfv}_1^\top z^{k,*})^2\\
    =&\sigma_1^2\frac{1}{K}\lnorm{\tilde{\bfv}_1^\top \bfZ^*}{}^2\\
    =&\sigma_1^2\frac{1}{K}\tilde{\bfv}_1^\top \bfZ^*\bfZ^{*\top}\tilde{\bfv}_1.\\
\end{align*}

Recall from Eq.\,(\ref{eq: low dimensional rewrite}) that $z^{k,*} = \bfB^*\omega^{k,*}$.  
Then,
\begin{align*}
\frac{\sigma_1^2}{K}\tilde{\bfv}_1^\top \bfZ^*\bfZ^{*\top}\tilde{\bfv}_1 \geq \frac{\sigma_1^2}{K}\lambda^+_{\min} (\bfZ^*\bfZ^{*\top}).
\end{align*}
Therefore, we have\begin{align*}
    \frac{1}{K}\sum_{k=1}^K \|x_t^k\|^2\geq \frac{\sigma_1^2}{K}\lambda^+_{\min} (\bfZ^*\bfZ^{*\top}) = \frac{\|m_t\|^2}{K}\lambda^+_{\min} (\bfZ^*\bfZ^{*\top}), 
\end{align*}
where the last equality follows from Eq.\,(\ref{eq: singular value: principal angle}).

\newpage 

\section{Proof of Lemma \ref{lm: PAD analysis} (Bounds on Principal Angle Distance)} 
\label{sec: proof of principal angle distance}

Recall from Algorithm \ref{alg: fedac-per} of the update of $\bar \bfB$ that 
$\bar{\mathbf{B}}_{t+1} = \mathbf{B}_t
    +\frac{\zeta}{KL}\sum_{k=1}^K \bfB_{t,\perp}\bfB_{t,\perp}^\top\delta_{t,L}^k\phi(s_{t,0}^k)(\omega_{t}^k)^\top. $
Since $\mathbf{B}_{t+1} = \bar{\mathbf{B}}_{t+1} \mathbf{R}_{t+1}^{-1}$, we have 
\begin{align*}
    \|\bfB_\perp^{* \top}\mathbf{B}_{t+1}\|^2_F = \|\bfB_\perp^{* \top}\bar{\mathbf{B}}_{t+1}\bfR^{-1}_{t+1}\|^2_F\leq \|\bfB_\perp^{* \top}\bar{\mathbf{B}}_{t+1}\|^2_F \|\bfR^{-1}_{t+1}\|^2.
\end{align*}
That is,
\begin{align*}
    \|m_{t+1}\|_F^2 \leq \|\bar m_{t+1}\|^2_F \|\bfR^{-1}_{t+1}\|^2.
\end{align*}
For $\|\bar m_{t+1}\|^2_F$, we have
\begin{align}
\label{eq: inner product: principal angle: upper bound}
    &\|\bar m_{t+1}\|^2_F = \|\bfB_\perp^{* \top}\bar{\mathbf{B}}_{t+1}\|^2_F=\|\bfB_\perp^{* \top}\mathbf{B}_t+\frac{\zeta}{KL}\sum_{k=1}^K \bfB_\perp^{* \top}\bfB_{t,\perp}\bfB_{t,\perp}^\top\delta_{t,L}^k\phi(s_{t,0}^k)(\omega_{t}^k)^\top\|^2_F\nonumber\\
    =&\|\bfB_\perp^{* \top}\mathbf{B}_t\|_F^2
    +2\langle\bfB_\perp^{* \top}\mathbf{B}_t,\frac{\zeta}{KL}\sum_{k=1}^K \bfB_\perp^{* \top}\bfB_{t,\perp}\bfB_{t,\perp}^\top\delta_{t,L}^k\phi(s_{t,0}^k)(\omega_{t}^k)^\top\rangle+\|\frac{\zeta}{KL}\sum_{k=1}^K \bfB_\perp^{* \top}\bfB_{t,\perp}\bfB_{t,\perp}^\top\delta_{t,L}^k\phi(s_{t,0}^k)(\omega_{t}^k)^\top\|_F^2 \nonumber\\
    = &\|m_t\|_F^2+ \underbrace{\|\frac{\zeta}{KL}\sum_{k=1}^K \bfB_\perp^{* \top}\bfB_{t,\perp}\bfB_{t,\perp}^\top\delta_{t,L}^k\phi(s_{t,0}^k)(\omega_{t}^k)^\top\|_F^2}_{I.1}
    +2\langle\bfB_\perp^{* \top}\mathbf{B}_t,\frac{\zeta}{KL}\sum_{k=1}^K\bfB_\perp^{* \top} \bfP_{t,\perp}\xi_{t,L}^k(\omega_{t}^k)^\top\rangle\nonumber\\
    &+\underbrace{2\langle\bfB_\perp^{* \top}\mathbf{B}_t,\frac{\zeta}{K}\sum_{k=1}^K\bfB_\perp^{* \top} \bfP_{t,\perp}y_t^k\phi(s_{t,0}^k)(\omega_{t}^k)^\top\rangle}_{I.2}
    +\underbrace{2\langle\bfB_\perp^{* \top}\mathbf{B}_t,\frac{\zeta}{KL}\sum_{k=1}^K\bfB_\perp^{* \top} \bfP_{t,\perp}A_L^kx_t^k(\omega_{t}^k)^\top\rangle}_{I.3},
    \end{align}
where the last equality follows from Proposition \ref{prop: key intermediate}.

For I.1 in Eq.\,(\ref{eq: inner product: principal angle: upper bound}), recalling that $\|\bfQ_t\|_F^2=\|\frac{\zeta}{KL}\sum_{k=1}^K \bfP_{t,\perp}\delta_{t,L}^k\phi(s_{t,0}^k)(\omega_{t}^k)^\top\|_F^2$, we have 
\begin{align*}
 I.1 = &\|\frac{\zeta}{KL}\sum_{k=1}^K\bfB_\perp^{* \top} \bfP_{t,\perp}\delta_{t,L}^k\phi(s_{t,0}^k)(\omega_{t}^k)^\top\|_F^2\\
    =&\|\bfB_\perp^{* \top}\frac{\zeta}{K}\sum_{k=1}^K \bfP_{t,\perp}\delta_{t,L}^k\phi(s_{t,0}^k)(\omega_{t}^k)^\top\|_F^2\\
    \leq& \|\bfB_\perp^{* \top}\|^2\|\bfQ_t\|_F^2\\
    \leq& \|\bfQ_t\|_F^2.
\end{align*}
Similar to the derivation in the proof of Lemma \ref{lm: local head: upper bound}, I.2 in Eq.\,(\ref{eq: inner product: principal angle: upper bound}) can be bounded as 
\begin{align*}
I.2 = 2\langle m_t,\frac{\zeta}{K}\sum_{k=1}^K\bfB_\perp^{* \top} \bfP_{t,\perp}y_t^k\phi(s_{t,0}^k)(\omega_{t}^k)^\top\rangle
\le 2 \|m_t\| \frac{\zeta}{K}\sum_{k=1}^K \|\bfB_\perp^{* \top} \bfP_{t,\perp}y_t^k\phi(s_{t,0}^k)(\omega_{t}^k)^\top\| 
\le 
2U_\omega \frac{\zeta}{K}\sum_{k=1}^K \|m_t\||y_t^k|.   
\end{align*}

$I.3$ in Eq.\,(\ref{eq: inner product: principal angle: upper bound}) is related to the negative drift of the $m_t$. Specifically,  
\begin{align*}
 \frac{1}{2}I.3 =  &\langle\bfB_\perp^{* \top}\mathbf{B}_t,\frac{\zeta}{KL}\sum_{k=1}^K\bfB_\perp^{* \top} \bfP_{t,\perp} A_L^kx_t^k(\omega_{t}^k)^\top\rangle\\
 & = \frac{\zeta}{KL}\sum_{k=1}^K 
\pth{ \bfP_{t,\perp} \bfP_{\perp}^* \bfB_t \omega_t^k }^\top A_L^k \pth{ \bfP_{t} + \bfP_{t,\perp} } x_t^k,
\end{align*}
which is true because 
\[
\iprod{AB}{vw^{\top}} = \text{trace}\pth{A^{\top} vw^{\top}} = w^{\top} B^{\top}A^{\top} v
\]
for any $r\times r$ matrix $A$ and $r$-dimensional vectors $v$ and $w$. 
Additionally,
\begin{align*}
\pth{ \bfP_{t,\perp} \bfP_{\perp}^* \bfB_t \omega_t^k }^\top A_L^k \pth{ \bfP_{t} + \bfP_{t,\perp} } x_t^k 
& \le \pth{ \bfP_{t,\perp} \bfP_{\perp}^* x_t^k }^\top A_L^k \bfP_{t,\perp} x_t^k 
+ 2 U_\omega \|m_t\|\|x_t^k\| \\
& \le \pth{ \bfP_{t,\perp} x_t^k }^\top A_L^k \bfP_{t,\perp} x_t^k 
+ 6 U_\omega \|m_t\|\|x_t^k\| \\
& \le - \lambda L \| \bfP_{t,\perp} x_t^k \|^2 + 6 U_\omega \|m_t\|\|x_t^k\|.
\end{align*}

For each $k\in [K]$, it holds that 
\begin{align*}
\|\bfP_{t,\perp} x_t^k\|^2 = \|\bfP_{t,\perp}(\mathbf{B}_t\omega_{t}^k-\bfB^*\omega^{k,*}) \|^2 
= \|\bfP_{t,\perp} z^{k,*}\|
= \|(\bfP_{t} - \bfP^*)z^{k,*}\|. 
\end{align*}
From the analysis in the proof of Lemma \ref{lm: local head: lower bound}, we know that 
\[
\frac{1}{K}\sum_{k=1}^K \|(\bfP_{t} - \bfP^*)z^{k,*}\|^2 \ge \|m_t\|_F^2 \frac{\lambda_{\min}(\bfZ^*\bfZ^{*\top})}{rK}. 
\]
So, 
\begin{align*}
\frac{1}{K} \sum_{k=1}^K\|\bfP_\perp^{*}(\mathbf{B}_t\omega_{t}^k-\bfB^*\omega^{k,*}) \|^2 
&\ge \|m_t\|_F^2 \frac{\lambda_{\min}(\bfZ^*\bfZ^{*\top})}{rK}. 
\end{align*}
Thus, 
\begin{align}
\label{eq: mt: I3: overall}
I.3 \le - 2\lambda\zeta \frac{\lambda_{\min}(\bfZ^*\bfZ^*)^{\top}}{rK} \|m_t\|_F^2 
+ 12U_\omega \zeta\frac{1}{KL}\sum_{k=1}^K\|m_t\|\|x_t^k\|.
\end{align}
Combining the bounds of I.1, I.2, and I.3 of Eq.\,(\ref{eq: inner product: principal angle: upper bound}), and applying Assumption \ref{assp: eigval ZZ^T}, we have  
\begin{align}
\label{eq: mt: first-stage bound}
    &\|\bar m_{t+1}\|_F^2 \leq (1-\frac{2\lambda\alpha \zeta}{r})\|m_{t}\|_F^2+\|\bfQ_t\|_F^2+2\langle m_t,\frac{\zeta}{KL}\sum_{k=1}^K \bfB_\perp^{*\top} \bfP_{t,\perp}\xi_{t,L}^k(\omega_t^k)^\top \rangle \nonumber \\
    &~~~~~~~~+2U_\omega \frac{\zeta}{K}\sum_{k=1}^K \|m_t\||y_t^k|+12U_\omega  \zeta\frac{1}{KL}\sum_{k=1}^K\|m_t\|\|x_t^k\|. 
\end{align}

Recall that $\|m_{t+1}\|_F^2 \leq \|\bar m_{t+1}\|^2_F \|\bfR^{-1}_{t+1}\|^2$,
we have \begin{align*}
    &\|m_{t+1}\|_F^2 \leq (1-\frac{2\lambda\alpha \zeta}{r})\|m_{t}\|_F^2\|\bfR^{-1}_{t+1}\|^2 \\
    &~~~~~~~~+\|\bfQ_t\|_F^2\|\bfR^{-1}_{t+1}\|^2+2\langle m_t,\frac{\zeta}{KL}\sum_{k=1}^K \bfB_\perp^{*\top} \bfP_{t,\perp}\xi_{t,L}^k(\omega_t^k)^\top \rangle\|\bfR^{-1}_{t+1}\|^2\\
    &~~~~~~~~+2U_\omega\frac{\zeta}{K}\sum_{k=1}^K \|m_t\||y_t^k|\|\bfR^{-1}_{t+1}\|^2+12U_\omega \zeta\frac{1}{KL}\sum_{k=1}^K\|m_t\|\|x_t^k\|\|\bfR^{-1}_{t+1}\|^2. 
\end{align*}

For the first term, we use Eq.\,(\ref{eqn: R related bounds: R inv}) that ${ \|\bfR_{t+1}^{-1}\|^2\leq (\frac{1}{1- U_\delta^2 U_\omega^2\zeta^2})^2}$, and for the remaining terms, we use a crude bound that $\|\bfR_{t+1}^{-1}\|^2\leq  (\frac{1}{1- U_\delta^2 U_\omega^2\zeta^2})^2 \leq 2$ when $ U_\delta^2 U_\omega^2\zeta^2\leq \frac{1}{6}$, then,
\begin{align*}
    &\|m_{t+1}\|_F^2 \leq (1-\frac{\lambda\alpha \zeta}{r}) (\frac{1}{1- U_\delta^2 U_\omega^2\zeta^2})^2\|m_{t}\|_F^2+2\|\bfQ_t\|_F^2 +4\langle m_t,\frac{\zeta}{KL}\sum_{k=1}^K \bfB_\perp^{*\top} \bfP_{t,\perp}\xi_{t,L}^k(\omega_t^k)^\top \rangle \\
    &~~~~~~~~+4U_\omega\frac{\zeta}{K}\sum_{k=1}^K \|m_t\||y_t^k|+24U_\omega\zeta\frac{1}{KL}\sum_{k=1}^K\|m_t\|\|x_t^k\|. 
\end{align*}
Moreover, since $U_\delta^2 U_\omega^2\zeta^2\leq \frac{1}{6}$, we have $(\frac{1}{1-U_\delta^2 U_\omega^2\zeta^2})^2 \leq 1+3U_\delta^2 U_\omega^2\zeta^2$.
Providing $\frac{\lambda\alpha }{r} \geq 3U_\delta^2 U_\omega^2\zeta$, it holds that 
\begin{align*}
    &(1-\frac{2\lambda\alpha \zeta}{r}) (\frac{1}{1- U_\delta^2 U_\omega^2\zeta^2})^2 \leq (1-\frac{2\lambda\alpha\zeta}{r})(1+3U_\delta^2 U_\omega^2\zeta^2) \leq (1-\frac{\lambda\alpha \zeta}{r}).
\end{align*}
Then, 
\begin{align}
\label{eqn: m squared intermediate bound}
    &\|m_{t+1}\|_F^2 \leq (1-\frac{\lambda\alpha \zeta}{r}) \|m_{t}\|_F^2+2\|\bfQ_t\|_F^2 +4\langle m_t,\frac{\zeta}{KL}\sum_{k=1}^K \bfB_\perp^{*\top} \bfP_{t,\perp}\xi_{t,L}^k(\omega_t^k)^\top \rangle\nonumber\\
    &~~~~~~~~+4U_\omega\frac{\zeta}{K}\sum_{k=1}^K \|m_t\||y_t^k|+24U_\omega \zeta\frac{1}{KL}\sum_{k=1}^K\|m_t\|\|x_t^k\|. %
\end{align}

By Young's inequality, we have for any $c_1>0,c_2>0$,
\begin{align*}
    &4U_\omega\frac{\zeta}{K}\sum_{k=1}^K \|m_t\||y_t^k|\leq c_1^2\zeta \|m_t\|^2 +  \frac{4U_\omega^2}{c_1^2}\zeta\frac{1}{K}\sum_{k=1}^K|y_t^k|^2, \\
    &16U_\omega \zeta\frac{1}{KL}\sum_{k=1}^K\|m_t\|\|x_t^k\|\leq c_2^2\zeta\|m_t\|^2 + \frac{144U_\omega^2}{c_2^2}\frac{\zeta}{L^2}\frac{1}{K}\sum_{k=1}^K\|x_t^k\|^2.
\end{align*}
Putting this back to \eqref{eqn: m squared intermediate bound}, we get,
\begin{align}
    \label{eqn: m squared intermediate bound after young}
    &\|m_{t+1}\|_F^2 \leq (1-\frac{\lambda\alpha \zeta}{r}+ c_1^2\zeta+c_2^2\zeta) \|m_{t}\|_F^2+2\|\bfQ_t\|_F^2 +4\langle m_t,\frac{\zeta}{KL}\sum_{k=1}^K \bfB_\perp^{*\top} \bfP_{t,\perp}\xi_{t,L}^k(\omega_t^k)^\top \rangle\nonumber\\
    &~~~~~~~~+  \frac{4U_\omega^2}{c_1^2}\zeta\frac{1}{K}\sum_{k=1}^K|y_t^k|^2
    + \frac{144U_\omega^2}{c_2^2}\frac{\zeta}{L^2}\frac{1}{K}\sum_{k=1}^K\|x_t^k\|^2.
\end{align}
For the third term of Eq.\,(\ref{eqn: m squared intermediate bound after young}), we have
\begin{align}
\label{eqn: markovian noise in M decomp}
    &\langle m_t,\frac{\zeta}{KL}\sum_{k=1}^K \bfB_\perp^{*\top}{ \bfP_{t,\perp}}\xi_{t,L}^k(\omega_t^k)^\top \rangle\nonumber\\
    =&\langle m_t,\frac{\zeta}{KL}\sum_{k=1}^K \bfB_\perp^{*\top}{ \bfP_{t,\perp}}\big[\tilde b_{t,L}^k - \bar b_L^k+ (\tilde A_{t,L}^k-A_L^k)\bfB_t\omega_t^k\big](\omega_t^k)^\top \rangle ~~~~~(\text{by \eqref{eqn: xi decomp}})\nonumber\\
    =&\langle m_t,\frac{\zeta}{KL}\sum_{k=1}^K \bfB_\perp^{*\top}{ \bfP_{t,\perp}}\big[\tilde b_{t,L}^k - \bar b_L^k\big](\omega_t^k)^\top \rangle
    + \langle m_t,\frac{\zeta}{KL}\sum_{k=1}^K \bfB_\perp^{*\top}{ \bfP_{t,\perp}} (\tilde A_{t,L}^k-A_L^k)\bfB_t\omega_t^k(\omega_t^k)^\top \rangle\nonumber\\
    =&\langle m_t, \bfB_\perp^{*\top}{ \bfP_{t,\perp}}\frac{\zeta}{KL}\sum_{k=1}^K\big[\tilde b_{t,L}^k - \bar b_L^k\big](\omega_t^k- \omega_{t-\tau}^k)^\top \rangle
    +\langle m_t-m_{t-\tau}, \bfB_\perp^{*\top}{ \bfP_{t,\perp}}\frac{\zeta}{KL}\sum_{k=1}^K\big[\tilde b_{t,L}^k - \bar b_L^k\big](\omega_{t-\tau}^k)^\top \rangle \nonumber\\
    &+\langle m_{t-\tau}, \bfB_\perp^{*\top} (\bfP_{t,\perp}-\bfP_{t-\tau,\perp})\frac{\zeta}{KL}\sum_{k=1}^K\big[\tilde b_{t,L}^k - \bar b_L^k\big](\omega_{t-\tau}^k)^\top \rangle 
    + \underbrace{\langle m_{t-\tau}, \bfB_\perp^{*\top} \bfP_{t-\tau,\perp}\frac{\zeta}{KL}\sum_{k=1}^K\big[\tilde b_{t,L}^k - \bar b_L^k\big](\omega_{t-\tau}^k)^\top \rangle}_{I_1}\nonumber\\
    &+ \langle m_t, \bfB_\perp^{*\top}{ \bfP_{t,\perp}} \frac{\zeta}{KL}\sum_{k=1}^K(\tilde A_{t,L}^k-A_L^k)\pth{\bfB_t-\bfB_{t-\tau}}\omega_t^k(\omega_t^k)^\top \rangle \nonumber\\
    &+ \langle m_t, \bfB_\perp^{*\top}{ \bfP_{t,\perp}} \frac{\zeta}{KL}\sum_{k=1}^K(\tilde A_{t,L}^k-A_L^k)\bfB_{t-\tau}(\omega_t^k-\omega_{t-\tau}^k)(\omega_t^k)^\top \rangle           \nonumber\\
    &+ \langle m_t, \bfB_\perp^{*\top}{ \bfP_{t,\perp}} \frac{\zeta}{KL}\sum_{k=1}^K(\tilde A_{t,L}^k-A_L^k)\bfB_{t-\tau}\omega_{t-\tau}^k(\omega_t^k-\omega_{t-\tau}^k)^\top \rangle           \nonumber\\
    &+ \langle m_t-m_{t-\tau}, \bfB_\perp^{*\top}{ \bfP_{t,\perp}} \frac{\zeta}{KL}\sum_{k=1}^K(\tilde A_{t,L}^k-A_L^k)\bfB_{t-\tau}\omega_{t-\tau}^k(\omega_{t-\tau}^k)^\top \rangle           \nonumber\\
    &+ \langle m_{t-\tau}, \bfB_\perp^{*\top}\pth{ \bfP_{t,\perp}-\bfP_{t-\tau,\perp}} \frac{\zeta}{KL}\sum_{k=1}^K(\tilde A_{t,L}^k-A_L^k)\bfB_{t-\tau}\omega_{t-\tau}^k(\omega_{t-\tau}^k)^\top \rangle           \nonumber\\
    &+ \underbrace{\langle m_{t-\tau}, \bfB_\perp^{*\top}{ \bfP_{t-\tau,\perp}} \frac{\zeta}{KL}\sum_{k=1}^K(\tilde A_{t,L}^k-A_L^k)\bfB_{t-\tau}\omega_{t-\tau}^k(\omega_{t-\tau}^k)^\top \rangle}_{I_2}       .
\end{align}
Note that $I_1$ and $I_2$ in the above decomposition can be bounded under expectation conditioned on $v_{0:t-\tau}$, and other terms are in terms of $\|\bfB_t-\bfB_{t-\tau}\|^2$, $\|\omega_t^k-\omega_{t-\tau}^k\|^2$, $\lnorm{\tilde b_{t,L}^k - \bar b_L^k}{}^2$, and $\lnorm{\tilde A_{t,L}^k-A_L^k}{}^2$ by applying Cauchy-Schwartz. Specifically, for the first term, we have  
\begin{align*}
    &\langle m_t, \bfB_\perp^{*\top}{ \bfP_{t,\perp}}\frac{\zeta}{KL}\sum_{k=1}^K\big[\tilde b_{t,L}^k - \bar b_L^k\big](\omega_t^k- \omega_{t-\tau}^k)^\top \rangle\\
    \leq& \frac{\zeta}{KL}\sum_{k=1}^K\lnorm{\tilde b_{t,L}^k - \bar b_L^k}{} \lnorm{\omega_t^k- \omega_{t-\tau}^k}{}\\
    \leq& \frac{\zeta}{KL}\sum_{k=1}^K\pth{\frac{\zeta}{2L}\lnorm{\tilde b_{t,L}^k - \bar b_L^k}{}^2+ \frac{L}{2\zeta}\lnorm{\omega_t^k- \omega_{t-\tau}^k}{}^2}\\
    \leq& \frac{1}{2}\frac{\zeta^2}{K L^2}\sum_{k=1}^K \lnorm{\tilde b_{t,L}^k - \bar b_L^k}{}^2 + \frac{1}{2K}\sum_{k=1}^K \lnorm{\omega_t^k- \omega_{t-\tau}^k}{}^2.
\end{align*}

Applying similar analysis to all other terms in Eq.\,(\ref{eqn: markovian noise in M decomp}), we obtain 
\begin{align*}
    &\langle m_t-m_{t-\tau}, \bfB_\perp^{*\top}{ \bfP_{t,\perp}}\frac{\zeta}{KL}\sum_{k=1}^K\big[\tilde b_{t,L}^k - \bar b_L^k\big](\omega_{t-\tau}^k)^\top \rangle 
    \leq\frac{1}{2}\|\bfB_t-\bfB_{t-\tau}\|^2 + \frac{U_\omega^2}{2}\frac{\zeta^2}{K L^2} \sum_{k=1}^K\lnorm{\tilde b_{t,L}^k - \bar b_L^k }{}^2 ,   \\
    &\langle m_{t-\tau}, \bfB_\perp^{*\top} (\bfP_{t,\perp}-\bfP_{t-\tau,\perp})\frac{\zeta}{KL}\sum_{k=1}^K\big[\tilde b_{t,L}^k - \bar b_L^k\big](\omega_{t-\tau}^k)^\top \rangle
    \leq2\|\bfB_{t}-\bfB_{t-\tau}\|^2+\frac{1}{2}U_\omega^2\frac{\zeta^2}{KL^2} \sum_{k=1}^K\lnorm{\tilde b_{t,L}^k - \bar b_L^k }{}^2,\\
       & \langle m_t, \bfB_\perp^{*\top}{ \bfP_{t,\perp}} \frac{\zeta}{KL}\sum_{k=1}^K(\tilde A_{t,L}^k-A_L^k)\pth{\bfB_t-\bfB_{t-\tau}}\omega_t^k(\omega_t^k)^\top \rangle 
\leq  \frac{1}{2} U_\omega^4\frac{\zeta^2}{KL^2}\sum_{k=1}^K\|\tilde A_{t,L}^k-A_L^k\|^2 + \frac{1}{2} \|\bfB_t-\bfB_{t-\tau} \|^2,  \\
    & \langle m_t, \bfB_\perp^{*\top}{ \bfP_{t,\perp}} \frac{\zeta}{KL}\sum_{k=1}^K(\tilde A_{t,L}^k-A_L^k)\bfB_{t-\tau}(\omega_t^k-\omega_{t-\tau}^k)(\omega_t^k)^\top \rangle \\          
    &\qquad \qquad \qquad \qquad
    \leq\frac{U_\omega}{2} \frac{\zeta^2}{KL^2}\sum_{k=1}^K\|\tilde A_{t,L}^k-A_L^k \|^2+\frac{U_\omega}{2}\! \frac{1}{K}\sum_{k=1}^K\|\omega_t^k-\omega_{t-\tau}^k \|^2,           \\
    & \langle m_t, \bfB_\perp^{*\top}{ \bfP_{t,\perp}} \frac{\zeta}{KL}\sum_{k=1}^K(\tilde A_{t,L}^k-A_L^k)\bfB_{t-\tau}\omega_{t-\tau}^k(\omega_t^k-\omega_{t-\tau}^k)^\top \rangle           \\
    &\qquad \qquad \qquad \qquad
    \leq \frac{U_\omega}{2}\frac{\zeta^2}{KL^2}\sum_{k=1}^K\|\tilde A_{t,L}^k-A_L^k\|^2+ \frac{U_\omega}{2} \frac{1}{K}\sum_{k=1}^K\|\omega_t^k-\omega_{t-\tau}^k\|^2 ,         \\
    & \langle m_t-m_{t-\tau}, \bfB_\perp^{*\top}{ \bfP_{t,\perp}} \frac{\zeta}{KL}\sum_{k=1}^K(\tilde A_{t,L}^k-A_L^k)\bfB_{t-\tau}\omega_{t-\tau}^k(\omega_{t-\tau}^k)^\top \rangle           \\
    &\qquad \qquad \qquad \qquad
    \leq \frac{1}{2}\| \bfB_t-\bfB_{t-\tau}\|^2+\frac{1}{2} U_\omega^4\frac{\zeta^2}{KL^2}\sum_{k=1}^K\|\tilde A_{t,L}^k-A_L^k \|^2,\\
    & \langle m_{t-\tau}, \bfB_\perp^{*\top}\pth{ \bfP_{t,\perp}-\bfP_{t,\perp}} \frac{\zeta}{KL}\sum_{k=1}^K(\tilde A_{t,L}^k-A_L^k)\bfB_{t-\tau}\omega_{t-\tau}^k(\omega_{t-\tau}^k)^\top \rangle           \\
    &\qquad \qquad \qquad \qquad
    \leq 2\| \bfB_{t}-\bfB_{t-\tau} \|^2+\frac{U_\omega^4}{2} \frac{\zeta^2}{KL^2}\sum_{k=1}^K\|\tilde A_{t,L}^k-A_L^k\|^2
\end{align*}

Combining the above bounds and pulgging them back to \eqref{eqn: markovian noise in M decomp}, and taking expectation conditioned on $v_{0:t-\tau}$, we get 
\begin{align}
\label{eqn: mt noise combined param perturb}
    &\bbE_{t-\tau}\langle m_t,\frac{\zeta}{KL}\sum_{k=1}^K \bfB_\perp^{*\top}{ \bfP_{t,\perp}}\xi_{t,L}^k(\omega_t^k)^\top \rangle\nonumber\\
    \leq&
    \pth{\frac{1}{2}+U_\omega^2}\frac{\zeta^2}{K L^2}\sum_{k=1}^K \bbE_{t-\tau}\lnorm{\tilde b_{t,L}^k - \bar b_L^k}{}^2 + \pth{\frac{1}{2}+U_\omega}\frac{1}{K}\sum_{k=1}^K\bbE_{t-\tau}\lnorm{\omega_t^k- \omega_{t-\tau}^k}{}^2\nonumber\\
    &+\pth{U_\omega+\frac{3}{2}U_\omega^4} \frac{\zeta^2}{KL^2}\sum_{k=1}^K\bbE_{t-\tau}\|\tilde A_{t,L}^k-A_L^k\|^2+  8\bbE_{t-\tau}\| \bfB_{t}-\bfB_{t-\tau} \|^2\nonumber\\
    &+ \underbrace{\langle m_{t-\tau},\frac{\zeta}{KL}\sum_{k=1}^K \bfB_\perp^{*\top}\bfP_{t-\tau,\perp}\bbE_{t-\tau}\big[\tilde b_{t,L}^k - \bar b_L^k\big] (\omega_{t-\tau}^k)^\top\rangle}_{I_1}\nonumber\\
    &+\underbrace{\langle m_{t-\tau},\frac{\zeta}{KL}\sum_{k=1}^K \bfB_\perp^{*\top}\bfP_{t-\tau,\perp} \bbE_{t-\tau}[\tilde A_{t,L}^k-A_L^k]\bfB_{t-\tau}\omega_{t-\tau}^k(\omega_{t-\tau}^k)^\top \rangle}_{I_2}.
\end{align}
Note that Lemma \ref{lmm: local noise reduction} holds for any $t\ge 0$ and any state $s_{t,k}^0$. Applying Lemma \ref{lmm: local noise reduction}, we have 
\begin{align}
\label{eqn: expected b diff bound}
\bbE_{t-\tau}\lnorm{\tilde b_{t,L}^k - \bar b_L^k}{}^2 
\leq  \pth{16U_r^2 + 32U_r^2\frac{m\rho}{1-\rho}}L. 
\end{align}
Similarly, we have,
\begin{align}
\label{eqn: expected A diff bound}
    \bbE_{t-\tau}\|\tilde A_{t,L}^k-A_L^k\|^2\leq \pth{16 + 32\frac{m\rho}{1-\rho}}L. ~~~~(\text{ Lemma \ref{lmm: local noise reduction}, \eqref{eqn: Atilde-Abar ^2 local noise}})
\end{align}

Applying Lemma \ref{lmm: one step temporal difference bound of omega} and Lemma \ref{lmm: one-step perturbation of B}, we have the following,
\begin{align}
\label{eqn: tau step omega bound}
&\bbE_{t-\tau}\| \omega_{t}^k-\omega_{t-\tau}^k \|^2\nonumber\\
\leq& \tau \sum_{j=t-\tau}^{t-1}\bbE_{t-\tau}\| \omega_{j+1}^k-\omega_{j}^k \|^2 \nonumber \\
\leq& \tau^2C_5(\tau^2)\frac{\beta^4}{L^2}
    + \tau^2C_6(\tau^2)\frac{\zeta^2\beta^2}{L^2}
    +8\tau^2\frac{\beta^2}{L^2} \bbE_{t-\tau}\|x_t^k\|^2
    + \tau^2C_7\frac{\beta^2}{L}
    + \tau^2C_8(\tau^2)\gamma^2\beta^2+2\tau^2\beta^2\bbE_{t-\tau}|y_t^k|^2,
\end{align}
and
\begin{align}
\label{eqn: tau step B bound}
&\bbE_{t-\tau}\| \bfB_{t}-\bfB_{t-\tau} \|^2\nonumber\\
\leq& \tau^2C_{16}(\tau^2)\frac{\zeta^2\beta^2}{L^2}
    +\tau^2 C_{17}(\tau^2)\frac{\zeta^4}{L^2}
    + \tau^2C_{18}\frac{\zeta^2}{L^2} \frac{1}{K} \sum_{k=1}^K \bbE_{t-\tau}\|x_t^k\|^2
    +\tau^2 C_{19}(\tau^2)\zeta^2 \beta^2
+\tau^2C_{20}\frac{\zeta^2}{K} \nonumber\\
&+\tau^2 C_{21}\zeta^2m^2\rho^{2\tau}
+ \tau^2C_{22}(\tau^2) \gamma^2\zeta^2
+\tau^2C_{23}\frac{\zeta^2}{K}\sum_{k=1}^K \bbE_{t-\tau}|y_t^k|^2
+ \tau^2C_{24}\zeta^4.
\end{align}

For $I_1$, %
by Lemma \ref{lmm: MC observable mixing}, we have
\begin{align}
\label{eqn: Mt markovian noise mix 1}
    &\bbE_{t-\tau}[I_1] = \bbE_{t-\tau}\langle m_{t-\tau},\frac{\zeta}{KL}\sum_{k=1}^K \bfB_\perp^{*\top}\bfP_{t-\tau,\perp}\big[\tilde b_{t,L}^k - \bar b_L^k\big] (\omega_{t-\tau}^k)^\top\rangle \nonumber\\
    =&\langle m_{t-\tau},\frac{\zeta}{KL}\sum_{k=1}^K \bfB_\perp^{*\top}\bfP_{t-\tau,\perp}\bbE_{t-\tau}\big[\tilde b_{t,L}^k - \bar b_L^k\big] (\omega_{t-\tau}^k)^\top\rangle\nonumber\\
    \leq& 4U_r U_\omega \zeta m\rho^{\tau}.
\end{align}
Similarly, for $I_2$, by Lemma \ref{lmm: MC observable mixing}, we have 
\begin{align}
\label{eqn: Mt markovian noise mix 2}
    &\bbE_{t-\tau}[I_2] = \bbE_{t-\tau}\langle m_{t-\tau},\frac{\zeta}{KL}\sum_{k=1}^K \bfB_\perp^{*\top}\bfP_{t-\tau,\perp} (\tilde A_{t,L}^k-A_L^k)\bfB_{t-\tau}\omega_{t-\tau}^k(\omega_{t-\tau}^k)^\top \rangle\nonumber\\
    \leq& \langle m_{t-\tau},\frac{\zeta}{KL}\sum_{k=1}^K \bfB_\perp^{*\top}\bfP_{t-\tau,\perp} \bbE_{t-\tau}\qth{\tilde A_{t,L}^k-A_L^k}\bfB_{t-\tau}\omega_{t-\tau}^k(\omega_{t-\tau}^k)^\top \rangle\nonumber\\
    \leq&   4 U_\omega^2\frac{\zeta}{L}m\rho^{\tau}.
\end{align}
Putting \eqref{eqn: expected b diff bound}, \eqref{eqn: expected A diff bound}, \eqref{eqn: tau step omega bound}, \eqref{eqn: tau step B bound}, \eqref{eqn: Mt markovian noise mix 1} and \eqref{eqn: Mt markovian noise mix 2} back to \eqref{eqn: mt noise combined param perturb}, taking conditional expectation, and using $\beta = c\zeta$, we get
\begin{align}
\label{eqn: mt noise final}
    &\bbE_{t-\tau}\langle m_t,\frac{\zeta}{KL}\sum_{k=1}^K \bfB_\perp^{*\top}{ \bfP_{t,\perp}}\xi_{t,L}^k(\omega_t^k)^\top \rangle\nonumber\\
    \leq& \pth{\frac{1}{2}+U_\omega^2}\pth{16U_r^2 + 32U_r^2\frac{m\rho}{1-\rho}}\frac{\zeta^2}{ L}\nonumber\\
    & + \pth{\frac{1}{2}+U_\omega}\bigg(\tau^2C_5(\tau^2)c^4\frac{\zeta^4}{L^2}
    + \tau^2C_6(\tau^2)c^2\frac{\zeta^4}{L^2}
    +8\tau^2c^2\frac{\zeta^2}{L^2} \frac{1}{K}\sum_{k=1}^K\bbE_{t-2\tau}\|x_t^k\|^2 \nonumber\\
    &\qquad \qquad+ \tau^2C_7c^2\frac{\zeta^2}{L}
    + \tau^2C_8(\tau^2)c^2\gamma^2\zeta^2
    +2\tau^2c^2\zeta^2\frac{1}{K}\sum_{k=1}^K\bbE_{t-2\tau}|y_t^k|^2\bigg)\nonumber\\
    &+\pth{U_\omega+\frac{3}{2}U_\omega^4} \pth{16 + 32\frac{m\rho}{1-\rho}}\frac{\zeta^2}{L} \nonumber\\
    &+  8\bigg(\tau^2C_{16}(\tau^2)c^2\frac{\zeta^4}{L^2}
    +\tau^2 C_{17}(\tau^2)\frac{\zeta^4}{L^2}
    + \tau^2C_{18}\frac{\zeta^2}{KL^2}\sum_{k=1}^K\bbE_{t-2\tau}\|x_t^k\|^2
    +\tau^2 C_{19}(\tau^2)c^2 \zeta^4
+\tau^2C_{20}\frac{\zeta^2}{K} \nonumber\\
&\qquad \qquad +\tau^2 C_{21}\zeta^2m^2\rho^{2\tau}
+ \tau^2C_{22}(\tau^2) \gamma^2\zeta^2
+\tau^2C_{23}\frac{\zeta^2}{K}\sum_{k=1}^K \bbE_{t-2\tau}|y_t^k|^2
+ \tau^2C_{24}\zeta^4\bigg)\nonumber\\
    &+ 2U_r U_\omega \zeta m\rho^{\tau} \nonumber
    +2U_\omega^2\frac{\zeta}{L}m\rho^{\tau}\nonumber\\
\leq& 
C_{25}(\tau^2)\frac{\zeta^2}{L}
\nonumber
+
C_{26}(\tau^4)\frac{\zeta^4}{L^2}
\nonumber
+
C_{27}(\tau^4)\zeta^4
+
C_{28}(\tau^4)\gamma^2\zeta^2
\nonumber\\ 
&+
C_{29}(\tau^2)\frac{\zeta^2}{KL}\sum_{k=1}^K\bbE_{t-2\tau}\|x_t^k\|^2
\nonumber
+
C_{30}(\tau^2)\frac{\zeta^2}{K}\sum_{k=1}^K\bbE_{t-2\tau}|y_t^k|^2
\nonumber\\ 
&+
8\tau^2C_{20}\frac{\zeta^2}{K}
+
8\tau^2C_{21}\zeta^2 m^2\rho^{2\tau}
+
2U_rU_\omega \zeta m\rho^\tau +2U_\omega^2\frac{\zeta}{L} m\rho^\tau.
\end{align}
where
\begin{align*}
    &C_{25}(\tau^2) = \Bigg[
\pth{\frac12+U_\omega^2}
\pth{16U_r^2 + 32U_r^2\frac{m\rho}{1-\rho}}
+
\pth{U_\omega+\frac32U_\omega^4}
\pth{16 + 32\frac{m\rho}{1-\rho}}
+\pth{\frac12+U_\omega}\tau^2C_7c^2
\Bigg],\\
&C_{26}(\tau^4)=\Bigg[
\pth{\frac12+U_\omega}
\pth{\tau^2C_5(\tau^2)c^4+\tau^2C_6(\tau^2)c^2}
+8\pth{\tau^2C_{16}(\tau^2)c^2+\tau^2C_{17}(\tau^2)}
\Bigg],\\
&C_{27}(\tau^4) = \Bigg[
8\tau^2C_{19}(\tau^2)c^2
+8\tau^2C_{24}
\Bigg],
C_{28}(\tau^4) = \Bigg[
\pth{\frac12+U_\omega}\tau^2C_8(\tau^2)c^2
+8\tau^2C_{22}(\tau^2)
\Bigg],\\
&C_{29}(\tau^2)=\Bigg[
8\tau^2c^2\pth{\frac12+U_\omega}
+8\tau^2C_{18}
\Bigg], C_{30}(\tau^2)=\Bigg[
2\tau^2c^2\pth{\frac12+U_\omega}
+8\tau^2C_{23}
\Bigg].
\end{align*}

Putting the above \eqref{eqn: mt noise final} and \eqref{eqn: Q bound} back to \eqref{eqn: m squared intermediate bound after young} and taking expectation, we get
\begin{align}
 &\bbE_{t-\tau}\|m_{t+1}\|_F^2 
\leq (1-\frac{\lambda\alpha \zeta}{r}+ c_1^2\zeta+c_2^2\zeta) \bbE_{t-\tau}\|m_{t}\|_F^2\nonumber\\
 &+ \pth{\frac{32U_\omega^2\zeta^2}{L^2}+4C_{29}(\tau^2)\frac{\zeta^2}{L}+\frac{144U_\omega^2}{c_2^2}\frac{\zeta}{L^2}}\frac{1}{K}\sum_{k=1}^K \bbE_{t-\tau}\|x_t^k\|^2  \nonumber \\
&+\pth{4C_{30}(\tau^2)\zeta^2+8U_\omega^2\zeta^2+\frac{4U_\omega^2}{c_1^2}\zeta}\frac{1}{K}\sum_{k=1}^K \bbE_{t-\tau}| y_t^k|^2 \nonumber\\
& +4C_{25}(\tau^2)\frac{\zeta^2}{L}
\nonumber
+
4C_{26}(\tau^4)\frac{\zeta^4}{L^2}
\nonumber
+
\pth{4C_{27}(\tau^4)+8\tau^2\frac{U_{\delta,L}^4}{L^4}c^2}\zeta^4
+
4C_{28}(\tau^4)\gamma^2\zeta^2
\nonumber\\ 
&\qquad \qquad +
\pth{32\tau^2C_{20}+8\frac{U_{\delta,L}^2}{L^2}U_\omega^2}\frac{\zeta^2}{K}
+
\pth{32\tau^2C_{21}+16 U_\omega^2 \frac{U_{\delta,L}^2}{L^2} }\zeta^2 m^2\rho^{2\tau}
+
8U_rU_\omega \zeta m\rho^\tau +8U_\omega^2\frac{\zeta}{L} m\rho^\tau.
\end{align}
Recall $M_t = \bbE_{t-\tau}\|m_{t}\|_F^2; \bar X_t = \frac{1}{K}\sum_{k=1}^K \bbE_{t-\tau}\|x_t^k\|^2 ;  \bar Y_t = \frac{1}{K}\sum_{k=1}^K \bbE_{t-\tau}| y_t^k|^2,  $ and choose $\tau = \lceil2 \log_\rho(\zeta)\rceil$, we get,
\begin{align}
\label{eqn: Mt+1 final}
    &M_{t+1} \leq (1-\frac{\lambda\alpha \zeta}{r}+ c_1^2\zeta+c_2^2\zeta) M_t\nonumber\\
 &\qquad+ \pth{\calC_{M,4}(\tau^2)\zeta^2+\frac{144U_\omega^2}{c_2^2}\frac{\zeta}{L^2} + 4\lambda \zeta}\bar X_t 
 +\pth{\calC_{M,5}(\tau^2)\zeta^2+\frac{4U_\omega^2}{c_1^2}\zeta}\bar Y_t \nonumber\\
&\qquad +4C_{25}(\tau^2)\frac{\zeta^2}{L}
+
\calC_{M,1}\zeta^4
+
4C_{28}(\tau^4)\gamma^2\zeta^2
 +
\calC_{M,2}(\tau^2)\frac{\zeta^2}{K}
+\calC_{M,3}
\zeta^6
+
\calC_{M,6}\zeta^3,
\end{align}
where
\begin{align*}
    &\calC_{M,1} = \pth{4C_{27}(\tau^4)+8\tau^2\frac{U_{\delta,L}^4}{L^4}c^2+4C_{26}(\tau^4)\frac{1}{L^2}},
    \calC_{M,2}(\tau^2) = \pth{32\tau^2C_{20}+8\frac{U_{\delta,L}^2}{L^2}U_\omega^2},\\
    &
    \calC_{M,3} =  \pth{32\tau^2C_{21}+16 U_\omega^2 \frac{U_{\delta,L}^2}{L^2} }m^2, \calC_{M,4}(\tau^2)=\qth{\frac{32U_\omega^2}{L^2}+\frac{4C_{29}(\tau^2)}{L}}, \calC_{M,5}(\tau^2)=\qth{4C_{30}(\tau^2) +8U_\omega^2},\\
    &\calC_{M,6}=\qth{8U_rU_\omega m  +8U_\omega^2 m\frac{1}{L}}.
\end{align*}

\newpage 
\section{Proof of Lemma \ref{lm: reward analysis}}
\label{sec: reward eta}

By Eq.\,(\ref{eq: main iterates}) and the update of $\eta^k$ in Algorithm \ref{alg: fedac-per}, we have
\begin{align*}
y_{t+1}^k  =& \eta_{t+1}^k - J^k
= \eta_{t}^k + \gamma\pth{\frac{1}{L}\sum_{\ell=0}^{L-1}r_{t,\ell}^k-\eta_t^k}  - J^k 
= y_t^k + \gamma\pth{\frac{1}{L}\sum_{\ell=0}^{L-1}r_{t,\ell}^k-\eta_t^k} \\
=&y_t^k + \gamma\pth{\frac{1}{L}\sum_{\ell=0}^{L-1}r_{t,\ell}^k-J^k+J^k-\eta_t^k}\\
=& y_t^k +\gamma\pth{\frac{1}{L}\sum_{\ell=0}^{L-1}r_{t,\ell}^k-J^k - y_t^k}\\
=& (1-\gamma) y_t^k + \gamma\pth{\frac{1}{L}\sum_{\ell=0}^{L-1}r_{t,\ell}^k-J^k}\\
=& (1-\gamma)^{t+1} y_0^k + \gamma\sum_{i=0}^{t}(1-\gamma)^{t-i} \pth{\frac{1}{L}\sum_{\ell=0}^{L-1}r_{i,\ell}^k-J^k}.
\end{align*}

Then,
\begin{align*}
    (y_{t}^k)^2 =&  (1-\gamma)^{2t} (y_0^k)^2 + \gamma^2\pth{\sum_{i=0}^{t-1}(1-\gamma)^{t-1-i} \pth{\frac{1}{L}\sum_{\ell=0}^{L-1}r_{i,\ell}^k-J^k}}^2\\
    & + 2(1-\gamma)^{t} y_0^k \gamma\sum_{i=0}^{t-1}(1-\gamma)^{t-1-i} \pth{\frac{1}{L}\sum_{\ell=0}^{L-1}r_{i,\ell}^k-J^k}.
\end{align*}
Taking expectation, we get
\begin{align}
\label{eq: y error decomposition}
\bbE(y_{t}^k)^2 =&  \underbrace{(1-\gamma)^{2t} (y_0^k)^2}_{I_1} + \underbrace{\gamma^2\bbE\pth{\sum_{i=0}^{t-1}(1-\gamma)^{t-1-i} \pth{\frac{1}{L}\sum_{\ell=0}^{L-1}r_{i,\ell}^k-J^k}}^2}_{I_2} \nonumber\\
& + \underbrace{2(1-\gamma)^{t} y_0^k \gamma\bbE\sum_{i=0}^{t-1}(1-\gamma)^{t-1-i} \pth{\frac{1}{L}\sum_{\ell=0}^{L-1}r_{i,\ell}^k-J^k}}_{I_3}.
\end{align}
For $I_1$ in Eq.\,(\ref{eq: y error decomposition}), using the fact that $\ln (1-\gamma)\le -\gamma$ for $\gamma\in (0, 1)$, we have 
\begin{align*}
    I_1 = (1-\gamma)^{2t} (y_0^k)^2 \leq (y_0^k)^2  e^{-2\gamma t}.
\end{align*}
For $I_3$ in Eq.\,(\ref{eq: y error decomposition}),  we have,
\begin{align*}
   & \bbE\sum_{i=0}^{t-1}(1-\gamma)^{t-1-i} \pth{\frac{1}{L}\sum_{\ell=0}^{L-1}r_{i,\ell}^k-J^k}\\
   =&\bbE\sum_{i=0}^{\tau-1}(1-\gamma)^{t-1-i} \pth{\frac{1}{L}\sum_{\ell=0}^{L-1}r_{i,\ell}^k-J^k}+  \bbE\sum_{i=\tau}^{t-1}(1-\gamma)^{t-1-i} \pth{\frac{1}{L}\sum_{\ell=0}^{L-1}r_{i,\ell}^k-J^k}\\
    \leq &2U_r \frac{(1-\gamma)^{t-\tau}(1-(1-\gamma)^\tau)}{1-(1-\gamma)} + \sum_{i=\tau}^{t-1}(1-\gamma)^{t-1-i} \bbE\pth{\frac{1}{L}\sum_{\ell=0}^{L-1}r_{i,\ell}^k-J^k} \\
    \leq & 2U_r \frac{(1-\gamma)^{t-\tau}}{\gamma}
    + 2U_r\sum_{i=\tau}^{t-1}(1-\gamma)^{t-1-i} \sup_{s\in \calS} d_{TV}( P_{0: i}^k(\cdot, \cdot, \cdot, ..., \cdot \mid s_0^k = s), \mu^k\otimes (\pi\otimes P^k)^{L-1}\otimes \pi)\\
    \leq & 2U_r \frac{(1-\gamma)^{t-\tau}}{\gamma} + 2U_r \sum_{i=\tau}^{t-1}(1-\gamma)^{t-1-i} m\rho^{i} \\
    \le & 2U_r \frac{(1-\gamma)^{t-\tau}}{\gamma}+\frac{ 2U_rm\rho^{\tau}}{\gamma}. 
\end{align*}
where the first inequality follows because $\abth{\frac{1}{L}\sum_{\ell=0}^{L-1}r_{i,\ell}^k-J^k}\le 2U_r$,
and the third inequality follows from Assumption \ref{assp: uniform ergodicity}. 
In the second inequality, we slightly abuse notation to use $P_{0: i}^k(\cdot, \cdot, \cdot, ..., \cdot \mid s_0^k = s)$ to denote the $L$-block state-action pair at time $t$. It holds because the mixing of the $L$-block is no slower than the original.  
Thus, 
\begin{align*}
    I_3
    \leq &2(1-\gamma)^{t} |y_0^k| \pth{2U_r(1-\gamma)^{t-\tau}+2U_rm\rho^{\tau}}\\
    \leq & 4 |y_0^k| U_r(1-\gamma)^{2t-\tau} + 4U_r(1-\gamma)^{t} |y_0^k|m\rho^{\tau}.
\end{align*}

Lastly, for $I_2$ in Eq.\,(\ref{eq: y error decomposition}), we have 
\begin{align}
\label{eq: y error further decomposition}
    &\bbE\pth{\sum_{i=0}^{t-1}(1-\gamma)^{t-1-i} \pth{\frac{1}{L}\sum_{\ell=0}^{L-1}r_{i,\ell}^k-J^k}}^2\nonumber\\
    =&\bbE \sum_{i=0}^{t-1}(1-\gamma)^{2t-2-2i} \pth{\frac{1}{L}\sum_{\ell=0}^{L-1}r_{i,\ell}^k-J^k}^2 \nonumber\\
    & + 2\bbE \sum_{i=0}^{t-1}\sum_{j=i+1}^{t-1}(1-\gamma)^{t-1-i} (1-\gamma)^{t-1-j} \pth{\frac{1}{L}\sum_{\ell=0}^{L-1}r_{i,\ell}^k-J^k}\pth{\frac{1}{L}\sum_{\ell=0}^{L-1}r_{j,\ell}^k-J^k} \nonumber\\
    \leq& \frac{1-(1-\gamma)^{2t}}{1-(1-\gamma)^2} 4U_r^2 + \underbrace{2\bbE \sum_{i=0}^{t-\tau-1}\sum_{j=i+\tau}^{t-1}(1-\gamma)^{t-1-i} (1-\gamma)^{t-1-j} \pth{\frac{1}{L}\sum_{\ell=0}^{L-1}r_{i,\ell}^k-J^k}\pth{\frac{1}{L}\sum_{\ell=0}^{L-1}r_{j,\ell}^k-J^k}}_{I_{21}} \nonumber\\
    &+ \underbrace{2\bbE \sum_{i=0}^{t-\tau-1}\sum_{j=i+1}^{i+\tau-1}(1-\gamma)^{t-1-i} (1-\gamma)^{t-1-j} \pth{\frac{1}{L}\sum_{\ell=0}^{L-1}r_{i,\ell}^k-J^k}\pth{\frac{1}{L}\sum_{\ell=0}^{L-1}r_{j,\ell}^k-J^k}}_{I_{22}} \nonumber\\
    &+ \underbrace{2\bbE \sum_{i=t-\tau}^{t-1}\sum_{j=i+1}^{t-1}(1-\gamma)^{t-1-i} (1-\gamma)^{t-1-j} \pth{\frac{1}{L}\sum_{\ell=0}^{L-1}r_{i,\ell}^k-J^k}\pth{\frac{1}{L}\sum_{\ell=0}^{L-1}r_{j,\ell}^k-J^k}}_{I_{23}}. 
\end{align}

Recall $v_{0:t}^k$ is the Markovian trajectory at agent $k$:  
\begin{align}
\label{eq: Markovian traectories}
v_{0:t}^k = (s_0^k,a_0^k,s_1^k,a_1^k,..., s_{t-\tau+1}^k,a_{t-\tau+1}^k,...,s_t^k,a_t^k,s_{t+1}^k). \end{align}

For $I_{21}$ in Eq.\,(\ref{eq: y error further decomposition}), we have
\begin{align*}
    I_{21} %
    = & 2\sum_{i=0}^{t-\tau-1}\sum_{j=i+\tau}^{t-1}(1-\gamma)^{t-1-i} (1-\gamma)^{t-1-j} \bbE\qth{\bbE \qth{\pth{\frac{1}{L}\sum_{\ell=0}^{L-1}r_{i,\ell}^k-J^k}\pth{\frac{1}{L}\sum_{\ell=0}^{L-1}r_{j,\ell}^k-J^k}|v_{0:i}^k}}\\
    \leq& 2\sum_{i=0}^{t-\tau-1}\sum_{j=i+\tau}^{t-1}(1-\gamma)^{t-1-i} (1-\gamma)^{t-1-j} 2U_r\bbE\qth{\abth{\bbE \qth{\pth{\frac{1}{L}\sum_{\ell=0}^{L-1}r_{j,\ell}^k-J^k}|v_{0:i}^k}}}\\
    \leq& 2\sum_{i=0}^{t-\tau-1}\sum_{j=i+\tau}^{t-1}(1-\gamma)^{t-1-i} (1-\gamma)^{t-1-j} 4U^2_r m\rho^{\tau}\\
    \leq& 8U^2_r m\rho^{\tau}\sum_{i=0}^{t-\tau-1}\sum_{j=i+\tau}^{t-1}(1-\gamma)^{t-1-i} (1-\gamma)^{t-1-j}\\
    \leq& 8U^2_r m\rho^{\tau}\sum_{i=0}^{t-\tau-1}(1-\gamma)^{t-1-i}\frac{1-(1-\gamma)^{t-i-\tau}}{\gamma}\\
    \leq& 8U^2_r m\rho^{\tau}\frac{(1-\gamma)^\tau}{\gamma^2}.
\end{align*}
Noting that 
\begin{align*}
    &\sum_{i=0}^{t-\tau-1}\sum_{j=i+1}^{i+\tau-1}(1-\gamma)^{t-1-i} (1-\gamma)^{t-1-j} \\
    =&\sum_{i=0}^{t-\tau-1}(1-\gamma)^{t-1-i} \pth{\sum_{j=i+1}^{i+\tau - 1} (1-\gamma)^{t-1-j}}\\
    =& \sum_{i=0}^{t-\tau-1}(1-\gamma)^{t-1-i} \frac{(1-\gamma)^{t-i-\tau}(1-(1-\gamma)^\tau)}{\gamma}\\
    =& \frac{(1-(1-\gamma)^\tau)}{\gamma}\sum_{i=0}^{t-\tau-1}(1-\gamma)^{2t-1-2i-\tau} .
\end{align*}
Since $(1-\gamma)^{\tau}\ge 1 - \tau \gamma$, we have  
\begin{align*}
&\frac{(1-(1-\gamma)^\tau)}{\gamma} 
\le \frac{1 - (1 - \tau\gamma)}{\gamma} 
= \tau. 
\end{align*}
In addition,  
\begin{align*}
\sum_{i=0}^{t-\tau-1}(1-\gamma)^{2t-1-2i-\tau} 
\leq \frac{(1-\gamma)^{\tau+1}}{1-(1-\gamma)^2}
\leq \frac{(1-\gamma)^\tau}{\gamma}.
\end{align*}
Thus, $I_{22}$ in Eq.\,(\ref{eq: y error further decomposition}) can be bounded as 
\begin{align*}
    I_{22}\leq 8U_r^2 \tau \frac{(1-\gamma)^\tau}{\gamma}.
\end{align*}
For $I_{23}$ in Eq.\,(\ref{eq: y error further decomposition}), we have 
\begin{align*}
    I_{23}& %
    \leq 2 \sum_{i=t-\tau}^{t-1}\sum_{j=i+1}^{t-1} 4U_r^2 
    \leq8U_r^2 \frac{\tau(\tau-1)}{2}.
\end{align*}
Now, we can bound $I_2$ in Eq.\,(\ref{eq: y error decomposition}) as 
\begin{align*}
    &I_2\leq \gamma^2\pth{\frac{1-(1-\gamma)^{2t}}{1-(1-\gamma)^2} 4U_r^2+8U^2_r m\rho^{\tau}\frac{(1-\gamma)^\tau}{\gamma^2}+8U_r^2 \tau\frac{(1-\gamma)^\tau}{\gamma}+8U_r^2 \frac{\tau(\tau-1)}{2}}\\
    \leq &\frac{\gamma}{2-\gamma} 4U_r^2+8U_r m\rho^{\tau}(1-\gamma)^\tau+8U_r^2 \tau\gamma(1-\gamma)^\tau+8U_r^2 \frac{\tau(\tau-1)}{2}\gamma^2.
\end{align*}
{ Putting $I_1,I_2$ and $I_3$ together, and choosing $\gamma = \frac{\log(T)}{2T}$, we get  }
\begin{align*}
    &\bbE(y_{T}^k)^2 
    \leq  (y_0^k)^2  e^{-2\gamma T}+  \frac{\gamma}{2-\gamma} 4U_r^2+8U^2_r m\rho^{\tau}(1-\gamma)^\tau+8U_r^2 \tau \gamma(1-\gamma)^\tau\\
    &+8U_r^2 \frac{\tau(\tau-1)}{2}\gamma^2 + 4 |y_0^k| U_r e^{\gamma \tau}e^{-\gamma 2T} + 4U_r(1-\gamma)^{T} |y_0^k| m\rho^{\tau}\\
     & \leq 4U_r^2 \frac{1}{T}+ \frac{4U_r^2 \log(T)}{2T} + 8U_r m\rho^{\tau} + 8U_r^2\tau\frac{\log(T)}{2T}\\
    &+ 8U_r^2 \frac{\tau(\tau-1)}{2}\frac{\log^2(T)}{4T^2} + 8 U_r^2 e^{\gamma \tau}\frac{1}{T}+ 8U_r\frac{1}{\sqrt{T}} m\rho^{\tau}. 
\end{align*}

When $\tau \le \min\{\frac{2T}{\log(T)},\lceil\log_\rho(\gamma)\rceil, \lceil2\log_\rho(\zeta)\rceil\}$, it holds that $e^{\gamma \tau} \le 3$ and that 
$8U_r^2 \frac{\tau(\tau-1)}{2}\frac{\log^2(T)}{4T^2} \le \frac{2\tau U_r^2 \log(T)}{T}$, and $\rho^\tau\leq \gamma$ . 
Thus, the  upper bound of $\bbE(y_{T}^k)^2 $ can be simplified as 
\begin{align}
\label{eqn:Y_T^k final}
\bbE(y_{T}^k)^2 
\le \frac{28 U_r^2}{T} + \frac{8U_r^2 \tau \log(T)}{T} + 16U_r m\rho^{\tau}
\leq \frac{28 U_r^2}{T} + \frac{8U_r^2 \tau \log(T)}{T} + \frac{8U_r m\log(T)}{T}, 
\end{align}
which trivially implies that 
\begin{align*}
   \bar Y_T = \frac{1}{K}\sum_{k=1}^KY_T^k \leq \frac{28 U_r^2}{T} + \frac{8U_r^2 \tau \log(T)}{T} + \frac{8U_r m\log(T)}{T}. 
\end{align*}

\newpage

\section{Proof of Theorem \ref{thm: 1}}
For ease of expostion, let $\bar Y_T = \frac{1}{K}\sum_{k=1}^KY_T^k$. 
From Lemma \ref{lm: reward analysis} (\eqref{eqn:Y_T^k final}), Lemma \ref{lm: local head: upper bound: formal}, Lemma \ref{lm: local head: lower bound}, and 
Lemma \ref{lm: PAD analysis} (\eqref{eqn: Mt+1 final}), we get, for $T\geq 2\tau$,
\begin{align*}
    &\bar Y_T\leq \frac{28 U_r^2}{T} + \frac{8U_r^2 \tau \log(T)}{T} + \frac{8U_r m\log(T)}{T},\\
    &\bar X_{T+1}  \nonumber
   \leq \Bigg(1-2\lambda c\zeta
   +\Big(c_3^2\frac{c }{L} 
+ c_4^2 c  
+ c_5^2\frac{1}{L}
+c_6^2 +4U_\omega^4\frac{1}{c_5^2L}\Big)\zeta   
+\calC_{X,1}(\tau^2)\zeta^2
+ \calC_{X,2}(\tau^2)\zeta^3
\Bigg)\bar X_T\nonumber\\
&\qquad+ \pth{\frac{36U_\omega^2}{c_3^2}\frac{c\zeta}{L}  +6U_\omega U_\delta c \zeta } M_T
\nonumber\\
&\qquad+ \Bigg(
\Big[
\frac{1}{c_4^2}c 
+\frac{U_\omega^4}{c_6^2}
\Big]\zeta
+\calC_{X,3}(\tau^2)\zeta^2
+\calC_{X,4}(\tau^2) \zeta^3
\Bigg)\bar Y_T\nonumber\\
&\qquad+\calC_{X,5}(\tau^2)\frac{\zeta^2}{L}
 +\calC_{X,6}(\tau^2)\frac{\zeta^2}{K} 
+\calC_{X,7}\zeta^3
 +\calC_{X,8}\zeta^4
 + \calC_{X,9} \gamma^2\zeta^2 
+
\calC_{X,10}\zeta^5
+
\calC_{X,11}\gamma^2\zeta^3
+\calC_{X,12}\zeta^7
+ \calC_{X,13}\zeta^6,\\
    &M_{T+1} \leq (1-\frac{\lambda\alpha \zeta}{r}+ c_1^2\zeta+c_2^2\zeta) M_T\nonumber\\
 &\qquad+ \pth{\calC_{M,4}(\tau^2)\zeta^2+\frac{16U_\omega^2}{c_2^2}\frac{\zeta}{L^2} }\bar X_T 
 +\pth{\calC_{M,5}(\tau^2)\zeta^2+\frac{4U_\omega^2}{c_1^2}\zeta}\bar Y_T \nonumber\\
&\qquad +4C_{25}(\tau^2)\frac{\zeta^2}{L}
+
\calC_{M,1}\zeta^4
+
4C_{28}(\tau^4)\gamma^2\zeta^2
 +
\calC_{M,2}(\tau^2)\frac{\zeta^2}{K}
+\calC_{M,3}
\zeta^6
+
\calC_{M,6}\zeta^3.
\end{align*}
Note that $\bar Y_T = \calO(\frac{\log(T)}{T})$, and we collect all terms of order $\zeta^3$ and higher into $\calO(\zeta^3)$, then the system can be simplified to,
\begin{align*}
  &\bar X_{T+1}  \nonumber
   \leq \Bigg(1-2\lambda c\zeta
   +\Big(c_3^2\frac{c }{L} 
+ c_4^2 c  
+ c_5^2\frac{1}{L}
+c_6^2 +4U_\omega^4\frac{1}{c_5^2L}\Big)\zeta   
+\calC_{X,1}(\tau^2)\zeta^2
+ \calC_{X,2}(\tau^2)\zeta^3
\Bigg)\bar X_T\nonumber\\
&\qquad+ \pth{\frac{36U_\omega^2}{c_3^2}\frac{c\zeta}{L}  +6U_\omega U_\delta c \zeta } M_T
\nonumber
+\calC_{X,5}(\tau^2)\frac{\zeta^2}{L}
 +\calC_{X,6}(\tau^2)\frac{\zeta^2}{K} 
 + \calO(\frac{\zeta\log(T)}{T})
+\calO(\zeta^3),\\
    &M_{T+1} \leq (1-\frac{\lambda\alpha \zeta}{r}+ c_1^2\zeta+c_2^2\zeta) M_T\nonumber\\
 &\qquad+ \pth{\calC_{M,4}(\tau^2)\zeta^2+\frac{16U_\omega^2}{c_2^2}\frac{\zeta}{L^2} } 
  +4C_{25}(\tau^2)\frac{\zeta^2}{L}
+
\calC_{M,2}(\tau^2)\frac{\zeta^2}{K}
+
\calO(\zeta^3)+   \calO(\frac{\zeta\log(T)}{T}).
\end{align*}
Choosing $c_3=\sqrt{\frac{\lambda L}{30}}, c_4 = \sqrt{\frac{\lambda}{30}}, c_5=1, c_6 = \min\{\sqrt{\frac{1}{L}}, \sqrt{\frac{4U_\omega^4}{L}}\}$,  and
\begin{align}
    \label{eqn: c feasible 1}
    c\geq \max\{\frac{30}{\lambda L}, \frac{120 U_\omega^4}{\lambda L}\},
\end{align}
and requiring $\zeta\leq \frac{\lambda c}{6\calC_{X,1}(\tau^2)}, \zeta^2 \leq \frac{\lambda c}{6\calC_{X,2}(\tau^2)}$,
we get 
\begin{align*}
    &\bar X_{T+1}
    \leq \pth{1-\frac{3}{2}\lambda c \zeta}\bar X_{T} + \pth{\frac{36U_\omega^2}{c_3^2}\frac{c\zeta}{L}  +6U_\omega U_\delta c \zeta } M_T
\nonumber
+\calC_{X,5}(\tau^2)\frac{\zeta^2}{L}
 +\calC_{X,6}(\tau^2)\frac{\zeta^2}{K} 
 + \calO(\frac{\zeta\log(T)}{T})
+\calO(\zeta^3).
\end{align*}
We choose $c_1 = c_2 = \sqrt{\frac{\lambda \alpha}{4r}}$, which requires $\calC_{M,4}(\tau^2)\zeta \leq \frac{16U_\omega^2}{c_2^2 L^2}$, we have 
\begin{align*}
    &M_{T+1}\leq \pth{1-\frac{\lambda\alpha  \zeta}{2r}} M_T+\frac{32U_\omega^2}{c_2^2}\frac{\zeta}{L^2} \bar X_T 
   +4C_{25}(\tau^2)\frac{\zeta^2}{L}
+
\calC_{M,2}(\tau^2)\frac{\zeta^2}{K}
+
\calO(\zeta^3)+   \calO(\frac{\zeta\log(T)}{T}).
\end{align*}

Substituting $c_2$ and $c_3$, we get,
\begin{align*}
 &\bar X_{T+1}
    \leq \pth{1-\frac{3}{2}\lambda c \zeta}\bar X_{T} + \pth{\frac{1080U_\omega^2}{\lambda L^2} c\zeta  +6U_\omega U_\delta c \zeta } M_T
\nonumber
+\calC_{X,5}(\tau^2)\frac{\zeta^2}{L}
 +\calC_{X,6}(\tau^2)\frac{\zeta^2}{K} 
 + \calO(\zeta\frac{\log(T)}{T})
+\calO(\zeta^3).\\
&M_{T+1}\leq \pth{1-\frac{\lambda\alpha  \zeta}{2r}} M_T+\frac{128rU_\omega^2}{\lambda \alpha}\frac{\zeta}{L^2}  \bar X_T 
   +4C_{25}(\tau^2)\frac{\zeta^2}{L}
+
\calC_{M,2}(\tau^2)\frac{\zeta^2}{K}
+
\calO(\zeta^3)+   \calO(\zeta\frac{\log(T)}{T}).
\end{align*}

Defining the weighted Lyapunov function as $\calV_{T+1} =\bar X_{T+1}+\kappa M_{T+1}$, we have, 
\begin{align*}
    \calV_{T+1} \leq& \pth{1-\frac{3}{2}\lambda c \zeta+\kappa\frac{128rU_\omega^2}{\lambda \alpha}\frac{\zeta}{L^2}}\bar X_{T}  
+\calC_{X,5}(\tau^2)\frac{\zeta^2}{L}
 +\calC_{X,6}(\tau^2)\frac{\zeta^2}{K} 
 + \calO(\zeta\frac{\log(T)}{T})
+\calO(\zeta^3)\\
&+\kappa\pth{\pth{1-\frac{\lambda\alpha  \zeta}{2r} +\frac{1}{\kappa}\pth{\frac{1080U_\omega^2}{\lambda L^2} c\zeta  +6U_\omega U_\delta c \zeta }} M_T
   +4C_{25}(\tau^2)\frac{\zeta^2}{L}
+
\calC_{M,2}(\tau^2)\frac{\zeta^2}{K}
+
\calO(\zeta^3)+   \calO(\zeta\frac{\log(T)}{T})}
\end{align*}

 For sufficiently large $L$, $c$ satisfies \eqref{eqn: c feasible 1} and the following condition
\begin{align}
\label{eqn: c feasible 2}
    \kappa\frac{256r U_\omega^2}{\lambda^2\alpha L^2}\leq c\leq \kappa\frac{\lambda^2\alpha L^2}{4r \pth{1080U_\omega^2+6U_\omega U_\delta\lambda L^2}}
\end{align}
such that
\begin{align*}
    &\frac{128rU_\omega^2}{\lambda \alpha L^2}\kappa \leq \frac{1}{2}\lambda c,\\
    &\frac{1}{\kappa}\pth{\frac{1080U_\omega^2}{\lambda L^2} c +6U_\omega U_\delta c  }\leq \frac{\lambda\alpha}{4r},
\end{align*}
then,
\begin{align}
\label{eqn: Lyap}
    \calV_{T+1} \leq& \pth{1-\lambda c \zeta}\bar X_{T}  
+\calC_{X,5}(\tau^2)\frac{\zeta^2}{L}
 +\calC_{X,6}(\tau^2)\frac{\zeta^2}{K} 
 + \calO(\zeta\frac{\log(T)}{T})
+\calO(\zeta^3)\nonumber\\
&+\kappa\pth{\pth{1-\frac{\lambda\alpha  \zeta}{4r} } M_T
   +4C_{25}(\tau^2)\frac{\zeta^2}{L}
+
\calC_{M,2}(\tau^2)\frac{\zeta^2}{K}
+
\calO(\zeta^3)+   \calO(\zeta\frac{\log(T)}{T})}
\end{align}
Choosing $\kappa \geq \frac{2(1080U_\omega^2+6U_\omega U_\delta \lambda L^2)}{\lambda^2  L^2} = \Theta(1)$, 
the upper bound of $c$ in \eqref{eqn: c feasible 2} satisfies,
\begin{align*}
    \kappa\frac{\lambda^2\alpha L^2}{4r \pth{1080U_\omega^2+6U_\omega U_\delta\lambda L^2}}\geq \frac{\alpha}{2r},
\end{align*}
leaving a room to further constrain $c$ to be greater than $\frac{\alpha}{4r}$ and choose $L=K$, such that
 \eqref{eqn: Lyap} can be written as,
\begin{align}
\label{eqn: Lyap 2}
    \calV_{T+1} \leq& \pth{1-\frac{\lambda\alpha  \zeta}{4r}}\bar X_{T}  
+\calC_{X,5}(\tau^2)\frac{\zeta^2}{L}
 +\calC_{X,6}(\tau^2)\frac{\zeta^2}{K} 
 + \calO(\zeta\frac{\log(T)}{T})
+\calO(\zeta^3)\nonumber\\
&+\kappa\pth{1-\frac{\lambda\alpha  \zeta}{4r} } M_T
+4\kappa C_{25}(\tau^2)\frac{\zeta^2}{L}
+\kappa\calC_{M,2}(\tau^2)\frac{\zeta^2}{K}
+\kappa\calO(\zeta^3)+\kappa\calO(\zeta\frac{\log(T)}{T})\nonumber \\
\leq& \pth{1-\frac{\lambda\alpha  \zeta}{4r}}  \calV_{T}  + \calC_{V}(\tau^2)\frac{\zeta^2}{K}  + \calO(\zeta\frac{\log(T)}{T})
+\calO(\zeta^3) ,
\end{align}
where 
\begin{align*}
    \calC_{V}(\tau^2) = \pth{\calC_{X,5}(\tau^2)+\calC_{X,6}(\tau^2)+4\kappa C_{25}(\tau^2)+\kappa\calC_{M,2}(\tau^2)}.
\end{align*}
Unrolling \eqref{eqn: Lyap 2}, we get
\begin{align*}
    \calV_{T+1} \leq&\pth{1-\frac{\lambda\alpha  \zeta}{4r}}  \calV_{T}+ \calC_{V}(\tau^2)\frac{\zeta^2}{K}  + \calO(\zeta\frac{\log(T)}{T}) +\calO(\zeta^3)\\
    \leq& \pth{1-\frac{\lambda\alpha  \zeta}{4r}}^{T-2\tau+1} \calV_{2\tau} + \sum_{i=0}^{T-2\tau} \pth{1-\frac{\lambda\alpha  \zeta}{4r}}^i \qth{\calC_{V}(\tau^2)\frac{\zeta^2}{K}  + \calO(\zeta\frac{\log(T)}{T}) +\calO(\zeta^3)}\\
    \leq& \exp\sth{-\frac{\lambda\alpha  \zeta(T-2\tau+1)}{4r}}+ \frac{4r}{\lambda\alpha}\qth{\calC_{V}(\tau^2)\frac{\zeta}{K}  + \calO(\frac{\log(T)}{T}) +\calO(\zeta^2)}.
\end{align*}
{Choosing $\zeta = \frac{4r}{\sqrt{T}\lambda\alpha}$, we have,
\begin{align*}
    \calV_{T+1} \leq \exp\sth{-\sqrt{T}/2} + \frac{16r^2}{\lambda^2\alpha^2}\calC_{V}(\tau^2)\frac{1}{\sqrt{T}K} + \calO(\frac{\log(T)}{T}).
\end{align*}
}
{Lastly, collecting the sufficient conditions on $c$ and $K$, we can always choose \begin{align*}
    \max\sth{\frac{30}{\lambda K}, \frac{120U_\omega^4}{\lambda K}} \leq \frac{\alpha}{2r}; \quad 
    \frac{\alpha}{4r}
    \leq c\leq\frac{\alpha}{2r}.
\end{align*}}

\newpage 
\section{Supporting Propositions and Lemmas}
\label{sec: supporting}

\subsection{Boundedness}
\label{subsec: supporting: boundedness}

\begin{lemma}
\label{lm: eta: bound}
    For any $t\geq 0, L \geq 0, k \in [K]$, it holds that $\eta_t^k \in [-U_r,U_r]$. %
\end{lemma}    
\begin{proof}
        We prove Lemma \ref{lm: eta: bound} by induction.
        The algorithm initializes $\eta_0^k = 0$.  
        Suppose that $\eta_{t^{\prime}}^k \in [-U_r,U_r]$ for $0\leq t^{\prime}\leq t-1$ and for any $k\in [K]$.
        We have 
        \begin{align*}
            \eta_t^k = \eta_{t-1}^k+\gamma \left( \frac{1}{L} \sum_{\ell=0}^{L-1} r_{t-1,\ell}^k - \eta^k_{t-1} \right)
            =(1-\gamma)\eta_{t-1}^k+ \frac{\gamma}{L} \sum_{\ell=0}^{L-1} r_{t-1,\ell}^k 
            \leq (1-\gamma)U_r+\frac{\gamma}{L} L\cdot U_r
            =U_r.
        \end{align*}
Similarly, it can be shown that $\eta_t^k \ge - U_r$. 

\end{proof}

\begin{lemma}
\label{lmm: U delta}
Suppose that Assumption \ref{ass: bounded features} holds. 
Define $U_{\delta,L}:=2 L U_r + 2U_\omega.$
For any $t\geq 0, L\geq0, k \in [K]$, it holds that 
\begin{align}
\label{eq: bounded of delta and Markovian noise 1}
|\delta^k_{t,L}| \leq U_{\delta,L}, ~~~~  \text{and} ~~~  
\|\bfb_{t,L}^k\| \leq U_{\delta,L}. 
\end{align}     
\end{lemma}
\begin{proof}
From Eq.\,(\ref{eq: L step TD error}), 
    \begin{align*}
        |\delta^k_{t,L}| &= |\sum_{\ell=0}^{L-1} (r_{t,\ell}^k - \eta_t^k) + (\phi(s_{t,L}^k) - \phi(s_{t,0}^k))^\top \mathbf{B}_t \omega_t^k| \\
        &\leq |\sum_{\ell=0}^{L-1} (r_{t,\ell}^k - \eta_t^k) |+| (\phi(s_{t,L}^k) - \phi(s_{t,0}^k))^\top \mathbf{B}_t \omega_t^k|\\
        &\leq 2 L U_r + 2U_\omega, %
    \end{align*}
    where the last inequality holds due to Lemma \ref{lm: eta: bound} and $\|\phi(s^k_{t+1}) - \phi(s^k_{t})\|\le 2$.

    Recall from \eqref{def: Markov drift and noise} that 
\begin{align*}
\bfb_{t,L}^k = \sum_{\ell=0}^{L-1}\pth{ (r_{t,\ell}^k - J^k) + (\phi(s_{t,\ell+1}^k) - \phi(s_{t,\ell}^k))^\top \mathbf{B}^*\omega^{k,*}}\phi(s_{t,0}^k).    
\end{align*}
So, 
\begin{align*}
\| \bfb_{t,L}^k\| 
& = \|\sum_{\ell=0}^{L-1}\pth{ (r_{t,\ell}^k - J^k) + (\phi(s_{t,\ell+1}^k) - \phi(s_{t,\ell}^k))^\top \mathbf{B}^*\omega^{k,*}}\phi(s_{t,0}^k)\| \\
& \le \| \sum_{\ell=0}^{L-1} (r_{t,\ell}^k - J^k)\| +\| (\phi(s_{t,L}^k) - \phi(s_{t,0}^k))^\top \mathbf{B}^*\omega^{k,*}\| \\
& \le 2LU_r + 2U_{\omega} 
= U_{\delta,L}.
\end{align*}
\end{proof}

\section{Other Supporting Results}
\label{app: other supporting}

\begin{proof}[Proof of Proposition \ref{prop: key intermediate}]

It follows from Eq.\,(\ref{eq: L step TD error}) that for any $t$ and $L$: 
\begin{align*}
    &\delta_{t,L}^k\phi(s_{t,0}^k) = \pth{\sum_{\ell=0}^{L-1} (r_{t,\ell}^k - \eta_t^k) + (\phi(s_{t,L}^k) - \phi(s_{t,0}^k))^\top \mathbf{B}_t \omega_t^k}\phi(s_{t,0}^k)\\
    =& \sum_{\ell=0}^{L-1}\pth{ (r_{t,\ell}^k - J^k) + (\phi(s_{t,\ell+1}^k) - \phi(s_{t,\ell}^k))^\top \mathbf{B}_t \omega_t^k}\phi(s_{t,0}^k)
    +\pth{\sum_{\ell=0}^{L-1} (J^k-\eta_t^k) }\phi(s_{t,0}^k)\\
    =& \underbrace{\sum_{\ell=0}^{L-1}\phi(s_{t,0}^k)(\phi(s_{t,\ell+1}^k) - \phi(s_{t,\ell}^k))^\top}_{\tilde A_{t,L}^k} x_t^k \quad \quad ~~~ \text{(recall that $x_t^k = \mathbf{B}_t \omega_t^k - \mathbf{B}^*\omega^{k,*}$)}\\
    &+\underbrace{\sum_{\ell=0}^{L-1}\pth{ (r_{t,\ell}^k - J^k) + (\phi(s_{t,\ell+1}^k) - \phi(s_{t,\ell}^k))^\top \mathbf{B}^*\omega^{k,*}}\phi(s_{t,0}^k)}_{\bfb_{t,L}^k}
    +L y_t^k\phi(s_{t,0}^k)\\
    =& \tilde A_{t,L}^kx_t^k+\bfb_{t,L}^k+L y_t^k\phi(s_{t,0}^k),
\end{align*}
proving Eq.\,(\ref{eq: delta-phi: rewritting 1}). 
In the remaining proof, we focus on proving the second decomposition in Eq.\,(\ref{eq: delta-phi: rewritting 2}). 
We have,
\begin{align*}
    &\delta_{t,L}^k\phi(s_{t,0}^k) = \pth{\sum_{\ell=0}^{L-1} (r_{t,\ell}^k - \eta_t^k) + (\phi(s_{t,L}^k) - \phi(s_{t,0}^k))^\top \mathbf{B}_t \omega_t^k}\phi(s_{t,0}^k)\\
    =& \sum_{\ell=0}^{L-1}\pth{ (r_{t,\ell}^k - J^k) + (\phi(s_{t,\ell+1}^k) - \phi(s_{t,\ell}^k))^\top \mathbf{B}_t \omega_t^k}\phi(s_{t,0}^k)
    +\pth{\sum_{\ell=0}^{L-1} (J^k-\eta_t^k) }\phi(s_{t,0}^k)\\
    =& \sum_{\ell=0}^{L-1} \pth{(r_{t,\ell}^k - J^k) + (\phi(s_{t,\ell+1}^k) - \phi(s_{t,\ell}^k))^\top \mathbf{B}_t \omega_t^k}\phi(s_{t,0}^k)\\
    &-\sum_{\ell=0}^{L-1}\bbE_{s^{(0)}\sim\mu^k, (a^{(\ell)}, s^{(\ell)}) \sim \pi \otimes P^k ~ \forall \ell\in [0, L-1]} \pth{ (R(s^{(\ell)},a^{(\ell)}) - J^k) + (\phi(s^{(\ell+1)}) - \phi(s^{(\ell)}))^\top \mathbf{B}_t \omega_t^k}\phi(s^{(0)})\\
    &+\sum_{\ell=0}^{L-1}\bbE_{s^{(0)}\sim\mu^k, (a^{(\ell)}, s^{(\ell)}) \sim \pi \otimes P^k ~ \forall \ell\in [0, L-1]} \pth{ (R(s^{(\ell)},a^{(\ell)}) - J^k) + (\phi(s^{(\ell+1)}) - \phi(s^{(\ell)}))^\top \mathbf{B}_t \omega_t^k}\phi(s^{(0)})\\
    &+L(J^k-\eta_t^k) \phi(s_{t,0}^k).
\end{align*}

Under steady-state distribution, by the Bellman equation, we have 
\begin{align}
\label{eqn: L-step fixed point}
&\bbE_{s^{(0)}\sim\mu^k, (a^{(\ell)}, s^{(\ell)}) \sim \pi \otimes P^k ~ \forall \ell\in [0, L-1]} \pth{ (R(s^{(\ell)},a^{(\ell)}) - J^k) + (\phi(s^{(\ell+1)}) - \phi(s^{(\ell)}))^\top z^{k,*}}\phi(s^{(0)})=0. 
\end{align}
Then,
\begin{align*}
&\sum_{\ell=0}^{L-1}\bbE_{s^{(0)}\sim\mu^k, (a^{(\ell)}, s^{(\ell)}) \sim \pi \otimes P^k ~ \forall \ell\in [0, L-1]} \pth{ (R(s^{(\ell)},a^{(\ell)}) - J^k) + (\phi(s^{(\ell+1)}) - \phi(s^{(\ell)}))^\top \mathbf{B}_t \omega_t^k}\phi(s^{(0)})\\
    =&\sum_{\ell=0}^{L-1}\bbE_{s^{(0)}\sim\mu^k, (a^{(\ell)}, s^{(\ell)}) \sim \pi \otimes P^k ~ \forall \ell\in [0, L-1]} \pth{ (R(s^{(\ell)},a^{(\ell)}) - J^k) + (\phi(s^{(\ell+1)}) - \phi(s^{(\ell)}))^\top \mathbf{B}_t \omega_t^k}\phi(s^{(0)})\\
    & - \sum_{\ell=0}^{L-1}\bbE_{s^{(0)}\sim\mu^k, (a^{(\ell)}, s^{(\ell)}) \sim \pi \otimes P^k ~ \forall \ell\in [0, L-1]} \pth{ (R(s^{(\ell)},a^{(\ell)}) - J^k) + (\phi(s^{(\ell+1)}) - \phi(s^{(\ell)}))^\top z^{k,*}}\phi(s^{(0)})\\
    =&\sum_{\ell=0}^{L-1}\bbE_{s^{(0)}\sim\mu^k, (a^{(\ell)}, s^{(\ell)}) \sim \pi \otimes P^k ~ \forall \ell\in [0, L-1]} (\phi(s^{(\ell+1)}) - \phi(s^{(\ell)}))^\top x_t^k\phi(s^{(0)})\\
    =& \bbE_{s^{(0)}\sim\mu^k, (a^{(\ell)}, s^{(\ell)}) \sim \pi \otimes P^k ~ \forall \ell\in [0, L-1]}\qth{  (\phi(s^{(L)}) - \phi(s^{(0)}))^\top x_t^k\phi(s^{(0)})}\\
    =&\underbrace{\bbE_{s^{(0)}\sim\mu^k, (a^{(\ell)}, s^{(\ell)}) \sim \pi \otimes P^k ~ \forall \ell\in [0, L-1]}\qth{\phi(s^{(0)})(\phi(s^{(L)}) - \phi(s^{(0)}))^\top}}_{A_L^k}x_t^k\\
    =& A_L^k x_t^k.
\end{align*}
Then, 
\begin{align*}
    \delta_{t,L}^k\phi(s_{t,0}^k) =  \xi_{t,L}^k + L y_t^k \phi(s_{t,0}^k) + A_L^{k}x_t^k,
\end{align*}
where
\begin{align*}
     \xi_{t,L}^k  =& \sum_{\ell=0}^{L-1} \pth{(r_{t,\ell}^k - J^k) + (\phi(s_{t,\ell+1}^k) - \phi(s_{t,\ell}^k))^\top \mathbf{B}_t \omega_t^k}\phi(s_{t,0}^k)\\
    &-\bbE_{s^{(0)}\sim\mu^k, (a^{(\ell)}, s^{(\ell)}) \sim \pi \otimes P^k ~ \forall \ell\in [0, L-1]}\pth{ (r(s^{(\ell)},a^{(\ell)}) - J^k) + (\phi(s^{(\ell+1)}) - \phi(s^{(\ell)}))^\top \mathbf{B}_t \omega_t^k}\phi(s^{(0)}).
\end{align*}

\end{proof}

\begin{lemma}%
\label{lmm: crude parameter bound}
Fix $\tau\in \naturals$. 
When $\frac{\zeta U_{\delta, L} U_{\omega}}{L} \le 1/2$,  
for any $t\geq \tau> 0$ and $k \in [K]$, it holds that  
\begin{align*}
\|\omega_t^k - \omega_{t-\tau}^k\|\leq \tau \frac{U_{\delta,L}}{L}\beta, ~
~\text{and }~ 
\|\mathbf{B}_t - \mathbf{B}_{t-\tau}\|\leq 2\tau  \frac{U_{\delta,L}}{L}U_{\omega}\zeta + 4\tau \frac{U_{\delta,L}^2}{L^2} U_{\omega}^2\zeta^2 .  
\end{align*}   
\end{lemma}
\begin{proof}
For any $t\geq \tau> 0$ and $k \in [K]$, we have 
    \begin{align*}
        &\|\omega_t^k-\omega_{t-\tau}^k\|\leq \sum_{i=t-\tau}^{t-1} \|\omega_{i+1}^k-\omega_{i}^k\|\\
        =&\sum_{i=t-\tau}^{t-1} \|\Pi_{U_{\omega}} \left( \omega^k_i + \beta\frac{1}{L} \delta_{i, L}^k \mathbf{B}_i^\top \phi(s_{i,0}^k) \right)-\Pi_{U_\omega}(\omega_{i}^k)\|\\
        \leq& \sum_{i=t-\tau}^{t-1} \|\beta\frac{1}{L} \delta_{i, L}^k \mathbf{B}_i^\top \phi(s_{i,0}^k)\|\\
        \leq& \tau \frac{U_{\delta,L}}{L}\beta,
    \end{align*}
    where the last inequality follows from Lemma \ref{lmm: U delta}. 
    In addition, 
    \begin{align*}
    &\|\bfB_{t}-\bfB_{t - \tau}\|
    \leq  \sum_{i=t-\tau}^{t-1}\|\bfB_{i+1} - \bfB_{i}\|\\
    =& \sum_{i=t-\tau}^{t-1}\| (\mathbf{B}_{i}
    +\frac{\zeta}{KL}\sum_{k=1}^K \bfB_{i,\perp}\bfB_{i,\perp}^\top\delta_{i,L}^k\phi(s_{i,0}^k)(\omega_{i}^k)^\top)\bfR_{i+1}^{-1}  - \bfB_{i}\|\\
    =& \sum_{i=t-\tau}^{t-1}\| (\mathbf{B}_{i}
    +\frac{\zeta}{KL}\sum_{k=1}^K \bfB_{i,\perp}\bfB_{i,\perp}^\top\delta_{i,L}^k\phi(s_{i,0}^k)(\omega_{i}^k)^\top)\bfR_{i+1}^{-1} -\bfB_i\bfR_{i+1}^{-1}+ \bfB_i\bfR_{i+1}^{-1}- \bfB_{i}\|\\
    \leq&  \sum_{i=t-\tau}^{t-1}\|\frac{\zeta}{KL}\sum_{k=1}^K \bfB_{i,\perp}\bfB_{i,\perp}^\top\delta_{i,L}^k\phi(s_{i,0}^k)(\omega_{i}^k)^\top\|\|\bfR_{i+1}^{-1}\| + \sum_{i=t-\tau}^{t-1}\|\bfR_{i+1}^{-1}-\bfI\| \|\bfB_i\|\\
    \leq & \zeta \frac{U_{\delta,L}}{L}U_{\omega} \sum_{i=t-\tau}^{t-1}\|\bfR_{i+1}^{-1}\| + \sum_{i=t-\tau}^{t-1}\|\bfR_{i+1}^{-1}-\bfI\|\\
    \leq &  \zeta \frac{U_{\delta,L}}{L}U_{\omega} \sum_{i=t-\tau}^{t-1} \frac{1}{1-2\|\bfQ_i\|_F^2}  +  \sum_{i=t-\tau}^{t-1} 4\|\bfQ_i\|_F^2, 
\end{align*}
where the last inequality follows from Lemma \ref{lm: perturbation QR}, and $\bfQ_i$ is defined in Eq.\,(\ref{eq: def Q}) as
\begin{align*}
\bfQ_i = \frac{\zeta}{KL}\sum_{k=1}^K \bfB_{i,\perp}\bfB_{i,\perp}^\top\delta_{i,L}^k\phi(s_{i,0}^k)(\omega_{i}^k)^\top.    
\end{align*}
We have 
\begin{align}
\label{eqn: Q crude bound}
\|\bfQ_i\|_F
& \le \frac{\zeta}{KL}\sum_{k=1}^K \|\bfP_{i,\perp}\delta_{i,L}^k\phi(s_i^k)(\omega_{i}^k)^\top  \|_F \nonumber\\
& \le \frac{\zeta}{KL}\sum_{k=1}^K |\delta_{i,L}^k| \|\bfP_{i,\perp}\| \|\phi(s_i^k)\| \|\omega_{i}^k\| \nonumber\\
& \le \zeta \frac{U_{\delta,L}}{L} U_{\omega}. 
\end{align}
When $\zeta \frac{U_{\delta,L}}{L} U_{\omega} \le 1/2$, $1-2\|\bfQ_i\|_F^2 \ge 1/2$.  
Therefore, we have   
\begin{align*}
\|\bfB_{t}-\bfB_{t - \tau}\| 
\leq &  2\tau \zeta \frac{U_{\delta,L}}{L}U_{\omega} + 4\tau \zeta^2 \frac{U_{\delta,L}^2}{L^2} U_{\omega}^2.  
\end{align*}

\end{proof}

\begin{lemma}
\label{lmm: one step temporal difference bound of omega}
Fix $\tau \in \naturals$. 
For any $k \in [K]$ and any $j\geq t-\tau> \tau$, where $t\ge 2\tau$, it holds that   
   \begin{align*}
   \bbE_{t-2\tau}\|\omega_{j+1}^k-\omega_{j}^k\| 
   &\leq
   C_1(\tau)\frac{\beta^2}{L} + C_2(\tau)\frac{\zeta\beta}{L} + C_3(\tau)\gamma\beta + C_4\frac{\beta}{\sqrt{L}}+2\frac{\beta}{L}\bbE_{t-2\tau}\|x_t^k\| +\beta\bbE_{t-2\tau}|y_t^k|, \\
   \bbE_{t-2\tau}\|\omega_{j+1}^k-\omega_{j}^k\|^2 
   &\leq
    C_5(\tau^2)\frac{\beta^4}{L^2}
    + C_6(\tau^2)\frac{\zeta^2\beta^2}{L^2}
    +8\frac{\beta^2}{L^2} \bbE_{t-2\tau}\|x_t^k\|^2
    + C_7\frac{\beta^2}{L}
    + C_8(\tau^2)\gamma^2\beta^2+2\beta^2\bbE_{t-2\tau}|y_t^k|^2,    
   \end{align*}
where 
\begin{align}
\label{eq: constants 1-8}
    &C_1(\tau) = 2\tau \frac{U_{\delta,L}}{L}, C_2(\tau)=8\tau\frac{U_{\delta,L}}{L}U_\omega^2,C_3(\tau)=2\tau U_r,C_4=\pth{U_\delta+ U_\delta\sqrt{\frac{2 m\rho }{1-\rho} }}, \nonumber \\
    &C_5(\tau^2) = 48\tau^2 \frac{U_{\delta,L}^2}{L^2}, C_6(\tau^2)=768 \tau^2 U_\omega^3  \frac{U_{\delta,L}^2}{L^2}, C_7 =3U_\delta^2(1+2\frac{m\rho}{1-\rho}), C_8(\tau^2) = 24\tau^2U_r^2.
\end{align}
\end{lemma}

\begin{proof}

For any $j$, one has 
\begin{align}
\label{eqn: one step omega decomp}
     &\|\omega_{j+1}^k - \omega_{j}^k\|= \|\Pi_{U_\omega}(\omega_{j}^k+\frac{\beta}{L}\delta_{j,L}^k\bfB_{j}^\top\phi(s_{j,0}^k) ) - \Pi_{U_\omega}(\omega_{j}^k)\|\nonumber \\
    \leq & \|\frac{\beta}{L}\bfB_{j}^\top \delta_{j,L}^k \phi(s_{j,0}^k)\| \nonumber \qquad \qquad \text{(by non-expansive property of projection)}\\
    =&\frac{\beta}{L}\|\bfB_j^\top \pth{\tilde{A}_{j,L}^kx_j^k+\bfb_{j,L}^k+Ly_j^k\phi(s_{j,0}^k)} \|\nonumber ~~~~~~ \text{(by Proposition \ref{prop: key intermediate})} \\
    \leq&\frac{\beta}{L}\|\bfB_j^\top { \tilde{A}_{j,L}^k}x_j^k\|+\frac{\beta}{L}\|\bfB_j^\top\bfb_{j,L}^k\|+\frac{\beta}{L}\|\bfB_j^\top Ly_j^k\phi(s_{j,0}^k) \|\nonumber \\
    \leq&{ \frac{2\beta}{L}}\|x_j^k\|+\frac{\beta}{L}\|\bfb_{j,L}^k\|+\beta | y_j^k|.
\end{align}

 By Lemma~\ref{lmm: local noise reduction}, 
and by the tower property and the notation defined in Section \ref{sec: Notations},
\begin{align*}
\bbE_{t-2\tau}\|\bfb_{j,L}^k\|
&\le
\pth{U_\delta+ U_\delta\sqrt{\frac{2 m\rho }{1-\rho} }}\sqrt{L}, 
\end{align*}
where $U_\delta = 2U_r+2U_\omega$ as defined in Lemma~\ref{lmm: local noise reduction}.  
 Then, taking expectation up to $t-2\tau$, \eqref{eqn: one step omega decomp} can be written as,
\begin{align}
\label{eqn: one step omega before substituting jth}
    \bbE_{t-2\tau}\|\omega_{j+1}^k - \omega_{j}^k\|\leq { \frac{2\beta}{L}}\bbE_{t-2\tau}\|x_j^k\| + \frac{\beta}{\sqrt{L}}\pth{U_\delta+ U_\delta\sqrt{\frac{2 m\rho }{1-\rho} }}+\beta  \bbE_{t-2\tau}|y_j^k|.
\end{align}

By Lemma \ref{lmm: jth to tth error bound}, \eqref{eqn: one step omega before substituting jth} can be written as,
\begin{align*}
    &\bbE_{t-2\tau}\|\omega_{j+1}^k-\omega_{j}^k\|\leq 2\frac{\beta}{L}\bbE_{t-2\tau}\|x_j^k\| + \frac{\beta}{\sqrt{L}}\pth{U_\delta+ U_\delta\sqrt{\frac{2 m\rho }{1-\rho} }}+\beta  \bbE_{t-2\tau}|y_j^k|\\
    \leq& 2\frac{\beta}{L}\pth{\tau \frac{U_{\delta,L}}{L}\beta + 4\tau\frac{U_{\delta,L}}{L}U_\omega^2 \zeta +\bbE_{t-2\tau}\|x_t^k\|} + \frac{\beta}{\sqrt{L}}\pth{U_\delta+ U_\delta\sqrt{\frac{2 m\rho }{1-\rho} }}\\
    &+\beta \pth{2\tau U_r\gamma+\bbE_{t-2\tau}|y_t^k|}\\
    =&\underbrace{2\tau \frac{U_{\delta,L}}{L}}_{:=C_1(\tau)}{ \frac{\beta^2}{L}} + \underbrace{8\tau\frac{U_{\delta,L}}{L}U_\omega^2 }_{:=C_2(\tau)}\frac{\zeta\beta}{L} + \underbrace{2\tau U_r}_{:=C_3(\tau)}\gamma\beta
    + \frac{\beta}{\sqrt{L}}\underbrace{\pth{U_\delta+ U_\delta\sqrt{\frac{2 m\rho }{1-\rho} }}}_{:=C_4}+2\frac{\beta}{L}\bbE_{t-2\tau}\|x_t^k\| +\beta\bbE_{t-2\tau}|y_t^k|\\
    \le &C_1(\tau)\frac{\beta^2}{L} + C_2(\tau)\frac{\zeta\beta}{L} 
    + C_3(\tau)\gamma\beta + C_4\frac{\beta}{\sqrt{L}}+2\frac{\beta}{L}\bbE_{t-2\tau}\|x_t^k\| +\beta\bbE_{t-2\tau}|y_t^k|.
\end{align*}

 Similarly,
 \begin{align}
\label{eqn: one step omega ^2 decomp}
     &\bbE_{t-2\tau}\|\omega_{j+1}^k - \omega_{j}^k\|^2= \bbE_{t-2\tau}\|\Pi_{U_\omega}(\omega_{j}^k+\frac{\beta}{L}\delta_{j,L}^k\bfB_{j}^\top\phi(s_{j,0}^k) ) - \Pi_{U_\omega}(\omega_{j}^k)\|^2\nonumber \\
    \leq & \bbE_{t-2\tau}\|\frac{\beta}{L}\bfB_{j}^\top \delta_{j,L}^k \phi(s_{j,0}^k)\|^2 \nonumber \\
    =&\frac{\beta^2}{L^2}\bbE_{t-2\tau}\|\bfB_j^\top \pth{\tilde{A}_{j,L}^kx_j^k+\bfb_{j,L}^k+Ly_j^k\phi(s_{j,0}^k)} \|^2\nonumber \quad \quad ~~~~~~ \text{(by Proposition \ref{prop: key intermediate})} \\
    \leq&3\frac{\beta^2}{L^2}\bbE_{t-2\tau}\|\bfB_j^\top { \tilde{A}_{j,L}^k}x_j^k\|^2 + 3\frac{\beta^2}{L^2}\bbE_{t-2\tau}\|\bfB_j^\top\bfb_{j,L}^k\|^2 + 3\frac{\beta^2}{L^2}\bbE_{t-2\tau}\|\bfB_j^\top Ly_j^k\phi(s_{j,0}^k) \|^2\nonumber \\
    \leq&{\frac{12\beta^2}{L^2}}\bbE_{t-2\tau}\|x_j^k\|^2+3\frac{\beta^2}{L^2}\bbE_{t-2\tau}\|\bfb_{j,L}^k\|^2+ 3\beta^2 \bbE_{t-2\tau}| y_j^k|^2.
\end{align}
By Lemma~\ref{lmm: local noise reduction} and the tower property, we get,
\begin{align*}
    \bbE_{t-2\tau}\|\bfb_{t,L}^k\|^2\leq U_\delta^2(1+2\frac{m\rho}{1-\rho})L.
\end{align*}
By Lemma \ref{lmm: jth to tth error bound}, \eqref{eqn: one step omega ^2 decomp} can be written as,
\begin{align*}
    &\bbE_{t-2\tau}\|\omega_{j+1}^k - \omega_{j}^k\|^2
    \leq 3\pth{{\frac{4\beta^2}{L^2}}\bbE_{t-2\tau}\|x_j^k\|^2+U_\delta^2(1+2\frac{m\rho}{1-\rho})\frac{\beta^2}{L}+\beta^2 \bbE_{t-2\tau}| y_j^k|^2}\\
    \leq&  48\tau^2 \frac{U_{\delta,L}^2}{L^4}\beta^4
    + 768 \tau^2 U_\omega^3  \frac{U_{\delta,L}^2}{L^4}\zeta^2\beta^2
    +24\frac{\beta^2}{L^2} \bbE_{t-2\tau}\|x_t^k\|^2\\
    & + 3U_\delta^2(1+2\frac{m\rho}{1-\rho})\frac{\beta^2}{L}+24\tau^2U_r^2 \gamma^2\beta^2+6\beta^2\bbE_{t-2\tau}|y_t^k|^2.
\end{align*}
\end{proof}

\begin{lemma}
\label{lmm: one-step perturbation of B}
Fix $\tau>0$. 
Choosing $\zeta$ so that $\zeta \frac{U_{\delta,L}}{L} U_{\omega} \le 1/2$ and $ U_\omega \zeta \leq 1/2$, 
for any $j\ge t-\tau\ge \tau$, we have 
\begin{align*}
\bbE_{t-2\tau}[\|\mathbf{B}_{j+1} - \mathbf{B}_j\| ] \le& {C_9(\tau)}\frac{\beta\zeta}{L}
    +{C_{10}(\tau)} \frac{\zeta^2}{L} 
    +4U_\omega\frac{\zeta}{KL}\sum_{k=1}^K\bbE_{t-2\tau}\|x_t^k\| +{C_{11}(\tau)}\zeta\beta+{C_{12}}\frac{\zeta}{\sqrt{K}}\\
    & +{C_{13}} \zeta m\rho^{\tau} +{C_{14}(\tau)}\gamma\zeta+2U_\omega\frac{\zeta}{K}\sum_{k=1}^K\bbE_{t-2\tau}|y_t^k|
    +{C_{15}}\zeta^2, \\
\bbE_{t-2\tau}\|\mathbf{B}_{j+1} - \mathbf{B}_j\|^2  
    \leq& C_{16}(\tau^2)\frac{\zeta^2\beta^2}{L^2}
    + C_{17}(\tau^2)\frac{\zeta^4}{L^2}
    + C_{18}\frac{\zeta^2}{KL^2}\sum_{k=1}^K\bbE_{t-2\tau}\|x_t^k\|^2
    + C_{19}(\tau^2)\zeta^2 \beta^2
+C_{20}\frac{\zeta^2}{K}\\
&+ C_{21}\zeta^2m^2\rho^{2\tau}
+ C_{22}(\tau^2) \gamma^2\zeta^2+C_{23}\frac{\zeta^2}{K}\sum_{k=1}^K \bbE_{t-2\tau}|y_t^k|^2
+ C_{24}\zeta^4,    
\end{align*} 
where
\begin{align*}
    &C_9(\tau) = 4U_\omega\tau \frac{U_{\delta,L}}{L}, C_{10}(\tau)=16\tau\frac{U_{\delta,L}}{L}U_\omega^3, C_{11}(\tau)=2\frac{U_{\delta,L}^2}{L^2}\tau, C_{12}=2U_\omega \frac{U_{\delta,L} }{L},\\& C_{13}=4\frac{U_{\delta,L}}{L}    U_\omega, C_{14}(\tau)=4U_\omega\tau U_r, C_{15}=4\frac{U_{\delta,L}^2 }{L}U_{\omega}^2, \\
    &C_{16}(\tau^2) = 384\tau^2U_\omega^2\frac{U_{\delta,L}^2}{L^2}, C_{17}(\tau^2) = 6144 \tau^2U_\omega^5\frac{U_{\delta,L}^2}{L^2}, C_{18}=192U_\omega^2, C_{19}=\frac{48U_{\delta,L}^4}{L^4} \tau^2 , C_{20} = \frac{48U_\omega^2 U_{\delta,L}^2}{L^2},\\
    &C_{21}=\frac{48U_\omega^2 U_{\delta,L}^2}{L^2}, C_{22}(\tau^2)=192U_\omega^2\tau^2U_r^2,C_{23}=48U_\omega^2,C_{24} = 32 \frac{U_{\delta,L}^4 }{L^4}U_{\omega}^4.
\end{align*}
\end{lemma}
\begin{proof}
Fix $\tau>0$. For any $j\geq t-\tau$, the one-step perturbation of $\bfB_j$ can be bounded as 
\begin{align}
\label{eq: auxiliary lemma: principle angle perturbation}
    &\|\bfB_{j+1}-\bfB_j\|\nonumber \\ 
    = & \| (\mathbf{B}_{j}
    +\frac{\zeta}{KL}\sum_{k=1}^K \bfP_{j,\perp}\delta_{j,L}^k\phi(s_{j,0}^k)(\omega_{j}^k)^\top)\bfR_{j+1}^{-1}  - \bfB_{j}\|\nonumber \\
    =&\| (\mathbf{B}_{j}
    +\frac{\zeta}{KL}\sum_{k=1}^K \bfP_{j,\perp}\delta_{j,L}^k\phi(s_{j,0}^k)(\omega_{j}^k)^\top)\bfR_{j+1}^{-1} -\bfB_j\bfR_{j+1}^{-1}+ \bfB_j\bfR_{j+1}^{-1}- \bfB_{j}\|\nonumber \\
    \leq& \left\|\frac{\zeta}{KL}\sum_{k=1}^K \bfP_{j,\perp}\delta_{j,L}^k\phi(s_{j,0}^k)(\omega_{j}^k)^\top\right\|\left\|\bfR_{j+1}^{-1}\right\| + \|\bfR_{j+1}^{-1}-\bfI\| \|\bfB_j\|\nonumber \\
    \leq& \left\|\frac{\zeta}{KL}\sum_{k=1}^K \bfP_{j,\perp}\delta_{j,L}^k\phi(s_{j,0}^k)(\omega_{j}^k)^\top\right\| \frac{1}{1-2\|\bfQ_j\|_F^2} + 4\|\bfQ_j\|_F^2, 
\end{align}
where the last inequality follows from Lemma \ref{lm: perturbation QR} and the fact that $\|\bfB_j\| \le 1$. 
From \eqref{eq: def Q}, we know 
$\|\bfQ_j\|_F \le \frac{U_{\delta,L} }{L}U_{\omega}\zeta.$ When $ \frac{U_{\delta,L}}{L} U_{\omega}\zeta \le 1/2$, it holds that $\frac{1}{1-2\|\bfQ_j\|_F^2} \le 2$. 
Thus, it remains to bound $\left\|\frac{\zeta}{KL}\sum_{k=1}^K \bfP_{j,\perp}\delta_{j,L}^k\phi(s_{j,0}^k)(\omega_{j}^k)^\top\right\|$, for which we have  
\begin{align}
\label{eq: auxiliary lemma: principle angle perturbation partial}
&\left\|\frac{\zeta}{KL}\sum_{k=1}^K \bfP_{j,\perp}\delta_{j,L}^k\phi(s_{j,0}^k)(\omega_{j}^k)^\top\right\| \nonumber \\
=&\frac{\zeta}{KL}\left\|\sum_{k=1}^K \bfP_{j,\perp}\delta_{j,L}^k\phi(s_{j,0}^k)(\omega_{j}^k)^\top\right\| \nonumber \\
\leq&\frac{\zeta}{KL}\left\|\sum_{k=1}^K \delta_{j,L}^k\phi(s_{j,0}^k)(\omega_{j}^k)^\top\right\| \nonumber \\
\leq&\frac{\zeta}{KL}\left\|\sum_{k=1}^K \pth{{ \tilde{A}_{j,L}^k} x_j^k +\bfb_{j,L}^k+Ly_j^k\phi(s_{j,0}^k)}(\omega_{j}^k)^\top\right\|\nonumber \\
\leq&\frac{\zeta}{KL}\left\|\sum_{k=1}^K { \tilde{A}_{j,L}^k} x_j^k(\omega_{j}^k)^\top\right\|
+\frac{\zeta}{KL}\left\|\sum_{k=1}^K \bfb_{j,L}^k(\omega_{j}^k)^\top\right\| 
+\frac{\zeta}{KL}\left\|\sum_{k=1}^K Ly_j^k\phi(s_{j,0}^k)(\omega_{j}^k)^\top\right\|\nonumber\\
\leq&2U_\omega\frac{\zeta}{KL}\sum_{k=1}^K   \| x_j^k\|
+  \underbrace{\frac{\zeta}{KL}\left\|\sum_{k=1}^K \bfb_{j,L}^k(\omega_{j}^k-\omega_{j-\tau}^k)^\top\right\|}_{I_{1}}
+\underbrace{\frac{\zeta}{KL}\left\|\sum_{k=1}^K \bfb_{j,L}^k(\omega_{j-\tau}^k)^\top\right\| }_{I_2}+ U_\omega\frac{\zeta}{K}\sum_{k=1}^K |y_j^k|. %
\end{align}
For $I_1$ in \eqref{eq: auxiliary lemma: principle angle perturbation partial}, 
\begin{align}
\label{eqn: one-step B fine bound partial 1}
    &\frac{\zeta}{KL}\left\|\sum_{k=1}^K \bfb_{j,L}^k(\omega_{j}^k-\omega_{j-\tau}^k)^\top\right\|
    \leq  \frac{U_{\delta,L}^2}{L^2}\tau\zeta\beta. ~~~~~(\text{ by Lemma \ref{lmm: crude parameter bound}})
\end{align}
For $I_2$ in \eqref{eq: auxiliary lemma: principle angle perturbation partial}, taking conditional expectation with respect to the Markovian trajectories up to time $j-\tau$, we have 
\begin{align}
\label{eqn: one-step B fine bound partial 2}
    &\frac{\zeta}{KL} \bbE_{j-\tau}\left\|\sum_{k=1}^K \bfb_{j,L}^k(\omega_{j-\tau}^k)^\top \right\|\nonumber\\
    \leq&\frac{\zeta}{KL} \pth{\bbE_{j-\tau}\left\|\sum_{k=1}^K \bfb_{j,L}^k(\omega_{j-\tau}^k)^\top \right\|^2}^{\frac{1}{2}} ~~~~(\text{by Jensen's inequality for expectation})\nonumber\\
    =&\frac{\zeta}{KL} \pth{\bbE_{j-\tau}\sum_{k=1}^K\left\| \bfb_{j,L}^k(\omega_{j-\tau}^k)^\top \right\|^2  + 2 \sum_{k=1}^K\sum_{k'=k +1 }^K \bbE_{j-\tau}\iprod{\bfb_{j,L}^k(\omega_{j-\tau}^k)^\top}{\bfb_{j,L}^{k'}(\omega_{j-\tau}^{k'})^\top}}^{\frac{1}{2}}\nonumber\\
    \overset{(a)}{\leq}&\frac{\zeta}{KL} \pth{KU_\omega^2U_{\delta,L}^2  + 2 \sum_{k=1}^K\sum_{ k'=k+1 }^K \bbE_{j-\tau}\qth{\bfb_{j,L}^k(\omega_{j-\tau}^k)^\top} \bbE_{j-\tau}\qth{\bfb_{j,L}^{k'}(\omega_{j-\tau}^{k'})^\top}}^{\frac{1}{2}} \nonumber\\ %
    \leq&\frac{\zeta}{KL} \pth{KU_\omega^2U_{\delta,L}^2  + K^2 U_{\delta,L}^2 U_\omega^2m^2\rho^{2\tau}}^{\frac{1}{2}}\nonumber\\
    \leq& U_\omega \frac{U_{\delta,L} }{L}\frac{\zeta}{\sqrt{K}} +\frac{U_{\delta,L}}{L}    U_\omega \zeta m\rho^{\tau}.
\end{align}
where inequality (a) follows from Lemma \ref{lmm: U delta} and conditional independence.

Plugging \eqref{eqn: one-step B fine bound partial 1} and \eqref{eqn: one-step B fine bound partial 2} back to \eqref{eq: auxiliary lemma: principle angle perturbation partial} and \eqref{eq: auxiliary lemma: principle angle perturbation}, we get, under expectation,
\begin{align}
\label{eqn: one step B before substituting jth}
\bbE_{j-\tau}[\|\mathbf{B}_{j+1} - \mathbf{B}_j\|] &= 
\bbE[\|\mathbf{B}_{j+1} - \mathbf{B}_j\| \mid v_{0: j-\tau}] \nonumber \\
&\leq 4U_\omega\frac{\zeta}{KL}\sum_{k=1}^K\bbE_{j-\tau}\|x_j^k\| +2\frac{U_{\delta,L}^2}{L^2}\tau\zeta\beta+ 2U_\omega \frac{U_{\delta,L} }{L}\frac{\zeta}{\sqrt{K}} \nonumber\\
& \quad +4\frac{U_{\delta,L}}{L}    U_\omega \zeta m\rho^{\tau} +2U_\omega\frac{\zeta}{K}\sum_{k=1}^K \bbE_{j-\tau}|y_j^k|
+4\frac{U_{\delta,L}^2 }{L}U_{\omega}^2\zeta^2.
\end{align}
Finally, %
\begin{align*}
    &\bbE_{t-2\tau}[\|\mathbf{B}_{j+1} - \mathbf{B}_j\| ] = \bbE[\|\mathbf{B}_{j+1} - \mathbf{B}_j\| \mid v_{0: t-2\tau}]  \nonumber \\
    =&\bbE[~\bbE[\|\mathbf{B}_{j+1} - \mathbf{B}_j\| \mid v_{0: j-\tau}]\mid v_{0: t-2\tau}] \quad \quad ~~~~~ \text{(by the tower property of conditional expecation)} \nonumber \\
    \leq &4U_\omega\frac{\zeta}{KL}\sum_{k=1}^K\bbE_{t-2\tau}[\bbE_{j-\tau}\|x_j^k\|] +2\frac{U_{\delta,L}^2}{L^2}\tau\zeta\beta+ 2U_\omega \frac{U_{\delta,L} }{L}\frac{\zeta}{\sqrt{K}} \quad ~~~~~ \text{(by applying Eq.\,(\ref{eqn: one step B before substituting jth}))}\\
    & +4\frac{U_{\delta,L}}{L}    U_\omega \zeta m\rho^{\tau} +2U_\omega\frac{\zeta}{K}\sum_{k=1}^K \bbE_{t-2\tau}[\bbE_{j-\tau}|y_j^k|]
    +4\frac{U_{\delta,L}^2 }{L}U_{\omega}^2\zeta^2\\
    =&4U_\omega\frac{\zeta}{KL}\sum_{k=1}^K\bbE_{t-2\tau}[\|x_j^k\|] +2\frac{U_{\delta,L}^2}{L^2}\tau\zeta\beta+ 2U_\omega \frac{U_{\delta,L} }{L}\frac{\zeta}{\sqrt{K}}\\
    & +4\frac{U_{\delta,L}}{L}    U_\omega \zeta m\rho^{\tau} +2U_\omega\frac{\zeta}{K}\sum_{k=1}^K \bbE_{t-2\tau}[|y_j^k|]
    +4\frac{U_{\delta,L}^2 }{L}U_{\omega}^2\zeta^2\\
    \leq& 4U_\omega\frac{\zeta}{L}\pth{\tau \frac{U_{\delta,L}}{L}\beta + 4\tau\frac{U_{\delta,L}}{L}U_\omega^2 \zeta +\frac{1}{K}\sum_{k=1}^K\bbE_{t-2\tau}\|x_t^k\|} +2\frac{U_{\delta,L}^2}{L^2}\tau\zeta\beta+ 2U_\omega \frac{U_{\delta,L} }{L}\frac{\zeta}{\sqrt{K}}\\
    & +4\frac{U_{\delta,L}}{L}    U_\omega \zeta m\rho^{\tau} +2U_\omega\frac{\zeta}{K}\sum_{k=1}^K\pth{2\tau U_r\gamma+\bbE_{t-2\tau}|y_t^k|}
    +4\frac{U_{\delta,L}^2 }{L}U_{\omega}^2\zeta^2 \quad \quad ~~~~~ \text{(applying Lemma \ref{lmm: jth to tth error bound})}\\
    =&\underbrace{4U_\omega\tau \frac{U_{\delta,L}}{L}}_{:=C_9(\tau)}\frac{\beta\zeta}{L}
    +\underbrace{16\tau\frac{U_{\delta,L}}{L}U_\omega^3}_{:=C_{10}(\tau)} \frac{\zeta^2}{L} 
    +4U_\omega\frac{\zeta}{KL}\sum_{k=1}^K\bbE_{t-2\tau}\|x_t^k\| +\underbrace{2\frac{U_{\delta,L}^2}{L^2}\tau}_{:=C_{11}(\tau)}\zeta\beta+\underbrace{ 2U_\omega \frac{U_{\delta,L} }{L}}_{:=C_{12}}\frac{\zeta}{\sqrt{K}}\\
    & +\underbrace{4\frac{U_{\delta,L}}{L}    U_\omega}_{:=C_{13}} \zeta m\rho^{\tau} +\underbrace{4U_\omega\tau U_r}_{:=C_{14}(\tau)}\gamma\zeta+2U_\omega\frac{\zeta}{K}\sum_{k=1}^K\bbE_{t-2\tau}|y_t^k|
    +\underbrace{4\frac{U_{\delta,L}^2 }{L}U_{\omega}^2}_{:=C_{15}}\zeta^2\\
    =& {C_9(\tau)}\frac{\beta\zeta}{L}
    +{C_{10}(\tau)} \frac{\zeta^2}{L} 
    +4U_\omega\frac{\zeta}{KL}\sum_{k=1}^K\bbE_{t-2\tau}\|x_t^k\| +{C_{11}(\tau)}\zeta\beta+{C_{12}}\frac{\zeta}{\sqrt{K}}\\
    & +{C_{13}} \zeta m\rho^{\tau} +{C_{14}(\tau)}\gamma\zeta+2U_\omega\frac{\zeta}{K}\sum_{k=1}^K\bbE_{t-2\tau}|y_t^k|
    +{C_{15}}\zeta^2.
\end{align*}

Analogously, 
\begin{align}
\label{eqn: B perturb ^2}
   &\|\bfB_{j+1}-\bfB_j\|^2\nonumber \\ 
    \leq& 2\left\|\frac{\zeta}{KL}\sum_{k=1}^K \bfP_{j,\perp}\delta_{j,L}^k\phi(s_{j,0}^k)(\omega_{j}^k)^\top\right\|^2 \pth{\frac{1}{1-2\|\bfQ_j\|_F^2}}^2 + 32\|\bfQ_j\|_F^4 \nonumber\\
    \leq& 8\left\|\frac{\zeta}{KL}\sum_{k=1}^K \bfP_{j,\perp}\delta_{j,L}^k\phi(s_{j,0}^k)(\omega_{j}^k)^\top\right\|^2 +  32\|\bfQ_j\|_F^4.
\end{align}
Similar to \eqref{eq: auxiliary lemma: principle angle perturbation partial}, the first term can be decomposed and bounded as,
\begin{align}
\label{eqn: B perturb ^2, 1}
    &8\left\|\frac{\zeta}{KL}\sum_{k=1}^K \bfP_{j,\perp}\delta_{j,L}^k\phi(s_{j,0}^k)(\omega_{j}^k)^\top\right\|^2\nonumber\\
\leq&\frac{24\zeta^2}{K^2L^2}\left\|\sum_{k=1}^K { \tilde{A}_{j,L}^k} x_j^k(\omega_{j}^k)^\top\right\|^2
+\frac{24\zeta^2}{K^2L^2}\left\|\sum_{k=1}^K \bfb_{j,L}^k(\omega_{j}^k)^\top\right\|^2 
+\frac{24\zeta^2}{K^2L^2}\left\|\sum_{k=1}^K Ly_j^k\phi(s_{j,0}^k)(\omega_{j}^k)^\top\right\|^2\nonumber\\
\leq&96U_\omega^2\frac{\zeta^2}{KL^2} \sum_{k=1}^K  \| x_j^k\|^2
+  \underbrace{\frac{48\zeta^2}{K^2L^2}\left\|\sum_{k=1}^K \bfb_{j,L}^k(\omega_{j}^k-\omega_{j-\tau}^k)^\top\right\|^2}_{I_{3}}
+\underbrace{\frac{48\zeta^2}{K^2L^2}\left\|\sum_{k=1}^K \bfb_{j,L}^k(\omega_{j-\tau}^k)^\top\right\|^2 }_{I_4}+ U_\omega^2\frac{24\zeta^2}{K}\sum_{k=1}^K |y_j^k|^2.
\end{align}

For $I_3$ in \eqref{eqn: B perturb ^2, 1}, by Lemma \ref{lmm: U delta} and Lemma \ref{lmm: crude parameter bound}, we have 
\begin{align}
\label{eqn: I_3 in B perturb ^2, 1}
    &\frac{48\zeta^2}{K^2L^2}\left\|\sum_{k=1}^K \bfb_{j,L}^k(\omega_{j}^k-\omega_{j-\tau}^k)^\top\right\|^2
    \leq \frac{48U_{\delta,L}^4}{L^4} \tau^2 \zeta^2 \beta^2. %
\end{align}
For $I_4$ in \eqref{eqn: B perturb ^2, 1}, we take the conditional expectation with respect to the Markovian trajectories up to time $j-\tau$. Similar to \eqref{eqn: one-step B fine bound partial 2}, we obtain 
\begin{align}
\label{eqn: I_4 in B perturb ^2, 1}
    \frac{48\zeta^2}{K^2L^2}\bbE_{j-\tau}\left\|\sum_{k=1}^K \bfb_{j,L}^k(\omega_{j-\tau}^k)^\top\right\|^2
    &\leq \frac{48\zeta^2}{K^2L^2} \pth{KU_\omega^2U_{\delta,L}^2  + K^2 U_{\delta,L}^2 U_\omega^2m^2\rho^{2\tau}} \nonumber \\
    &\leq \frac{48U_\omega^2 U_{\delta,L}^2}{L^2}\frac{\zeta^2}{K} + \frac{48U_\omega^2 U_{\delta,L}^2}{L^2}\zeta^2m^2\rho^{2\tau}.
\end{align}
Plugging \eqref{eqn: I_3 in B perturb ^2, 1} and \eqref{eqn: I_4 in B perturb ^2, 1} back to \eqref{eqn: B perturb ^2, 1} and \eqref{eqn: B perturb ^2}, and applying the tower property, we get   
\begin{align*}
&\bbE_{t-2\tau}\|\bfB_{j+1}-\bfB_j\|^2
=\bbE_{t-2\tau}\qth{\bbE_{j-\tau}\|\bfB_{j+1}-\bfB_j\|^2}\\
\leq& 96U_\omega^2\frac{\zeta^2}{KL^2}  \sum_{k=1}^K \bbE_{t-2\tau}\| x_j^k\|^2
+  \frac{48U_{\delta,L}^4}{L^4} \tau^2 \zeta^2 \beta^2
+\frac{48U_\omega^2 U_{\delta,L}^2}{L^2}\frac{\zeta^2}{K}\\
& + \frac{48U_\omega^2 U_{\delta,L}^2}{L^2}\zeta^2m^2\rho^{2\tau}
+ U_\omega^2\frac{24\zeta^2}{K^2}\sum_{k=1}^K \bbE_{t-2\tau}|y_j^k|^2
+ 32 (\frac{U_{\delta,L} }{L}U_{\omega}\zeta)^4\\
\leq& 96U_\omega^2\frac{\zeta^2}{L^2}\pth{
4\tau^2 \frac{U_{\delta,L}^2}{L^2}\beta^2 
+  64 \tau^2 U_\omega^3  \frac{U_{\delta,L}^2}{L^2}\zeta^2 
+ 2\frac{1}{K}\sum_{k=1}^K\bbE_{t-2\tau}\|x_t^k\|^2}
+  \frac{48U_{\delta,L}^4}{L^4} \tau^2 \zeta^2 \beta^2
+\frac{48U_\omega^2 U_{\delta,L}^2}{L^2}\frac{\zeta^2}{K}\\
& + \frac{48U_\omega^2 U_{\delta,L}^2}{L^2}\zeta^2m^2\rho^{2\tau}
+ U_\omega^2\frac{24\zeta^2}{K^2}\sum_{k=1}^K \pth{8\tau^2U_r^2 \gamma^2+2\bbE_{t-2\tau}|y_t^k|^2}
+ 32 (\frac{U_{\delta,L} }{L}U_{\omega}\zeta)^4~~~~~ \text{(applying Lemma \ref{lmm: jth to tth error bound})}\\
\leq& 384\tau^2U_\omega^2\frac{U_{\delta,L}^2}{L^2}\frac{\zeta^2\beta^2}{L^2}+6144 \tau^2 U_\omega^5\frac{U_{\delta,L}^2}{L^2} \frac{\zeta4}{L^2} + 192U_\omega^2\frac{\zeta^2}{KL^2}\sum_{k=1}^K\bbE_{t-2\tau}\|x_t^k\|^2+  \frac{48U_{\delta,L}^4}{L^4} \tau^2 \zeta^2 \beta^2
+\frac{48U_\omega^2 U_{\delta,L}^2}{L^2}\frac{\zeta^2}{K}\\
&+ \frac{48U_\omega^2 U_{\delta,L}^2}{L^2}\zeta^2m^2\rho^{2\tau}
+ 192U_\omega^2\tau^2U_r^2 \gamma^2\zeta^2+48U_\omega^2\frac{\zeta^2}{K}\sum_{k=1}^K \bbE_{t-2\tau}|y_t^k|^2
+ 32 \frac{U_{\delta,L}^4 }{L^4}U_{\omega}^4\zeta^4.
\end{align*}

\end{proof}

\begin{lemma}
\label{lmm: jth to tth error bound}
 For $0\leq t-\tau\leq j\leq t-1$, we have for any $k\in[K]$,
    \begin{align*}
        &\bbE_{t-2\tau}\|x_j^k\|\leq \tau \frac{U_{\delta,L}}{L}\beta + 4\tau\frac{U_{\delta,L}}{L}U_\omega^2 \zeta +\bbE_{t-2\tau}\|x_t^k\|,\\
         &\bbE_{t-2\tau}\|x_j^k\|^2\leq 4\tau^2 \frac{U_{\delta,L}^2}{L^2}\beta^2 +  64 \tau^2 U_\omega^3  \frac{U_{\delta,L}^2}{L^2}\zeta^2 + 2\bbE_{t-2\tau}\|x_t^k\|^2,\\
        &\bbE_{t-2\tau}|y_j^k| \leq 2\tau U_r\gamma+\bbE_{t-2\tau}|y_t^k|,\\
        &\bbE_{t-2\tau}|y_j^k|^2 \leq 8\tau^2U_r^2 \gamma^2+2\bbE_{t-2\tau}|y_t^k|^2,
    \end{align*}
when $ 4\tau\frac{U_{\delta,L}^2}{L^2}U_\omega^3 \zeta\leq 2\tau\frac{U_{\delta,L}}{L}U_\omega^2$, and $128 \tau^2 U_\omega^5  \frac{U_{\delta,L}^2}{L^2}\zeta^2\leq 32 \tau^2 U_\omega^3  \frac{U_{\delta,L}^2}{L^2}$.
\end{lemma}

Lemma~\ref{lmm: jth to tth error bound} aligns the error terms indexed by $j$ to a common reference
time $t$, which is required for a consistent Lyapunov analysis.
The bounds follow from the triangle inequality applied along the trajectory from $j$ to $t$; 
since $j \ge t-\tau$, the difference between $x_j^k$ (resp., $y_j^k$) and $x_t^k$ (resp., $y_t^k$)
accumulates at most $\tau$ one-step perturbations.
Each perturbation is upper bounded using the update dynamics, yielding deviations that scale
linearly in $\tau$, while $\bbE_{t-2\tau}\|x_t^k\|$ and $\bbE_{t-2\tau}|y_t^k|$ serve as anchored
reference terms.

\begin{proof}
For any $t-1\ge j\ge t-\tau$, we have  
\begin{align*}
    &\|x_j^k\|^2 \leq 2\|x_j^k-x_t^k\|^2+2\|x_t^k\|^2\\
    =& 2\|\bfB_t\omega_t^k-\bfB_j\omega_j^k\|^2+2\|x_t^k\|^2\\
    \leq&4\|\bfB_t\omega_t^k-\bfB_t\omega_j^k\|^2+4\|\bfB_t\omega_j^k-\bfB_j\omega_j^k\|^2+2\|x_t^k\|^2\\
    \leq& 4 \|\omega_{t}^k-\omega_j^k\|^2+ 4 U_\omega \|\bfB_{t}-\bfB_j\|^2 + 2\|x_t^k\|^2\\
    \leq& 4\tau^2 \frac{U_{\delta,L}^2}{L^2}\beta^2 + 4 \tau^2 U_\omega(2\frac{U_{\delta,L}}{L}U_\omega \zeta + 4\frac{U_{\delta,L}^2}{L^2}U_\omega^2 \zeta^2)^2+2\|x_t^k\|^2 ~~~~~ (\text{by Lemma \ref{lmm: crude parameter bound}}). 
\end{align*}
 
Then, taking expectation, we get
\begin{align*}
    &\bbE_{t-2\tau}\|x_j^k\|^2 \leq 4\tau^2 \frac{U_{\delta,L}^2}{L^2}\beta^2 +  32 \tau^2 U_\omega^3  \frac{U_{\delta,L}^2}{L^2}\zeta^2 +  128 \tau^2 U_\omega^5  \frac{U_{\delta,L}^2}{L^2}\zeta^4+ 2\bbE_{t-2\tau}\|x_t^k\|^2\\
    \leq& 4\tau^2 \frac{U_{\delta,L}^2}{L^2}\beta^2 +  64 \tau^2 U_\omega^3  \frac{U_{\delta,L}^2}{L^2}\zeta^2 + 2\bbE_{t-2\tau}\|x_t^k\|^2.
\end{align*}
The last inequality holds since $128 \tau^2 U_\omega^5  \frac{U_{\delta,L}^2}{L^2}\zeta^2\leq 32 \tau^2 U_\omega^3  \frac{U_{\delta,L}^2}{L^2}$.

Analogously, we get,
\begin{align*}
    &\bbE_{t-2\tau}\|x_j^k\|\leq\tau \frac{U_{\delta,L}}{L}\beta + 2\tau\frac{U_{\delta,L}}{L}U_\omega^2 \zeta + 4\tau\frac{U_{\delta,L}^2}{L^2}U_\omega^3 \zeta^2+\bbE_{t-2\tau}\|x_t^k\|\\
    \leq&\tau \frac{U_{\delta,L}}{L}\beta + 4\tau\frac{U_{\delta,L}}{L}U_\omega^2 \zeta +\bbE_{t-2\tau}\|x_t^k\|.
\end{align*}
The last inequality holds since $4\tau\frac{U_{\delta,L}^2}{L^2}U_\omega^3 \zeta\leq 2\tau\frac{U_{\delta,L}}{L}U_\omega^2$.
 
Similarly, for $|y_j^k|$, we have,
\begin{align*}
    &|y_j^k|^2 \leq 2|y_j^k-y_t^k|^2+2|y_t^k|^2\\
    \leq& 2\tau\sum_{i=j}^{t-1}|y_{i+1}^k-y_i^k|^2+2|y_t^k|^2\\
    =& 2\tau\sum_{i=j}^{t-1}|\eta_{i+1}^k-\eta_{i}^k|^2+|y_t^k|\\
    =&2\tau\gamma^2\sum_{i=j}^{t-1} \left| \frac{1}{L} \sum_{\ell=0}^{L-1} r_{t,\ell}^k - \eta^k_t \right|^2+2|y_t^k|^2\\
    \leq& 8\tau^2U_r^2 \gamma^2 +2|y_t^k|^2.
\end{align*}
Taking expectation, we get
\begin{align*}
    \bbE_{t-2\tau}|y_j^k|^2 \leq8\tau^2U_r^2 \gamma^2+2\bbE_{t-2\tau}|y_t^k|^2.
\end{align*}
Analogously, we get,
\begin{align*}
    \bbE_{t-2\tau}|y_j^k| \leq 2\tau U_r\gamma+\bbE_{t-2\tau}|y_t^k|.
\end{align*}
\end{proof}

\begin{lemma}
\label{lmm: local noise reduction} 
For $k\in[K]$, $t\geq0$, $L>0$, and any state $s_{t,0}^k$,  
we have
    \begin{align}
    &\bbE\left[\|\bfb_{t,L}^k\|\right]
    \le
    \pth{U_\delta+ U_\delta\sqrt{\frac{2 m\rho }{1-\rho} }}\sqrt{L},\label{eqn: bfb local noise}\\
    &\bbE\|\bfb_{t,L}^k\|^2\leq U_\delta^2(1+2\frac{m\rho}{1-\rho})L, \label{eqn: bfb^2 local noise}\\
    &\bbE\|\tilde b_{t,L}^k - \bar b_L^k\|\leq \sqrt{16U_r^2 + 32U_r^2\frac{m\rho}{1-\rho}}\sqrt{L},\label{eqn: btilde-bbar local noise}\\
    &\bbE\|\tilde b_{t,L}^k - \bar b_L^k\|^2\leq \pth{16U_r^2 + 32U_r^2\frac{m\rho}{1-\rho}}L,\label{eqn: btilde-bbar ^2 local noise}\\
    &\bbE\lnorm{\tilde A_{t,L}^k-A_L^k}{}\leq \sqrt{16  + 32\frac{m\rho}{1-\rho} }\sqrt{L} \label{eqn: Atilde-Abar local noise},\\
    &\bbE\lnorm{\tilde A_{t,L}^k-A_L^k}{}^2\leq \pth{16  + 32\frac{m\rho}{1-\rho} }L\label{eqn: Atilde-Abar ^2 local noise},
    \end{align}
    where $U_\delta:= 2U_r+2U_\omega$, and the expectation is taken over the Markovian sampling trajectories. 
\end{lemma}
\begin{proof}
We know that $V(s) = \phi(s)^\top  z^{k,*} = \bbE_{\pi\otimes P^k}[R(s,a)-J^k+V(s')]$, which implies for any $s$,
\begin{align}
\label{eqn: Bellman equation}
    \bbE_{s\sim u^k, (a, s^{\prime})\sim \pi\otimes P^k}[R(s,a)-J^k+(\phi(s')-\phi(s))^\top z^{k,*}]=0.
\end{align}

Then, for $\|\bfb_{t,L}^k\|$, taking expectation, we have
\begin{align*}
    &\bbE\|\bfb_{t,L}^k\|\leq \sqrt{\bbE\|\bfb_{t,L}^k\|^2}~~~~(\text{by Jensen's inequality})\\
    =&\sqrt{\bbE\left\|\sum_{\ell=0}^{L-1}\qth{ (r_{t,\ell}^k -J^k) + (\phi(s_{t,\ell+1}^k) - \phi(s_{t,\ell}^k))^\top \mathbf{B}^* \omega^{k,*}}\phi(s_{t,0}^k)\right\|^2} \\
    \leq& \Bigg(\sum_{\ell=0}^{L-1}\bbE\left\|\qth{ (r_{t,\ell}^k -J^k) + (\phi(s_{t,\ell+1}^k) - \phi(s_{t,\ell}^k))^\top \mathbf{B}^* \omega^{k,*}}\phi(s_{t,0}^k)\right\|^2\\
    &+ 2\sum_{\ell=0}^{L-1}\sum_{\ell'=\ell+1}^{L-1}\bbE \bigg\langle\qth{ (r_{t,\ell}^k -J^k) + (\phi(s_{t,\ell+1}^k) - \phi(s_{t,\ell}^k))^\top \mathbf{B}^* \omega^{k,*}}\phi(s_{t,0}^k),\\
    &\qquad\qquad\qquad\qquad\qth{ (r_{t,\ell'}^k -J^k) + (\phi(s_{t,\ell'+1}^k) - \phi(s_{t,\ell'}^k))^\top \mathbf{B}^* \omega^{k,*}}\phi(s_{t,0}^k)\bigg\rangle\Bigg)^{\frac{1}{2}}\\
    =& \Bigg(\sum_{\ell=0}^{L-1}\bbE\left\|\qth{ (r_{t,\ell}^k -J^k) + (\phi(s_{t,\ell+1}^k) - \phi(s_{t,\ell}^k))^\top \mathbf{B}^* \omega^{k,*}}\phi(s_{t,0}^k)\right\|^2\\
    &+ 2\sum_{\ell=0}^{L-1}\sum_{\ell'=\ell+1}^{L-1}\bbE\left[ \bbE\left[\bigg\langle\qth{ (r_{t,\ell}^k -J^k) + (\phi(s_{t,\ell+1}^k) - \phi(s_{t,\ell}^k))^\top \mathbf{B}^* \omega^{k,*}}\phi(s_{t,0}^k),\right.\right.\\
    &\qquad\qquad\qquad\qquad \left.\left.\qth{ (r_{t,\ell'}^k -J^k) + (\phi(s_{t,\ell'+1}^k) - \phi(s_{t,\ell'}^k))^\top \mathbf{B}^* \omega^{k,*}}\phi(s_{t,0}^k)\bigg\rangle \mid v_{0:{(j,\ell)}} \right ] \right] \Bigg)^{\frac{1}{2}}\\ 
    \leq& \Bigg(U_\delta^2 L+ 2\sum_{\ell=0}^{L-1}\sum_{\ell'=\ell+1}^{L-1}\bbE \left [\langle\qth{ (r_{t,\ell}^k -J^k) + (\phi(s_{t,\ell+1}^k) - \phi(s_{t,\ell}^k))^\top \mathbf{B}^* \omega^{k,*}}\phi(s_{t,0}^k),\right.\\
    &\qquad\qquad\qquad\qquad \left.\bbE[\qth{ (r_{t,\ell'}^k -J^k) + (\phi(s_{t,\ell'+1}^k) - \phi(s_{t,\ell'}^k))^\top \mathbf{B}^* \omega^{k,*}}\phi(s_{t,0}^k) \mid v_{0:{(j,\ell)}} ]\right]\rangle \Bigg)^{\frac{1}{2}}\\
    \leq& \Bigg(U_\delta^2 L+ 2\sum_{\ell=0}^{L-1}\sum_{\ell'=\ell+1}^{L-1}\bbE [\lnorm{\qth{ (r_{t,\ell}^k -J^k) + (\phi(s_{t,\ell+1}^k) - \phi(s_{t,\ell}^k))^\top \mathbf{B}^* \omega^{k,*}}\phi(s_{t,0}^k)}{}\\
    &\qquad\qquad\qquad\qquad
    \|\bbE[\qth{ (r_{t,\ell'}^k -J^k) + (\phi(s_{t,\ell'+1}^k) - \phi(s_{t,\ell'}^k))^\top \mathbf{B}^* \omega^{k,*}}\phi(s_{t,0}^k) \mid v_{0:{(t,\ell)}} ] \| ]\Bigg)^{\frac{1}{2}}(\text{by Cauchy-Schwarz})\\
    \overset{(a)}{\leq}& \sqrt{U_\delta^2L+ 2\sum_{\ell=0}^{L-1}\sum_{\ell'=\ell+1}^{L-1}U_{\delta} \left\|\bbE\qth{\qth{ (r_{t,\ell'}^k -J^k) + (\phi(s_{t,\ell'+1}^k) - \phi(s_{t,\ell'}^k))^\top \mathbf{B}^* \omega^{k,*}}\phi(s_{t,0}^k) \mid v_{0:{(t,\ell)}}} \right\|}\\
    \leq& \sqrt{U_\delta^2 L+ 2\sum_{\ell=0}^{L-1}\sum_{\ell'=\ell+1}^{L-1}U_\delta^2 m\rho^{\ell'-\ell} } ~~~~~~~~~~(\text{by the ergodicity Assumption})\\
    \leq& U_\delta\sqrt{L}+ \sqrt{2U_\delta^2 m\frac{\rho L}{1-\rho} }\\
    =& \pth{U_\delta+ U_\delta\sqrt{\frac{2 m\rho }{1-\rho} }}\sqrt{L}, 
\end{align*} 
where inequality (a) holds because $\lnorm{\qth{ (r_{t,\ell}^k -J^k) + (\phi(s_{t,\ell+1}^k) - \phi(s_{t,\ell}^k))^\top \mathbf{B}^* \omega^{k,*}}\phi(s_{t,0}^k)}{}\leq U_\delta$.

Similarly, for \eqref{eqn: bfb^2 local noise}, we have 
\begin{align*}
    \bbE\|\bfb_{t,L}^k\|^2\leq U_\delta^2L+2U_\delta^2\frac{m\rho}{1-\rho}L = U_\delta^2(1+2\frac{m\rho}{1-\rho})L.
\end{align*}

Now we prove \eqref{eqn: btilde-bbar ^2 local noise}.
For ease of exposition, we omit the explicit specification of the underlying distributions in the expectation when it is clear from the context.

\begin{align*}
    &\bbE\|\tilde b_{t,L}^k - \bar b_L^k\|^2\\
    =&  \bbE\lnorm{\sum_{\ell=0}^{L-1}(r_{t,\ell}^k - J^k)\phi(s_{t,0}^k)-  \bbE_{s^{(0)}\sim\mu^k, (a^{(\ell)}, s^{(\ell)}) \sim \pi \otimes P^k ~ \forall \ell\in [0, L-1]}\qth{\sum_{\ell=0}^{L-1}(r(s^{(\ell)},a^{(\ell)}) - J^k)\phi(s^{(0)})}}{}^2\\
    \leq& \sum_{\ell=0}^{L-1}\bbE\lnorm{\qth{(r_{t,\ell}^k - J^k)\phi(s_{t,0}^k)-  \bbE_{s^{(0)}\sim\mu^k, (a^{(\ell)}, s^{(\ell)}) \sim \pi \otimes P^k ~ \forall \ell\in [0, L-1]}\qth{(r(s^{(\ell)},a^{(\ell)}) - J^k)\phi(s^{(0)})}}}{}^2\\
    &+ 2\sum_{\ell=0}^{L-1}\sum_{\ell'=\ell+1}^{L-1}\bbE\big\langle (r_{t,\ell}^k - J^k)\phi(s_{t,0}^k)-  \bbE_{s^{(0)}\sim\mu^k, (a^{(\ell)}, s^{(\ell)}) \sim \pi \otimes P^k ~ \forall \ell\in [0, L-1]}\qth{(r(s^{(\ell)},a^{(\ell)}) - J^k)\phi(s^{(0)})} , \\
    & \qquad\qquad\qquad (r_{t,\ell'}^k - J^k)\phi(s_{t,0}^k)-  \bbE_{s^{(0)}\sim\mu^k, (a^{(\ell)}, s^{(\ell)}) \sim \pi \otimes P^k ~ \forall \ell\in [0, L-1]}\qth{(r(s^{(\ell')},a^{(\ell')}) - J^k)\phi(s^{(0)})}\big\rangle\\
    \overset{(a)}{\leq}& 16U_r^2 L+ 32U_r^2\sum_{\ell=0}^{L-1}\sum_{\ell'=\ell+1}^{L-1} m\rho^{\ell'-\ell}\\
    \leq& \pth{16U_r^2 + 32U_r^2\frac{m\rho}{1-\rho}}L.
\end{align*}
(a) follows analogously to the proof of \eqref{eqn: bfb local noise}.

For \eqref{eqn: btilde-bbar local noise}, by Jensen's inequality, we easily get
\begin{align*}
    \bbE\|\tilde b_{t,L}^k - \bar b_L^k\| \leq \sqrt{\bbE\|\tilde b_{t,L}^k - \bar b_L^k\|^2} \leq \sqrt{16U_r^2 + 32U_r^2\frac{m\rho}{1-\rho}}\sqrt{L}.
\end{align*}

Similarly, for \eqref{eqn: Atilde-Abar ^2 local noise}, we have,
\begin{align*}
    &\bbE\lnorm{\tilde A_{t,L}^k-A_L^k}{}^2\\
    =& \bbE\lnorm{\sum_{\ell=0}^{L-1} \phi(s_{t,0}^k)(\phi(s_{t,\ell+1}^k)-\phi(s_{t,\ell}^k))^\top-\mathbb{E}_{s^{(0)}\sim\mu^k, (a^{(\ell)}, s^{(\ell)}) \sim \pi \otimes P^k ~ \forall \ell\in [0, L-1]}[\phi(s^{(0)})(\phi(s^{(\ell+1)})-\phi(s^{(\ell)}))^\top}{}^2\\
    =& \sum_{\ell=0}^{L-1} \bbE \lnorm{\phi(s_{t,0}^k)(\phi(s_{t,\ell+1}^k)-\phi(s_{t,\ell}^k))^\top-\mathbb{E}_{s^{(0)}\sim\mu^k, (a^{(\ell)}, s^{(\ell)}) \sim \pi \otimes P^k ~ \forall \ell\in [0, L-1]}[\phi(s^{(0)})(\phi(s^{(\ell+1)})-\phi(s^{(\ell)}))^\top}{}^2 \\
    &+ 2\sum_{\ell=0}^{L-1}\sum_{\ell'=0}^{L-1}\bbE \left\langle \phi(s_{t,0}^k)(\phi(s_{t,\ell+1}^k)-\phi(s_{t,\ell}^k))^\top-\mathbb{E}_{s^{(0)}\sim\mu^k, (a^{(\ell)}, s^{(\ell)}) \sim \pi \otimes P^k ~ \forall \ell\in [0, L-1]}[\phi(s^{(0)})(\phi(s^{(\ell+1)})-\phi(s^{(\ell)}))^\top, \right.\\
    & \qquad\qquad\qquad \left. \phi(s_{t,0}^k)(\phi(s_{t,\ell+1}^k)-\phi(s_{t,\ell}^k))^\top-\mathbb{E}_{s^{(0)}\sim\mu^k, (a^{(\ell)}, s^{(\ell)}) \sim \pi \otimes P^k ~ \forall \ell\in [0, L-1]}[\phi(s^{(0)})(\phi(s^{(\ell+1)})-\phi(s^{(\ell)}))^\top\right\rangle\\
    \leq&16 L + 32 \sum_{\ell=0}^{L-1}\sum_{\ell'=\ell+1}^{L-1} m\rho^{\ell'-\ell}\\
    \leq& \pth{16  + 32\frac{m\rho}{1-\rho} }L.
\end{align*}
By Jensen's inequality, \eqref{eqn: Atilde-Abar local noise} can be proved,
\begin{align*}
    \bbE\lnorm{\tilde A_{t,L}^k-A_L^k}{} \leq \sqrt{\bbE\lnorm{\tilde A_{t,L}^k-A_L^k}{}^2}\leq \sqrt{16  + 32\frac{m\rho}{1-\rho} }\sqrt{L}.
\end{align*}
\end{proof}

\begin{lemma}
\label{lmm: MC observable mixing} 
Suppose that Assumption \ref{assp: uniform ergodicity}. 
Fix $\tau>0$. For any $t\ge \tau$, $L>0$, and $k\in [K]$, it holds that  
\begin{align*}
&\|\bbE_{t-\tau}\big[\tilde b_{t,L}^k - \bar b_L^k \big] \| \le 4 LU_{r} m\rho^{\tau}\\ 
& \|\bbE_{t-\tau}\big[\tilde A_{t,L}^k - \bar A_L^k \big] \| \le 4 m\rho^{\tau}\\
& \|\bbE_{t-\tau}[\bfb_{t,L}^k]\| \le 2U_{\delta, L} m \rho^{\tau}. 
\end{align*}  
\end{lemma}
\begin{proof}
Recall from Eq.\,(\ref{def: Markov drift and noise}) that 
\begin{align*}
\bfb_{t,L}^k = \sum_{\ell=0}^{L-1}\big( (r_{t,\ell}^k - J^k) + (\phi(s_{t,\ell+1}^k) - \phi(s_{t,\ell}^k))^\top \mathbf{B}^*\omega^{k,*}\big)\phi(s_{t,0}^k),      
\end{align*}
and the average-reward Bellman equation 
\[
\mathbb{E}_{s^{(0)}\sim\mu^k, (a^{(\ell)}, s^{(\ell)}) \sim \pi \otimes P^k ~ \forall \ell\in [0, L-1]} \qth{\sum_{\ell=0}^{L-1} (r(a^{(\ell)}, s^{(\ell)}) - J^k) + (\phi(s^{(\ell+1)}) - \phi(s^{(\ell)}))^{\top} B^*w^{k,*}\phi(s^{(0)})} =0.  
\]
Thus, we have  
\begin{align*}
\|\bbE_{t-\tau}[\bfb_{t,L}^k]\|  
=&\left | \bbE_{t-\tau}\qth{\sum_{\ell=0}^{L-1}\big( (r_{t,\ell}^k - J^k) + (\phi(s_{t,\ell+1}^k) - \phi(s_{t,\ell}^k))^\top \mathbf{B}^*\omega^{k,*}\big)\phi(s_{t,0}^k)} \right.\\
&\quad \left. - \mathbb{E}_{s^{(0)}\sim\mu^k, (a^{(\ell)}, s^{(\ell)}) \sim \pi \otimes P^k ~ \forall \ell\in [0, L-1]} \qth{\sum_{\ell=0}^{L-1} (r(a^{(\ell)}, s^{(\ell)}) - J^k) + (\phi(s^{(\ell+1)}) - \phi(s^{(\ell)}))^{\top} B^*w^{k,*}\phi(s^{(0)})} \right | \\
\le & 2U_{\delta, L} \cdot d_{TV}(\mathbb{P}_{0: \tau}^k(s_{t,0}^k, a_{t,1}^k, s_{t,1}^k, \cdots, a_{t, L-1}^k, s_{t, L}^k \mid s_{t-\tau}^k = s), \mu^k \otimes (\pi\otimes P^k)^{L-1})\\
\le & 2U_{\delta, L} m \rho^{\tau}, 
\end{align*}
where the notation $\mathbb{P}_{0:\tau}^k$ is defined in Assumption \ref{assp: uniform ergodicity}, the first inequality follows from Lemma \ref{lmm: U delta}, and the last inequality follows from Assumption \ref{assp: uniform ergodicity} and the fact that $L$-block of a Markov chain mixes at the same rate as the original chain. 

Recall from Eq.\,(\ref{eq: local reward Markovian noise}) that 
\begin{align*}
    \tilde{b}_{t,L}^k = \sum_{\ell=0}^{L-1}(r_{t,\ell}^k - J^k)\phi(s_{t,0}^k);\quad \bar{b}_L^k=  \bbE_{s^{(0)}\sim\mu^k, (a^{(\ell)}, s^{(\ell)}) \sim \pi \otimes P^k ~ \forall \ell\in [0, L-1]}\qth{\sum_{\ell=0}^{L-1}(r(s^{(\ell)},a^{(\ell)}) - J^k)\phi(s^{(0)})}.    
\end{align*}
Note that $\abth{\tilde{b}_{t,L}^k} \le 2L U_r.$ Similar to the above derivation, we have  
\begin{align*}
\|\bbE_{t-\tau}\big[\tilde b_{t,L}^k - \bar b_L^k \big]\|  
&\le 2LU_{r} 2 \cdot d_{TV}(\mathbb{P}_{0: \tau}^k(s_{t,0}^k, a_{t,1}^k, s_{t,1}^k, \cdots, a_{t, L-1}^k, s_{t, L}^k \mid s_{t-\tau}^k = s), \mu^k \otimes (\pi\otimes P^k)^{L-1}) \\
& \le 4 LU_{r} m\rho^{\tau}.   
\end{align*}

Recall from Eq.\,(\ref{def: Markov drift and noise}) that $\tilde A_{t,L}^k= \phi(s_{t,0}^k)(\phi(s_{t,L}^k)-\phi(s_{t,0}^k))^\top$. Since $\|\tilde A_{t,L}^k\| \le 2$, we have 
\begin{align*}
\|\bbE_{t-\tau}\big[\tilde A_{t,L}^k - \bar A_L^k \big] \| \le 4 m\rho^{\tau}.     
\end{align*}

\end{proof}

\subsection{Proof of Eq.\,(\ref{eq: fact 2})}
\label{app:fact2-proof}
\begin{proof}
Recall that $\calR$ is the optimal rotation matrix that aligns $\bfB_{t+1}$ and $\bfB^*$ in the sense that $\calR=\text{argmin}_{\calR^\top\calR=\bfI}\|\bfB_{t+1}\calR-\bfB^*\|_F$. That is, $\calR = \bfU\bfV^\top$, where $\bfU\Sigma\bfV^\top=\bfB_{t+1}\bfB^*$ is the SVD. We have  
    \begin{align*}
     &\bfI -(\bfB_{t+1}\calR)^\top\bfB^* \\
     =& \bfI -\calR^\top\bfB_{t+1}^\top\bfB^*\\
     =& \bfI -(\bfU\bfV^\top)^\top\bfU\Sigma\bfV^\top\\
     =& \bfI- \bfV\Sigma\bfV^\top.
    \end{align*}
    For any given matrix $\bfA$, let $\sigma_i(\bfA)$ denote its $i$-th largest singular value. %
    Since multiplying an orthogonal matrix does not change the singular values, the singular values of $\bfI -(\bfB_{t+1}\calR)^\top\bfB^*$ can be bounded as follows,
    \begin{align*}
        \sigma_i(\bfI- \bfV\Sigma\bfV^\top) = \sigma_i(\bfV(\bfI- \Sigma)\bfV^\top) 
        =\sigma_i(\bfI-\Sigma)
        = 1-\cos(\theta_i),
    \end{align*}
    where $\cos(\theta_i)$ are the singular values of the matrix $(\bfB_{t+1}\calR)^\top\bfB^*$, and $\theta_i$ are the principal angles between the two subspaces of interest. Then 
    \begin{align*}
        \sigma_i(\bfI -(\bfB_{t+1}\calR)^\top\bfB^*)=1-\cos(\theta_i)\leq 1- \cos^2(\theta_i) = \sin^2(\theta_i).
    \end{align*}
    Recall $\sigma_i(\bfB_\perp^{*\top}\bfB_{t+1}) = \sin(\theta_i)$. All the singular values of $\bfI -(\bfB_{t+1}\calR)^\top\bfB^*$ are smaller than the largest singular value of $\bfB_\perp^{*\top}\bfB_{t+1}$ squared, i.e., 
    \begin{align*}
        \sigma_i(\bfI -(\bfB_{t+1}\calR)^\top\bfB^*)\leq \|\bfB_\perp^{*\top}\bfB_{t+1}\|^2,
    \end{align*}
    which implies
    \begin{align*}
        \|\bfI -(\bfB_{t+1}\calR)^\top\bfB^*\|\leq \|\bfB_\perp^{*\top}\bfB_{t+1}\|^2.
    \end{align*}
\end{proof}

\newpage
\section{More Experiment Results}
\begin{figure}[!ht]
    \centering
    \vspace{2em}
    
    \includegraphics[width=1\textwidth]{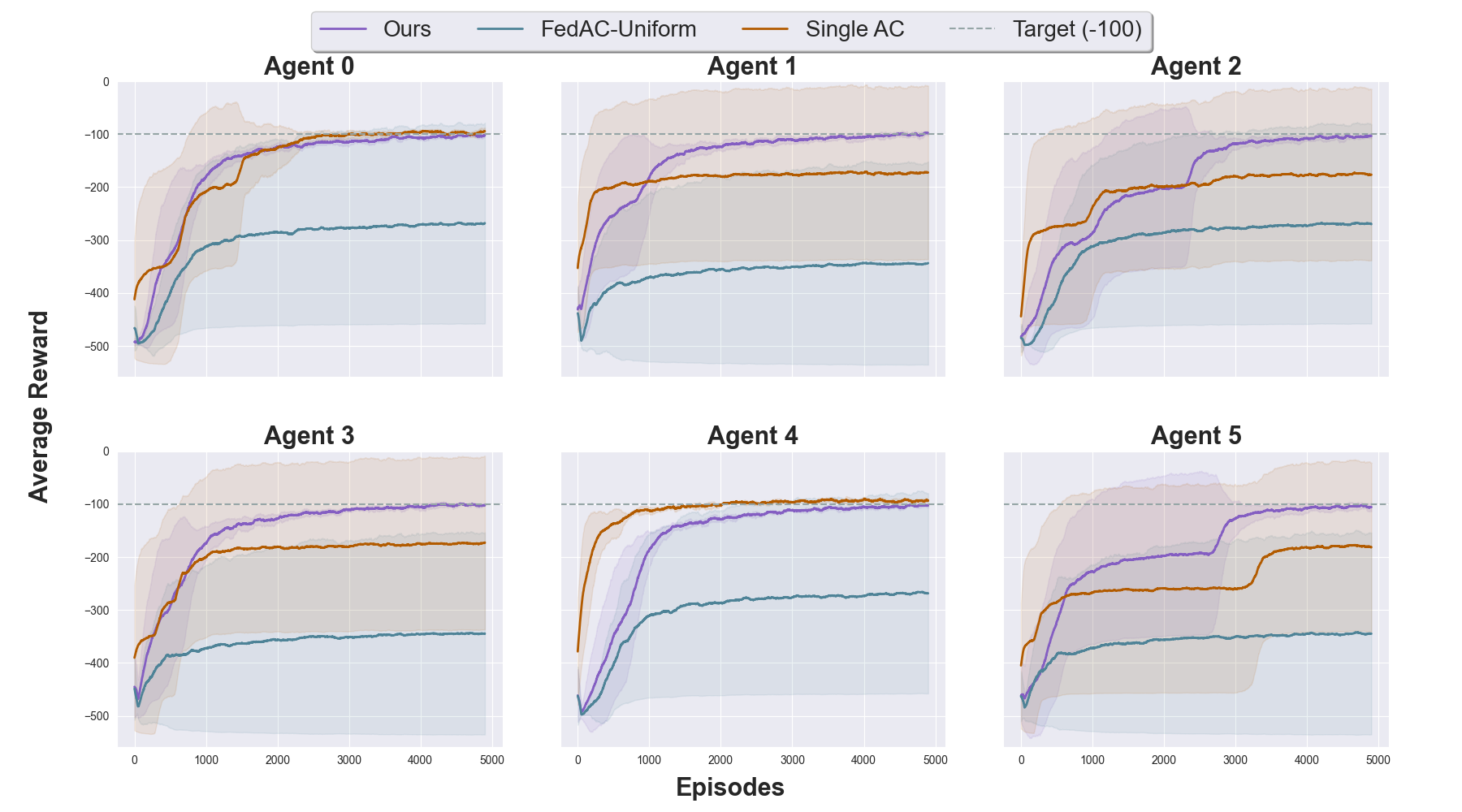}
    \caption{Policy Learning Performance Comparison on Acrobot}
    \label{fig:control_acro}
    
    \vspace{2em}
    
    \includegraphics[width=1\textwidth]{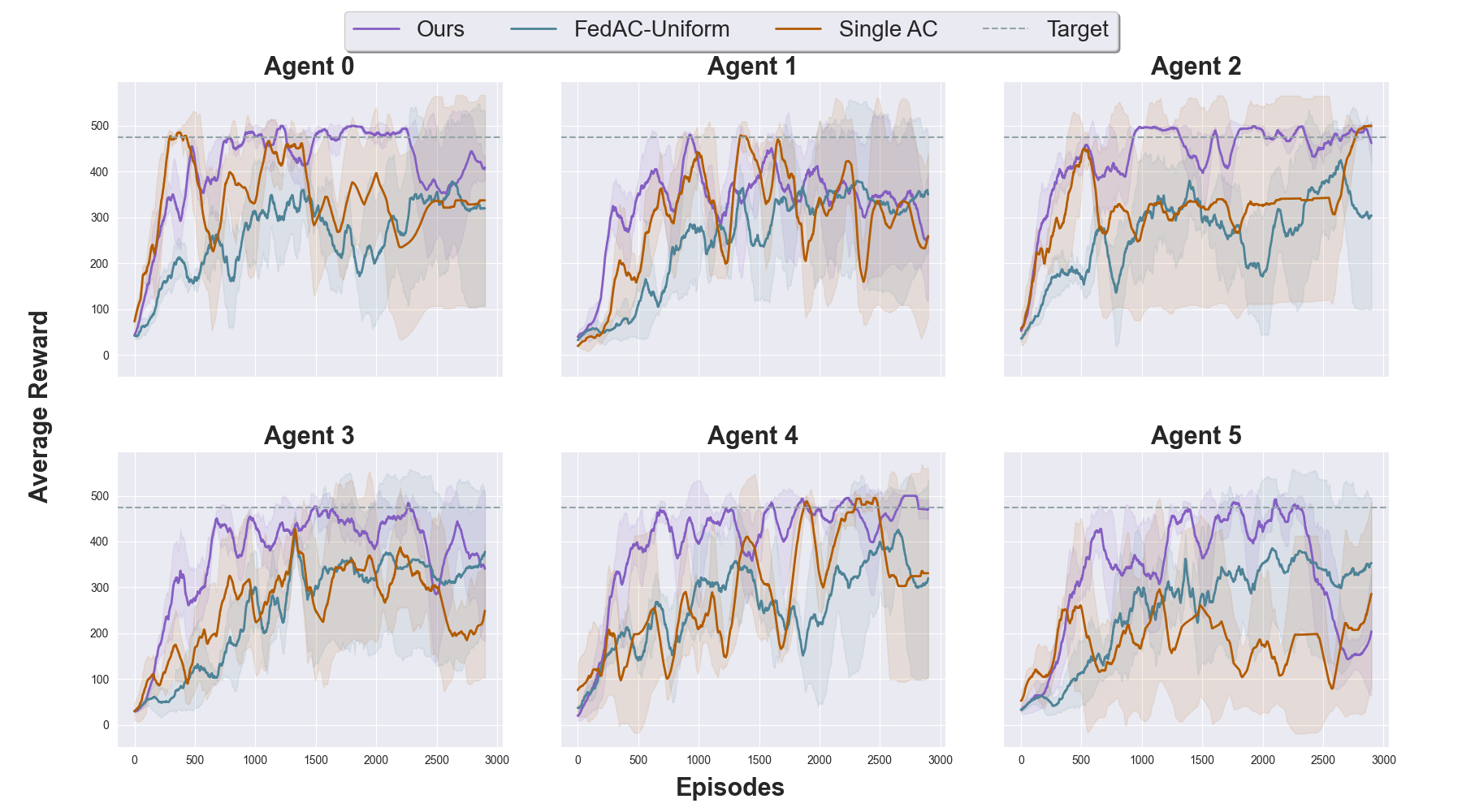}
    \caption{Policy Learning Performance Comparison on Cartpole}
    \label{fig:control_cartpole}
\end{figure}

\end{document}